\documentclass{article}

\usepackage{microtype}
\usepackage{graphicx}
\usepackage{subfigure}
\usepackage{booktabs} 
\usepackage{multirow} 
\usepackage{textcomp}
\usepackage{enumitem}
\usepackage{fullpage}
\usepackage{natbib} 
\usepackage[colorlinks=true, linkcolor=blue, citecolor=blue, urlcolor=blue]{hyperref}

\usepackage{amsmath}
\usepackage{amssymb}
\usepackage{mathtools}
\usepackage{amsthm}
\usepackage[capitalize,noabbrev]{cleveref}

\usepackage{amsmath,amsfonts,bm}

\def\eqref#1{equation~\ref{#1}}

\def\1{\bm{1}}

\DeclareMathAlphabet{\mathsfit}{\encodingdefault}{\sfdefault}{m}{sl}
\SetMathAlphabet{\mathsfit}{bold}{\encodingdefault}{\sfdefault}{bx}{n}

\theoremstyle{plain}
\newtheorem{theorem}{Theorem}[section]

\newtheorem{lemma}[theorem]{Lemma}

\theoremstyle{definition}

\theoremstyle{remark}

\newcommand{\ital}{\textit}

\title{Reconstructing Cell Lineage Trees from Phenotypic Features with Metric Learning}

\author{
    Da Kuang$^{1}$, Guanwen Qiu$^{1}$, Junhyong Kim$^{1,2}$ \\
    \small $^{1}$Department of Computer and Information Science, University of Pennsylvania, Philadelphia, USA \\
    \small $^{2}$Department of Biology, University of Pennsylvania, Philadelphia, USA \\
    \texttt{kuangda@seas.upenn.edu, guanwenq@seas.upenn.edu, junhyong@sas.upenn.edu}
}

\date{}

\begin{document}

\maketitle

\begin{abstract}
How a single fertilized cell gives rise to a complex array of specialized cell types in development is a central question in biology. The cells grow, divide, and acquire differentiated characteristics through poorly understood molecular processes. A key approach to studying developmental processes is to infer the tree graph of cell lineage division and differentiation histories, providing an analytical framework for dissecting individual cells’ molecular decisions during replication and differentiation. Although genetically engineered lineage-tracing methods have advanced the field, they are either infeasible or ethically constrained in many organisms. In contrast, modern single-cell technologies can measure high-content molecular profiles (\textit{e.g.}, transcriptomes) in a wide range of biological systems.

Here, we introduce \emph{CellTreeQM}, a novel deep learning method based on transformer architectures that learns an embedding space with geometric properties optimized for tree-graph inference. By formulating lineage reconstruction as a tree-metric learning problem, we have systematically explored supervised, weakly supervised, and unsupervised training settings and present a \emph{Cell Lineage Reconstruction Benchmark} to facilitate comprehensive evaluation of our learning method. We benchmarked the method on (1) synthetic data modeled via Brownian motion with independent noise and spurious signals and (2) lineage-resolved single-cell RNA sequencing datasets. Experimental results show that \emph{CellTreeQM} recovers lineage structures with minimal supervision and limited data, offering a scalable framework for uncovering cell lineage relationships in challenging animal models. To our knowledge, this is the first method to cast cell lineage inference explicitly as a metric learning task, paving the way for future computational models aimed at uncovering the molecular dynamics of cell lineage.
\end{abstract}

\section{Introduction}
\begin{figure}[ht]
    \centering
    \includegraphics[width=0.5\linewidth]{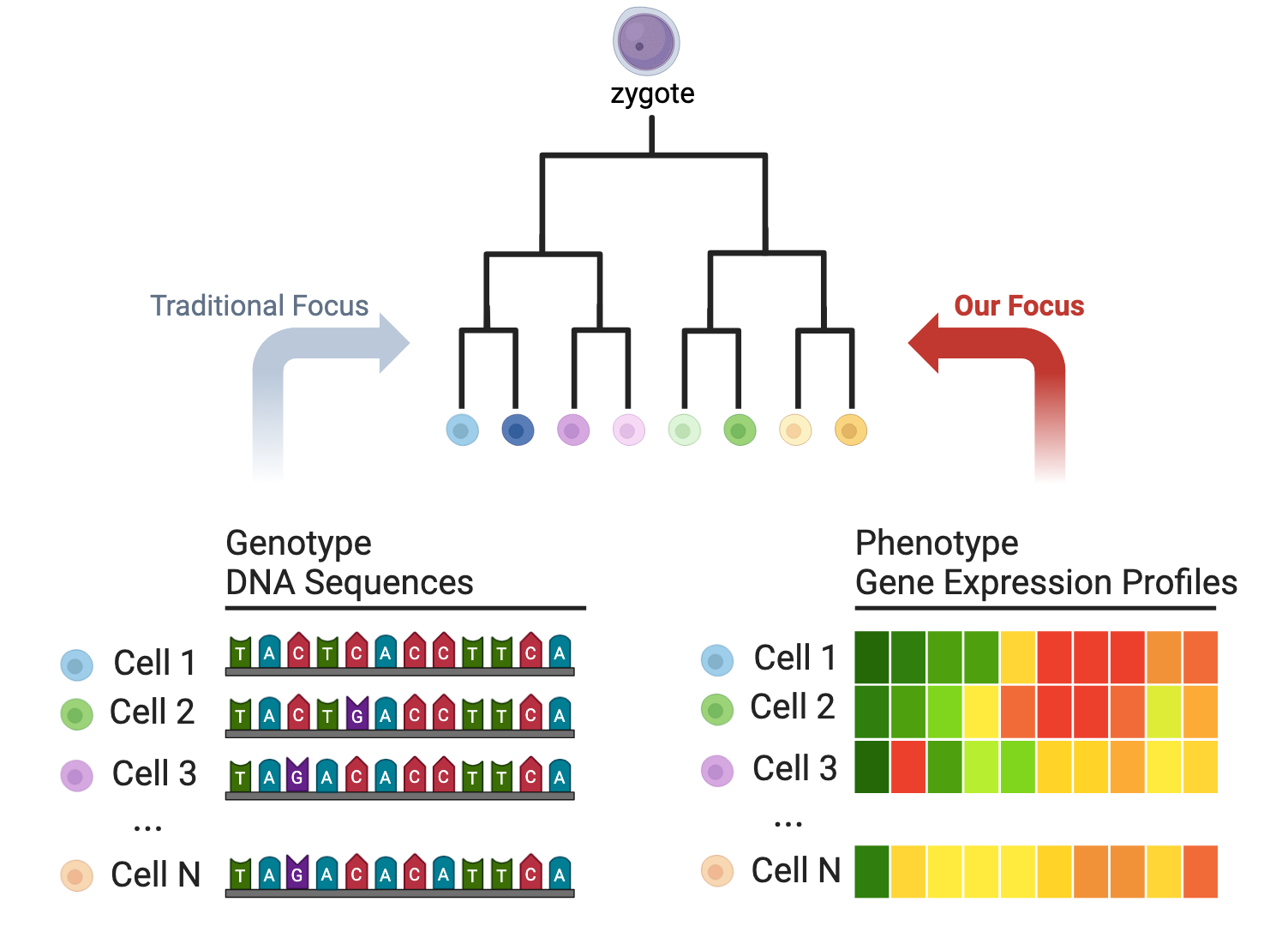}
    \caption{\textbf{Exploring phenotype-based cell lineage reconstruction.} This figure highlights the focus of our study on reconstructing cell lineage trees using phenotype data, specifically gene expression profiles (right panel), in contrast to traditional methods that rely on genotype data, such as DNA sequences (left panel).}
    \label{fig:figure1}
\end{figure}
Understanding how the single cell of a fertilized egg repeatedly divides and differentiates, \textit{i.e.}, the process of different cells acquiring unique characteristics (\textit{e.g.}, skin cell, muscle cell), to give rise to a fully formed animal has been a long-standing goal in biology \citep{wolpert_principles_2015,slack_essential_2021}. A key component of this developmental process is the tree-graph of \emph{cell lineages}, which provides a roadmap of how diverse cell types arise from a single progenitor \citep{clevers_cancer_2011,shapiro_single-cell_2013,wagner_lineage_2020}. The cell replication process cannot be directly observed for most organisms; therefore, inferring or reconstructing the cell lineage tree from features measured on the individual cells is an important challenge. Beyond organismal development, knowledge of the cell lineages has wide biomedical applications, including in deciphering molecular processes of cell injury and repair, tumor development, degenerative diseases, and more \citep{zhang_lineage_2020,sivandzade_regenerative_2021}. 

Currently, the gold standard for cell lineage tree reconstruction is \emph{prospective lineage tracing} \citep{kretzschmar_lineage_2012}. One popular approach leverages CRISPR-Cas9 to genetically engineered ``recorders'' \textemdash{} that is, exogenous DNA sequences designed to accumulate mutations. (Fig.~\ref{fig:figure1}, left). Although powerful, these recorders face key limitations. The number of divisions that can be tracked is limited by molecular constraints, such as the size of the target array. Additionally, achieving an optimal mutation rate is crucial; if mutations accumulate too slowly, they provide insufficient information for lineage reconstruction, while too rapid accumulation can lead to early saturation, obscuring lineage relationships \citep{salvador-martinez_is_2019}. Most importantly, since this method involves genetic engineering, it can only be used in contexts where such genome manipulation is feasible or ethically permissible \citep{mckenna_recording_2019,zafar_single-cell_2020}.

Before organismal DNA sequences were widely available, researchers developed phylogenetic methods based on \emph{phenotypes}—such as morphological traits—to infer species relationships. For foundational results and methodologies in this field, we refer readers to \citep{kim_tutorial_1999}. Today, advances in single-cell biology enable researchers to measure high-content molecular phenotypes from individual cells. For example, it is now routine to measure a cell’s transcriptome—the complete set of RNA molecules in the cell—and represent it as a vector of counts for each RNA species. \textbf{An open question is whether such high-content molecular phenotypes give enough information for tree reconstruction and if so what algorithm can be used to find it} (Fig.~\ref{fig:figure1}, right).

We show a positive answer to the above question by introducing \emph{CellTreeQM} (Cell-Tree Quartet Metric), a Deep Learning framework built on transformer encoder that maps the transcriptome data to a space in which distances reflect tree-like relationships. Our results demonstrate that reconstructing cell lineage structures from purely transcriptome data is both tractable and data-efficient.

Our main contributions are as follows:

\paragraph{Formulating the Cell Lineage Tree Reconstruction Problem.}
We frame the reconstruction of cell lineage trees as a metric learning problem and identify three practical scenarios encompassing supervised, weakly supervised, and unsupervised settings (Fig.~\ref{fig:intro-framework}). Within the weakly supervised paradigm, we introduce two practical cases. In the \emph{high-level partitioning setting}, biologists possess prior knowledge of major clades within the lineage. In the \emph{partially leaf-labeled setting}, lineage-tracing technologies provide topological labels for a subset of cells, with the objective of extrapolating these relationships to the unlabeled leaves.

\paragraph{Proposed Solution.}
We present \emph{CellTreeQM}, a feature learning framework for cell lineage tree reconstruction. Inspired by \citep{de_soete_least_1983}, we design a loss function that explicitly encourages the learned embedding space to satisfy tree-metric properties and show that stochastic gradient descent is efficient in finding a generalizable embedding.

\paragraph{Lineage Reconstruction Benchmark.}
We introduce a Lineage Reconstruction Benchmark with (a) synthetic datasets based on Brownian motion with independent noise and spurious signals, (b) lineage-resolved single-cell RNA sequencing (scRNA-seq) datasets. Experimental results on the benchmark demonstrate that \emph{CellTreeQM} efficiently reconstructs lineage structures under weak supervision and limited data, providing a scalable framework for uncovering cell lineage relationships.

\begin{figure}[ht]
    \centering
    \includegraphics[width=\linewidth]{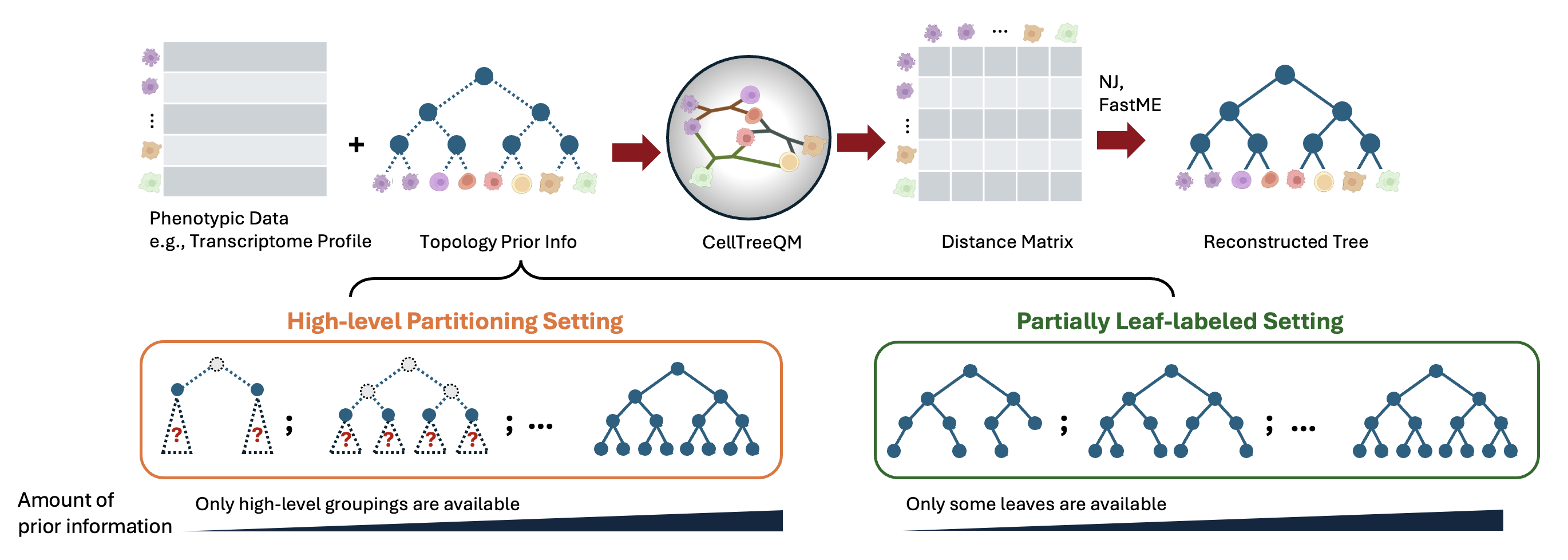}
    \caption{\textbf{Overview of the CellTreeQM workflow for lineage reconstruction using Metric Learning.} When the full tree is known as prior knowledge, this is a supervised setting. When no prior information about the tree is available, the setting is unsupervised. In between, we highlight two weakly supervised settings: the High-level Partitioning Setting, where only high-level groupings are available, and the Partially Leaf-labeled Setting, where topological labels are provided for a subset of leaves.}
    \label{fig:intro-framework}
\end{figure}

\section{Key Challenges}
Reconstructing a cell lineage tree from \emph{phenotypic} data conceptually parallels phylogenetics, which infers evolutionary relationships from discrete data such as aligned DNA or protein sequences. Historically, classical phylogenetics relied on phenotypic information (e.g., morphological traits) encoded by biological experts; however, with the advent of affordable sequencing technologies, modern phylogenetic methods predominantly use genotype-based data and well-defined stochastic models of sequence evolution to estimate both topology and branch lengths \citep{felsenstein_inferring_2003}.

Although the field now has access to high-content, high-dimensional \emph{molecular} phenotypic data (e.g., gene expression), these data have not been widely applied to lineage reconstruction. We attribute this gap primarily to two key challenges (detailed below). These challenges collectively underscore why comprehensive lineage annotations remain scarce and why robust methods for lineage reconstruction—particularly under incomplete or noisy labels—are needed.

\paragraph{Uncharacterized Stochastic Processes in Gene Expression Data}\label{uncharacterized_stochastic_processes}

In phylogenetics, species relationships are often modeled using well-defined stochastic processes such as DNA mutation rates, which have been empirically validated across diverse datasets. These models provide a principled framework for lineage reconstruction, typically through likelihood-based principles. Unlike genetic sequences, transcriptomic profiles are highly dynamic and context-dependent.  Gene expression in a cell reflects both its functional and regulatory states, many of which are transient, environmentally induced, or unrelated to heritable lineage history. Variables such as stage in the cell cycle, nutrient availability, and niche-specific signaling can obscure true ancestral signals. Additionally, distantly related cells may exhibit convergent expression profiles due to functional similarity. For instance, in the model nematode, \textit{Caenorhabditis elegans}, transcriptomic similarity initially correlates with lineage but diverges after gastrulation, illustrating how lineage signals can be lost as cells commit to specialized fates \citep{packer_lineage-resolved_2019,qiu_systematic_2022}.  

Nevertheless, because cell lineages make history-dependent decisions during developmental differentiation, each cell is hypothesized to retain a record of these decisions in its molecular state. Despite the challenges posed by transient and non-heritable components of gene expression, researchers have used transcriptomic data as an auxiliary source of information to refine genetic barcoding approaches (e.g., CRISPR-Cas9) for lineage tracing \citep{zafar_single-cell_2020,wang_cospar_2022,pan_linrace_2023}. Moreover, certain gene expression patterns have been successfully leveraged to reconstruct portions of multicellular structures \citep{phillips_memory_2019,shaffer_memory_2020,ruckert_clonal_2022,mold_clonally_2024}. Additionally, selecting specific gene subsets has been shown to improve clustering of cells with shared lineage labels \citep{eisele_gene-expression_2024}. 

\paragraph{Limited ground-truth lineage annotations}
Although technologies like CRISPR-based barcoding can provide direct lineage labels, they remain impractical or ethically restricted for many organisms, resulting in a shortage of well-annotated datasets for training and evaluation \citep{mckenna_recording_2019,zafar_single-cell_2020}. \emph{C. elegans} is a rare exception: by painstakingly observing and directly tracking cell divisions, Brenner, Sulston, and colleagues established its entire embryonic lineage \citep{sulston_embryonic_1983}. Beyond this model nematode, fully resolved lineage annotations are virtually nonexistent in other species, and even partial annotations—where only certain branches or cell types are labeled—are both scarce and often incomplete \citep{domcke_reference_2023}.

A central reason for this scarcity is the inherent difficulty of reliably assigning lineage labels in single-cell data. There are multiple biological barriers. First, the majority of cell division and differentiation happens inside the body where, except in some special cases, observation without sacrificing the animal is impossible. Second, cells are typically around $10^{-5}$ meters (10 micrometers) in size, making both visible traits and molecular measurements technically challenging and prone to noise. Lastly, the number of cells in a typical multicellular organism is extremely large—for example, the human body comprises approximately $10^{14}$ cells—making tracking individual lineages extremely difficult.

\section{Related Works}
\paragraph{Metric Learning} Metric learning is a broad field focused on learning representations in which distances capture meaningful similarity relationships. In recent years, a subset of metric learning techniques known as contrastive learning has gained significant traction in deep learning. These methods have been applied in both supervised \citep{khosla_supervised_2021} and unsupervised \citep{he_momentum_2020} settings, with the main objective being to embed similar examples closer together while pushing dissimilar ones farther apart. In this paper, we compare two widely used metric learning objectives—Triplet loss \citep{schroff_facenet_2015} and Quadruplet loss \citep{chen_beyond_2017}—against our proposed approach.  

\paragraph{Representation Learning for scRNA-seq}
ScRNA-seq data are high-dimensional and subject to diverse technical and biological noise. Common strategies to handle this complexity often involve dimensionality reduction (e.g., PCA \citep{heumos_best_2023}), or deep generative models (e.g., scVI \citep{lopez_deep_2018}), which embed cells in a lower-dimensional space that preserves essential variability.

Recently, large-scale pretrained models have emerged in single-cell analysis. Transformer-based architectures such as Geneformer \citep{theodoris_transfer_2023} and scGPT \citep{cui_scgpt_2024} learn embeddings from massive single-cell corpora, facilitating tasks such as cell-type classification, data integration, and cross-modality predictions. Meanwhile, supervised contrastive learning on scRNA-seq has been applied to capture nuanced cellular states, offering efficient data usage and strong generalization \citep{yang_contrastive_2022, heimberg_cell_2024, zhao_sccobra_2025}.

Nevertheless, most existing approaches rely on a notion of similarity derived from overall transcriptomic profiles, which does not necessarily align with lineage relationships (as discussed in the next paragraph). Although these methods successfully map cellular states and correct for technical artifacts, they do not directly address how to exploit these representations for cell lineage reconstruction.

\paragraph{Trajectory Inference vs.\ Lineage Reconstruction}
Methods such as RNA velocity \citep{la_manno_rna_2018,bergen_generalizing_2020} and trajectory inference \citep{qiu_reversed_2017,street_slingshot_2018} reveal continuous trajectories in the molecular state space of the transcriptomes. Although these tools capture \emph{average} progression trends, they do not directly yield a cell-by-cell lineage hierarchy. Rather, they provide trajectory embeddings that broadly reflect state set evolution, rather than the lineage history of cells. Moreover, genes driving these gene trajectory embeddings are typically selected based on high global variance, potentially missing the key drivers for lineage-specific processes.

\paragraph{Cell Types vs. Cell Linages}
In standard single-cell analysis, cell types are typically inferred by grouping cells with similar transcriptomic profiles, either through unsupervised clustering \citep{heumos_best_2023} or reference-based classification methods \citep{aran_reference-based_2019,ianevski_fully-automated_2022,hu_cellmarker_2023}. While these approaches effectively capture functional similarities, they do not explicitly account for the developmental origins of each cell. In other words, two cells with nearly identical gene expression patterns may not necessarily share a recent common ancestor. In contrast, cell lineages focus on the actual historical branching process by which cells emerge and differentiate—thus requiring methods that go beyond mere transcriptomic similarity to capture genealogical relationships.

\paragraph{Phylogenetic Inference} 
Reconstructing the lineage history of species or cells is a phylogenetic inference problem. The parameter space is vast and inherently complex due to the combination of discrete topologies and continuous branch lengths, making the problem NP-hard in most formulations \citep{felsenstein_inferring_2003}. Despite this, various data assumptions and heuristic-based reconstruction algorithms have been developed, achieving reasonable performance on typical phylogenetic datasets. The most widely used approaches infer trees using maximum parsimony, maximum likelihood, Bayesian inference, or distance-based methods \citep{jones_inference_2020, gong_single_2022}. 

Our method falls within the category of distance-based lineage reconstruction, where the goal is to fit the data to its closest tree metric. Finding the optimal tree metric is NP-hard under $\ell_1$, $\ell_2$, and $\ell_\infty$ norms for unrooted trees, but significant progress has been made in improving approximation algorithms \citep{ailon_fitting_2005}. For instance, \cite{de_soete_least_1983} proposed a greedy approach by directly using gradient descent to find the closest additive distance matrix. However, instead of explicitly optimizing pairwise distances, our approach learns an embedding function that maps data points into a space where the geometric properties facilitate tree inference.

Notably, the most recent work \citep{schluter_integrating_2025} investigates how to identify phylogenetically informative features under a fully supervised setting using permutation approaches. In our study, we conduct similar permutation experiments to assess the feasibility of our approach, but with the added objective of not only characterizing the features but also directly reconstructing the tree.

\section{Problem Formulation: Cell Lineage Reconstruction}\label{sec:problem_formulation}
\subsection{Terminology}\label{sec:terminology}
Tree graph algorithms in biology originate from the mathematical systematics field dealing with lineage reconstruction of whole organisms, represented by species or taxons, while the cell lineage reconstruction problem arises from the field of molecular and developmental biology. These two different areas have distinct but also overlapping terminology \citep{kim_tutorial_1999,clevers_what_2017,zeng_what_2022,domcke_reference_2023,rafelski_establishing_2024}. Here,we establish following definitions to avoid confusion. 

\paragraph{Single-Cell Biology}
\begin{itemize}
    \item Cell state: Characterization of a cell’s molecular phenotype. This phenotype typically varies with time and space, comprised of measurements of gene expression, metabolism, and other functional properties.
    

    \item Cell lineage: The sequential path of cell divisions that traces a given cell’s ancestry back to the zygote. Some times, cell lineage is used to refer to parts of the path. For example, ``neuronal lineage'' might refer to the parts of the path that distictly lead to cell groups identified as neuronal cell type (see next).
    
    \item Cell type: A classification of cells based on cell phenotype and sometimes also cell lineages. In general, cells of the same type typically share functional and structural characteristics but does not necessarily imply lineage relationships. However, in some biological cases, a particular cell type might be established by its particular lineage relationship rather than just the cell phenotype.
    
\end{itemize}

\paragraph{Phylogenetics and Evolutionary Biology}
\begin{itemize}
    \item Phylogeny: A tree-structured graph describing the evolutionary history and relationships among a set of biological entities. In Systematics, these entities are typically taxons (see next). In particular, a phylogeny is typically a leaf-labeled tree graph, in the sense that only the leaves of the tree are measured and named entities and interior vertices are hypothetical unnamed ancestors.
    
    When the entities are cells, a phylogeny can be seen as a \emph{cell lineage tree}.
    \item Cell lineage tree: A specialized tree-structured graph where each node represent a cell or a group of cells, and edges represent cell division events. The root corresponds to the common ancestral cell for all cells in the tree. Ultimately, all cells lead to the fertilized egg as the root.  Branches typically depict temporal progression leading to changes in cell state termed differentiation. Sometimes, branches might represent a collection of cell divisions.
    
    \item Vertices: The nodes in a lineage tree. An \emph{internal vertex} usually corresponds to a cell division point, producing two (or more) daughter cells. A \emph{leaf vertex} (or \emph{leaf})  typically represents a terminally differentiated cell. However, if the cells are sampled from middle of the cell replication process, the leaf vertex may represent transient cell(s).
    
    \item Taxon (plural: taxa): \emph{taxon} is an abstract unit in Systematics, referring to a distinct group of organisms that has been annotated by an expert for the purposes of classification. A canonical example is a species or an isolated population. Here we use it to denote the biological entity represented by a vertex of the cell lineage tree. This entity might be a single cell or it might be a class of cells denoted as a cell type.
    
    \item Clade: A subset of a lineage tree that includes a common ancestor and all its descendant leaves. 
    
\end{itemize}

\subsection{Distance-Based Lineage Reconstruction}
Given a dissimilarity matrix $D$, the goal of a distance-based approach is to solve the following optimization problem \citep{felsenstein_inferring_2003}:
\begin{equation}\label{eq:distance_based_formulation}
    \min_{T \in \mathcal{T}} \|D - D_T\|_2, 
\end{equation}
where $\mathcal{T}$ denotes the space of all tree metrics, also called additive distance matrices. This problem is known to be NP-hard in the general case \citep{kim_tutorial_1999}. However, when $D$ is close to a perfect tree metric, there exist efficient heuristic algorithms that can recover the exact optimal tree. For instance, a well-known sufficient condition for the Neighbor-Joining (NJ) algorithm \citep{saitou_neighbor-joining_1987} to reconstruct the optimal tree is:
\begin{equation*}
    \|d' - d_{T}\|_\infty \leq x^*/2,
\end{equation*}
where $x^*$ denotes the shortest edge length in $T$ \citep{atteson_performance_1999}.  

A key motivation for formulating the problem in Eq.~\ref{eq:distance_based_formulation} arises when the leaf data are generated by a tree-structured Markov process. Specifically, for phenotypic data, we suppose each vertex $v$ in a tree $T$ has a random state $\mathbf{x}_v \in \mathbb{R}^d$ generated by an independent-increments process (e.g., Brownian motion). Under standard assumptions (zero-mean Gaussian increments, independence along edges, etc.), the \emph{expected} squared distance between any two leaves $i$ and $j$ is \emph{additive} in the path length on $T$.

\begin{lemma}[Additivity of Expected Distances, Informal]\label{lemma:informal}
\emph{If $\{\mathbf{x}_i\}$ arise from a continuous-time Markov process on a tree $T$ with independent Gaussian increments along each edge (e.g., Brownian motion), then for any two leaves $i$ and $j$,}
\[
    \mathbb{E}\bigl[\|\mathbf{x}_i - \mathbf{x}_j\|^2\bigr] 
    \;=\; D_{T}(i, j),
\]
\emph{where $D_{T}(i, j)$ is an additive (tree) distance reflecting the unique path between $i$ and $j$.}
\end{lemma}
A proof of this lemma, including detailed assumptions and covariance structure, is provided in Appendix~\ref{appendix:lemma-proof}. In essence, these conditions ensure that the expected squared deviation contributed by each edge along the path from $i$ to $j$ sums linearly, inducing an additive tree metric on the leaves.

Despite this theoretical grounding, in reality, tree reconstruction using additive distance principles can be challenging. Phenotypic data often violate the idealized Markov assumptions due to non-heritable effects, measurement noise, and convergent expression. These violations of model assumptions can be more problematic than the challenges of fitting a model-optimal tree. Standard sequence-based phylogenetic data are typically preprocessed and curated by experts, whereas molecular phenotypic data lack principles for preprocessing. Thus, fitting a tree to the observed distance matrix $D$, is likely to lead to a large deviation from the ground truth tree. 

Instead of relying solely on tree-fitting methods, we propose learning an \emph{embedding function} $f: \mathbb{R}^p \rightarrow \mathbb{R}^d$ that maps each phenotype vector $\mathbf{x}\in\mathbb{R}^p$ into a (lower) $d$-dimensional representation $\mathbf{z} = f(\mathbf{x})$. Our goal is for pairwise distances among embedded points $\{\mathbf{z}_i\}$ to approximate an additive tree metric, thus extracting hierarchical signals. Formally, we can write:
\begin{equation}\label{eq:embedding_objective}
    \min_{f, T} \|D(f(\bm x)) - D_T\|_2^2 + \lambda \Omega(f),
\end{equation}
where $\Omega(f)$ is a regularization term that prevents overfitting, weighted by $\lambda$. 

When $T$ is unknown, we jointly search for both the tree topology $T$ and the function $f$. When $T$ is known, Eq.~\ref{eq:embedding_objective} reduces to a supervised metric learning problem. However, in many real biological settings, \emph{the full tree is not known} but biologists may have coarse lineage information (e.g., grouping data into a few large clades, while the finer-grained structure within each clade remains unknown). Therefore, rather than formulating the problem over the full joint pair-wise relationship of the leaves, we find it more useful to decompose the problem into \emph{local constraints} on subsets of the leaves consisting of \emph{quartets}, which capture the additive distance property more flexibly. For instance, given a clade-level prior, some quartet subsets will be ``known'' in the sense that we know branching relationships (also termed ``tree topology''), while others will not be known (Fig.~\ref{fig:quartets_parition}). We can formulate a strategy to only learn from the known quartets,

\begin{equation}
\label{eq:quartet_objective}
\min_{f, T}
\sum_{q \in \mathcal{Q}}
\mathcal L( D\bigl(f(x_q)\bigr), D_{T}\bigl(x_q\bigr))
+
\lambda \Omega(f),
\end{equation}
where $Q$ is the set of known quartets.

Phylogenetic algorithms that operate on quartets, such as quartet puzzling \citep{strimmer_quartet_1996}, have shown strong performance in the face of moderate noise or partial data. Additional theoretical and empirical motivations for focusing on quartets—particularly their robustness to local errors and compatibility with partial supervision—are discussed in Appendix~\ref{appendix:quartet_motivation}. We formalize our quartet-based approach in Section~\ref{sec:proposed_approach} and show how it accommodates supervised, weakly supervised, and unsupervised regimes.

\subsection{Scenarios}\label{supervision_scenarios}

\paragraph{Supervised setting.} We assume that we have the \emph{true} additive distances of the edge-weighted tree for every pair of vertices. Our model is trained to ensure that the embedding distance $\|\mathbf{z}_i - \mathbf{z}_j\|$ (where $\mathbf{z}_i = f(\mathbf{x}_i)$) between phenotype vectors $\mathbf{x}_i$ and $\mathbf{x}_j$ is proportional to their known tree distances. At test time, given a new set of phenotype vectors $\{\mathbf{x}_{\text{new}}\}$ from the same ground-truth lineage tree, we embed them using $f$ to obtain $\{\mathbf{z_{\text{new}}}\}$ and reconstruct their lineage relationships based on the test set embeddings.

\paragraph{Weakly Supervised Setting.} In most biological contexts, the full ground-truth lineage is unavailable. Instead, only coarser or partial annotations might be provided, and our goal is to leverage the limited data to reconstruct the underlying lineage structure (see Appendix \ref{app:prior_info} for details). We highlight two practical cases:
\begin{itemize}
    \item \textit{High-level Partitioning Setting.} 
    We know that the population of cells can be divided into a few large clades based on prior biological knowledge (e.g., tissue origin or cell fate categories). However, the finer-scale relationships \emph{within} each clade are unknown.     
    \item \textit{Partially Leaf-labeled Setting.} 
    We have access to ground-truth lineage relationships for only a subset of cells, often obtained via lineage-tracing techniques (e.g., CRISPR-Cas9-based barcoding). The remaining cells are unlabeled. 
 \end{itemize}
\textbf{Unsupervised Setting.} No explicit lineage information is available at all. The model relies solely on the raw phenotypic data and data-driven estimation of desired metric properties in the embedding space. This scenario is the most challenging because the algorithm must disentangle lineage-related signals from confounding variation in the data without any direct lineage cues.

\subsection{Constructing the Tree}
After learning the embeddings $\{\bm{z}_i\}$ in which pairwise distances approximate the true tree distances, the next step is to build the lineage tree $\hat{{T}}$. We use the NJ algorithm by default which runs in time $O(n^3)$ for $n$ leaves. This is typically manageable for moderate dataset sizes. More discussion about phylogeny reconstruction is in \S\ref{app:prelimnaries_of_phylogeny}.

\subsection{Evaluation Metrics} \label{evaluation_metrics}
We evaluate the reconstructed tree $\hat{{T}}$ against the ground-truth tree ${T}$ using the following metrics. Additional details can be found in \S\ref{app:evaluation}.  

\paragraph{Robinson–Foulds Distance (RF).}  
The Robinson–Foulds (RF) \citep{robinson_comparison_1981} distance quantifies the topological difference between two leaf-labeled unrooted trees by comparing their sets of partitions, where each partition corresponds to a bipartition of the leaves induced by an internal edge. The RF distance counts the number of partitions that differ between the inferred and true trees. We report the normalized RF distance, where 0 indicates identical tree topologies and 1 implies that no partitions are shared between the two trees.  

\paragraph{Quartet Distance (QD).}  
Although widely used in the literature, the RF distance can be insensitive to finer-grained subtree structures \citep{bocker_generalized_2013}. To complement RF distance, we also report the Quartet Distance (QD) \citep{bryant_computing_2000}. Given any four leaf vertices (a quartet), there are three possible unrooted tree topologies describing their relationships. The quartet distance is computed as the fraction of quartets that are resolved differently between the inferred and true trees. For large trees, we approximate the quartet distance by randomly sampling a subset of leaf quartets.  

\section{Proposed Framework: \textit{CellTreeQM}}\label{sec:proposed_approach}
We present \textit{CellTreeQM}, (Cell Tree Quartet Metric), a framework designed to reconstruct lineage relationships from phenotypic data. As introduced in \S\ref{sec:problem_formulation}, our main objective is to learn an embedding function based on the assumption of additive pairwise distances. In this section, we introduce our approach to consider quartets of leaves and learn an embedding that optimizes tree-metric properties of (known) quartets in the latent space.


\subsection{Additivity Loss via the Four-Point Condition}
The additivity loss is derived from a classic property of the additive distance matrix known as the  \textit{four-point condition} (Theorem \ref{four_point_condition}).
Denote a quartet of leaf vertices ${A,B,C,D}$ and define the three distance sums:
\begin{align*}
    S_1 =& D(\bm z_A, \bm z_B) + D(\bm z_C, \bm z_D),\\
    S_2 =& D(\bm z_A, \bm z_C) + D(\bm z_B, \bm z_D),\\
    S_3 =& D(\bm z_A, \bm z_D) + D(\bm z_B, \bm z_C).
\end{align*}\label{dist_sums}
The Four-Point Condition suggests that if $D$ are additive tree distances, two of these sums match exactly and both exceed the remaining sum. In addition, the ordering of the $S_1$, $S_2$, $S_3$ quantities defines a unique unrooted tree topology (Fig.~\ref{fig:quartet_intuition}a-c). Considering relaxation where $D(,)$ are not additive, we define terms that measure deviation from additivity. Assume we have an expected ordering $S_1 \ge S_2 \ge S_3$, where the ordering is either from supervised information or based on observed magnitudes. We define:
\begin{itemize}
    \item $\mathcal{L}_{\text{close}}$ measures the gap between the top two sums and as a loss function component encourages the two terms to be equal, 
    \[\mathcal{L}_{\text{close}} = |S_1 - S_2|.\]
    \item $\mathcal{L}_{\text{push}}$ enforces a margin so that the smallest sum ($S_3$) is sufficiently less than the average of the top two,
    \[\mathcal{L}_{\text{push}} = \Bigl[S_{3} -  \frac{S_{1} + S_{2}}{2} + m_0\Bigr]_+,\]
    where $m_0 > 0$ is a margin hyperparameter.
\end{itemize}

We combine them into a single quartet loss: $\mathcal{L}_{\text{quartet}} = \mathcal{L}_{\text{close}} + \mathcal{L}_{\text{push}}$, then average over all (or a sampled subset of) quartets $Q$ in each training batch: 
\begin{equation}\label{eq:quartet_loss}   
\mathcal{L}_{\text{addivity}} = \frac{1}{Q} \sum_{Q} \mathcal{L}_{\text{quartet}} = \frac{1}{Q} \sum_{Q} \mathcal{L}_{\text{close}} + \mathcal{L}_{\text{push}}
\end{equation}

Figure~\ref{fig:quartet_intuition} illustrates the geometric intuition of the loss. When the distance matrix is perfectly additive, the four points form an unrooted tree, such as in Figure~\ref{fig:quartet_intuition}a where $S_2 = D(A,C)+D(B,D)$ and $S_3 = D(A,D)+D(B,C)$ and $S_2 = S_3$. When additivity is violated, this balance is disrupted, and the structure can be imagined as a ``box'' with an extra edge (Figure~\ref{fig:quartet_intuition}d). The $\mathcal{L}_{\text{close}}$ term encourages the top two distance sums to become more similar, thereby reducing the imbalance that creates the box-like distortion. In effect, it ensures that the box is not ``fat''. Meanwhile, the $\mathcal{L}_{\text{push}}$ term increases the gap between the smallest sum and the average of the top two, effectively ``widening the bridge.'' This widening enhances the tree model’s robustness to noise and distortions from the ideal tree structure.
\begin{figure}[ht]
    \centering
    \includegraphics[width=0.5\linewidth]{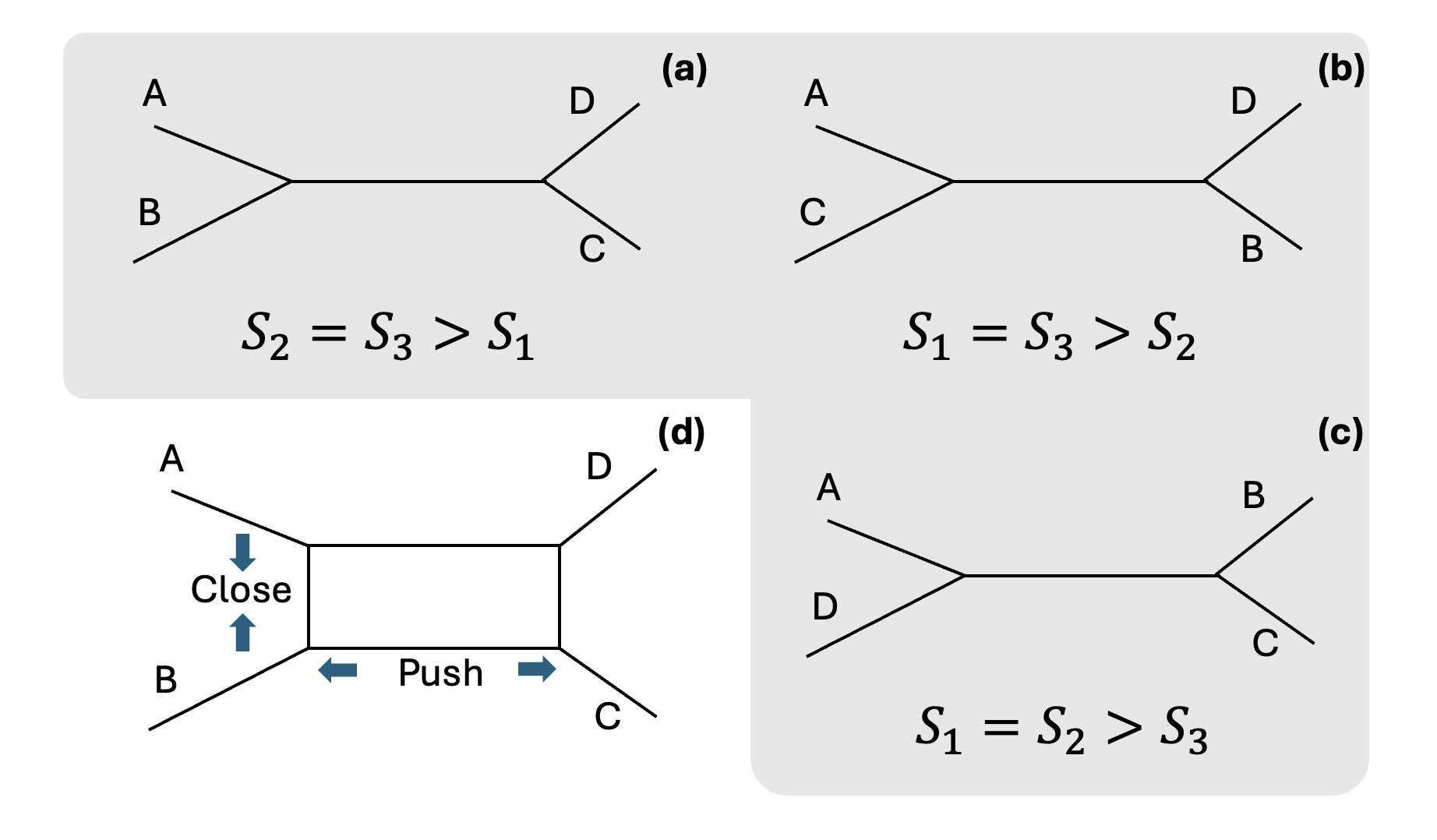}
    \caption{Geometric intuition of the quartet loss. (a-c) show the three possible unrooted tree topologies for a quartet $\{A,B,C,D\}$, (d) depicts the ``box-like'' distortion that arises if the distances are not perfectly additive.}
    \label{fig:quartet_intuition}
\end{figure}
\subsection{Regularization: Deviation Loss}
To prevent the learned embedding from drifting too far away relative to the original data structure, we introduce a deviation loss $\Omega$. This loss penalizes large discrepancies between the original distance matrix $D(X)$ and the induced distance matrix $D(f(X))$ in the latent space:
\setlength{\abovedisplayskip}{4pt}
\setlength{\belowdisplayskip}{4pt}
\begin{equation*} 
\Omega(f, X) = \frac{1}{N} \| \mathcal{D}(f(X)) - \mathcal{D}(X) \|_F^2 
\end{equation*}
where $\| \cdot \|_F^2$ denotes the squared Frobenius norm and $N$ is the number of points. This ensures that the latent space remains faithful to the measured phenotypic similarities while preventing the learning process from degenerating in scale.

\subsection{Feature Gating}
High-dimensional phenotypic datasets (e.g., scRNA-seq) usually include numerous features that do not reflect lineage-related variation. To address this, we introduce a \textit{feature gating} module, which adaptively modulates the contribution of each input feature. By emphasizing lineage-relevant signals and down-weighting confounding or redundant attributes, the gating module can improve downstream tree reconstruction. To learn an effective feature mask, we use a Gumbel-Softmax \citep{gumbel_statistical_1954} for discrete gating decisions. 
\[
\bm g \;=\; \text{GumbelSoftmax}\bigl(\text{MLP}(\bm{E}), \tau, \text{hard}=\text{True}\bigr),
\]
where $\bm{E}$ is an embedding for each feature, $\text{MLP}$ projects this embedding into 2-dimensional logits $\ell_i$ (one for ``off'' and one for ``on''), and $\tau$ is a temperature parameter. In the hard-sampling regime ($\text{hard} = \text{True}$), the gating weights become nearly binary ($0 \text{ or } 1$), effectively pruning or retaining each feature. Applying the gating is a simple element-wise product: $\tilde{\bm{x}}_i = \bm{x}_i \cdot g_i$.

\paragraph{Sparsity Penalty.}
To encourage gating out superfluous features, we introduce a \textit{sparsity penalty}:
\begin{equation*}
\Omega_{\text{gates}}^{(\text{sparsity})} = \lambda_{\text{spar}} \frac{\sum_{i} g_i}{\text{input dim}},
\end{equation*}
where $\lambda_{\text{spar}}$ is a tunable coefficient. Lower total activation ($\sum_i g_i$) results in a smaller penalty, thus encouraging the network to keep fewer features ``on.'' This penalty helps reduce noise and highlight relevant signals for lineage reconstruction.

\paragraph{Integrated Objective.}
Combining Gumbel gating with our core metric-learning objectives, we arrive at the overall optimization function:

\begin{equation}
\min_{f, g}
\Bigl[
\mathcal{L}_{\text{additivity}}(f \circ g, X) +
\lambda \,\Omega_(f \circ g,X) +
\Omega_{\text{gates}}^{(\text{sparsity})}
\Bigr].
\end{equation}

This integrated framework allows us to \emph{jointly} learn which features are most informative for lineage reconstruction and how to embed them to satisfy the quartet-based additivity constraints. 

\begin{figure*}
    \centering
    \includegraphics[width=\linewidth]{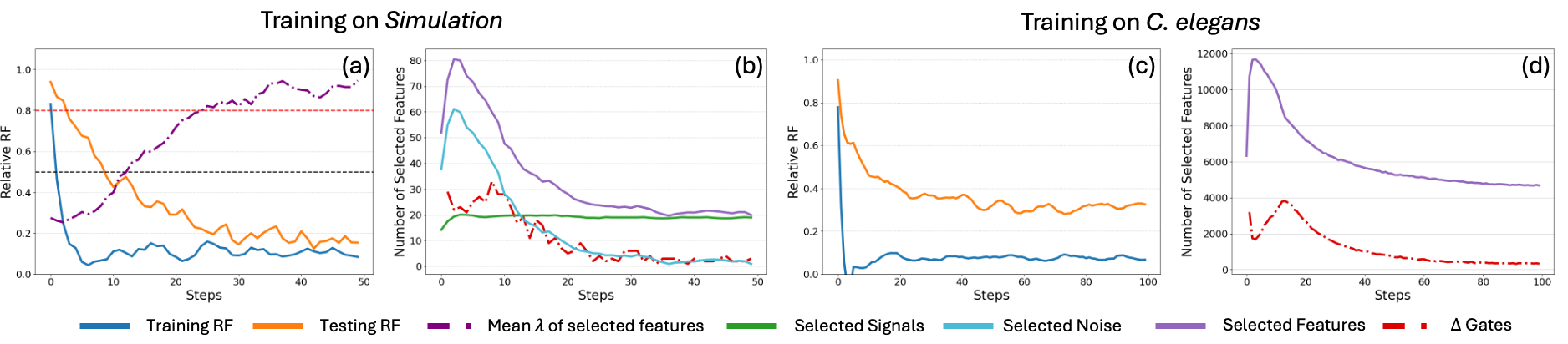}
    \caption{Supervised training dynamics on the simulation and \textit{C. elegans} Small dataset.
    \textbf{(a)} For the simulated dataset, the dashed pink line represents the RF distance of the tree reconstructed from the raw data, while the dashed black line indicates the RF distance using only the ``signal’’ features. The dashed purple line shows the average Pagel’s $\lambda$ of the selected features, serving as a benchmark for phylogenetic signal strength.
    \textbf{(b)} The feature selection process gradually excludes noise features (blue), while the number of selected signal features (green) stabilizes. The difference in gating values (dashed red) decreases over time, indicating that feature selection is converging.
    \textbf{(c-d)} Training on the \textit{C. elegans} Small dataset: (c) RF distance and (d) the number of selected features exhibit similar trends as observed in the simulation dataset.}
    \label{fig:supervised_training}    
\end{figure*}

\subsection{Model Architecture}
Our proposed framework, \emph{CellTreeQM}, aims to learn embeddings from high-dimensional phenotypic data that facilitate phylogenetic reconstruction. To effectively learn relationships among cells, we use a sequence of Transformer encoder blocks as the backbone of the network, illustrated in Figure~\ref{fig:phylodist}. Unlike classical Transformer models, we do not include positional encodings, as our input leaves are not inherently ordered. Without positional constraints, the self-attention module focuses purely on learning meaningful relationships based on the feature similarities between cells. We compared the performance between the network with Transformer encoders as backbones and fully connected layers as the backbone under a supervised setting on two \emph{C. elegans} real datasets (see Table.\ref{tab:transformer-vs-fc} in Appendix). Networks with the attention module performed much better than the fully connected architecture.

\section{Cell Lineage Reconstruction Benchmark}
To systemically evaluate lineage-reconstruction methods, we introduce a benchmark spanning synthetic datasets and three scales of lineage-resolved real datasets. Figure \ref{fig:supervised_training} shows the training dynamics of \textit{CellTreeQM} on a synthetic dataset and a real dataset under supervised settings. \S\ref{app:dataset} provides full dataset descriptions, parameter choices, and curation procedures.

\paragraph{Lineage-Resolved \emph{C.~elegans} Dataset:}
\label{sec:celegans_benchmark}
Among model organisms, \emph{C.~elegans} is uniquely suited for benchmarking lineage reconstruction because its embryonic cell lineage is \emph{invariant} and fully \emph{resolved}. We curate three subsets of increasing size -- \emph{C.~elegans Small}, \emph{Mid}, and \emph{Large} -- containing 102, 183, and 295 leaves, respectively, using transcriptomic atlases from \citep{packer_lineage-resolved_2019,large_lineage-resolved_2024}. Each subset provides fully lineage-resolved ground truth topologies. These datasets allow us to measure reconstruction quality at multiple scales.

\paragraph{Synthetic Brownian-Motion Simulations:}
\label{sec:sims_benchmark}
We generate synthetic data by simulating a fully binary tree with random edge weights and evolving feature vectors along branches via Brownian motion. To introduce realistic complexity, we add independent Gaussian noise and ``alternative-tree'' features that follow some confounding lineages. By tuning parameters such as the number of leaves or the relative strengths of signal and noise, we obtain multiple levels of reconstruction difficulty (Fig.\ref{fig:signal_dataset}, Table~\ref{tbl:supervised_simulation}).

\section{Experiments}\label{results}
\subsection{Baselines}
We consider the following two common contrastive losses in metric learning. The details can be found in \S\ref{app:baselines}. \emph{CellTreeQM} and baselines produce learned embedding spaces from which we extract pairwise distances between leaves. 

\begin{itemize}
    \item \textbf{Triplet Loss:}~ We designate the closest pair from each quartet as anchor–positive and the farthest leaf as negative. The loss encourages $\|f(A)-f(P)\| < \|f(A)-f(N)\|$.
    \item \textbf{Quadruplet Loss:}~ Extends Triplet Loss by introducing a second negative leaf, enforcing additional pairwise margins to improve global distance structure.
\end{itemize}

\subsection{Cell Aggregation}
Recall that a cell lineage is a sequential path of cell divisions and a leaf is the terminal state of the division.\footnote{Readers may find it helpful to refer to Section \ref{sec:terminology} for a review of relevant terminology.} A taxon is a biological entity that the leaf stands for. Thus far, we have assumed that each taxon corresponds to one point of cell state. In other words, we assume the observed molecular profile of one cell can be used to represent the terminal state. In the simplest case—where all cells in the dataset come from one individual—each cell can indeed serve as a leaf in the tree (as in our \textit{Synthetic Brownian-Motion Simulations}). However, in most real scRNA-seq experiments, one point of cell state alone is typically insufficient to robustly characterize the terminal state of a cell lineage due to many transient variations. Therefore, the taxon should be a set of replications of cells (e.g., from multiple animals, experimental batches) to achieve adequate statistical power. One can imagine that each taxon has a distribution of cell states with intrinsic and extrinsic variation. The observed cell states are sampled from the distributions of lineage terminals

A common approach is therefore to group cells by taxon and aggregate their identities to reduce technical noise and biological variation. Specifically, let $v$ denote a taxon, and let $S_v$ be the set of cells assigned to $v$. One may then define an aggregate molecular vector
\begin{equation*}
    \bm{x}_v = \mathrm{Aggregate}\bigl(\bm{x}^i : i \in S_v\bigr),
\end{equation*}
where $\bm{x}^i$ is the cell identity (i.e. molecular profile) of a cell $i$ and \(\mathrm{Aggregate}\) could be as simple as the arithmetic mean, which is a standard practice in many single-cell lineage-resolved atlases \citep{packer_lineage-resolved_2019,large_lineage-resolved_2024}.

In the \textit{Lineage-Resolved C. elegans Dataset}, the authors provide the map $\{S_v\}$ from cells to taxon $v$. For datasets without direct annotation, one can employ methods such as local pooling \citep{l_lun_pooling_2016}, which merges nearby cells based on transcriptomic similarity. These grouped or ``pseudo-bulk'' profiles can then serve as robust taxon-level representations for downstream lineage reconstruction.

\begin{table}[ht]
\centering
\caption{\textbf{Supervised results on the \emph{C. elegans Small} dataset.} 
Direct reconstruction on raw data yields RF = 0.923 and QDist = 0.554. Suffix ``-G'' denotes feature gating, and ``-p'' indicates label permutation. The reported values are means across three runs, with standard deviations in parentheses. $\Delta$\%RF and $\Delta$\%DQ measure the relative improvement in RF or QD on the test dataset over direct raw-data reconstruction.}
\begin{tabular}{@{}lcccc@{}}
\toprule
& \textbf{Train RF \(\downarrow\)}       & \textbf{Test RF \(\downarrow\)} & \textbf{$\Delta$\%RF \(\uparrow\)}     & \textbf{$\Delta$\%QDist \(\uparrow\)}  \\ \hline
CellTreeQM   & \textbf{0.000} (0.00) & \textbf{0.286 (0.05)}                      & \textbf{0.690} (0.05) & \textbf{0.867} (0.02) \\
CellTreeQM-G & \textbf{0.000} (0.00) & \textbf{0.226 (0.02)}                      & \textbf{0.757} (0.03) & \textbf{0.848} (0.01) \\
CellTreeQM-p & 0.013 (0.00)                           & 0.566 (0.04)                               & 0.434 (0.04)                           & 0.691 (0.05)                           \\ \hline
Triplet      & 0.519 (0.03)                           & 0.741 (0.01)                               & 0.179 (0.01)                           & 0.637 (0.02)                           \\
Triplet-G    & 0.545 (0.03)                           & 0.724 (0.00)                               & 0.203 (0.03)                           & 0.631 (0.02)                           \\
Triplet-p    & 0.677 (0.06)                           & 0.963 (0.01)                               & 0.037 (0.01)                           & 0.385 (0.06)                           \\ \hline
Quadruplet   & 0.057 (0.00)                           & 0.492 (0.03)                               & 0.454 (0.02)                           & 0.784 (0.01)                           \\
Quadruplet-G & 0.061 (0.00)                           & 0.471 (0.04)                               & 0.484 (0.04)                           & 0.791 (0.01)                           \\
Quadruplet-p & 0.118 (0.00)                           & 0.848 (0.01)                               & 0.149 (0.01)                           & 0.538 (0.02)                           \\ \bottomrule
\end{tabular}
\label{tbl:supervsied_celegans_small}
\end{table}

\subsection{Supervised setting}\label{se:weakly_supervised_setting_all_quartets}
Similar to \citep{schluter_integrating_2025}, we conduct experiments in a supervised scenario where all quartet topologies are assumed known and we train a model that enforces such topology in the embedding space. While this setting has little empirical significance, its results justify the soundness of our approach.
We use $\Delta\%\text{RF}$ and $\Delta\%\text{QD}$ (see Eq.\ref{delta_percentage}) to measure relative improvement compared to tree estimates on original raw data and Training RF to assess how well the learned embedding preserves the global tree structure implied by the quartets.

Table~\ref{tbl:supervsied_celegans_small} shows results on the \emph{C. elegans Small} dataset, whose direct reconstruction from the raw data yields high RF (0.923), suggesting almost no lineage information. We train each method both with and without feature gating (\emph{-G}). Additionally, to test whether our method is overfit to any given topology, we randomly permute (\emph{-p}) the labels of the leaves of ${T}$ to obtain an arbitrary tree ${T}'$ for training. Several key trends emerge:

\paragraph{CellTreeQM consistently outperforms contrastive losses.} By explicitly enforcing the four-point condition across quartets, \textit{CellTreeQM} consistently achieves higher improvement over raw data ($\Delta\%$RF and $\Delta\%$QD) compared to the triplet and quadruplet baselines. Moreover, triplet loss can recover moderate levels of quartet topology, but its latent space is not aligned on the global tree topology as represented by a large Train RF and low $\Delta\%$RF. This is reasonable because both of the contrastive losses do not explicitly enforce tree topology in the embedding space. Furthermore, feature gating (\emph{-G}) typically yields additional gains, especially in RF distance, likely by pruning out non-heritable or noisy gene expressions. 

\paragraph{Permutation experiments validate quartet-based fitting.} Training \emph{CellTreeQM} on a label permuted tree ${T}'$ still embeds the data with moderate fidelity ($\Delta\%$RF), reflecting information in an unlabeled tree, but less accurately than with the true tree. In contrast, triplet and quadruplet models struggle to align a random tree with the data. Similar trends hold for the \emph{C. elegans Mid} and \textit{Large} datasets, where triplet/quadruplet losses improve over raw data but still lag behind \emph{CellTreeQM} (see Table~\ref{tbl:supervised_full}). These results echo the recent findings in \citep{schluter_integrating_2025} showing that \emph{true lineage topologies} yield lower training losses and more generalizable embeddings, thus confirming there is ample lineage signal in the gene expression data---our method is not merely overfitting. 

Figure~\ref{fig:supervised_tsne_all} (2D t-SNE) shows that \textit{CellTreeQM} better reconstructs the hierarchical structure than Triplet or Quadruplet. Figure~\ref{fig:supervised_circle_tree} further illustrates the reconstructed trees, where \textit{CellTreeQM} yields lineage-consistent subtrees.

\subsection{Weakly Supervised}\label{sec:weakly_supervised_setting_partial_quartets}
In most biological applications, we lack access to the complete ground-truth lineage. To represent realistic conditions, as introduced in \S\ref{supervision_scenarios}, we formulated two scenarios where only partial knowledge of the underlying tree is available. If we consider all possible quartets, partial knowledge will allow us to identify the topology of some quartets (termed ``known quartets'') while leaving others unresolved (``unknown quartets''). In the following subsections, we demonstrate that trained on the known quartets, \emph{CellTreeQM} is able to generalize to the unknown quartets and effectively recovers large portions of the global tree structure under these reduced supervision regimes.

\paragraph{High-level Partitioning Setting.} 
We define a node’s \emph{level} as the number of branching steps from the root (e.g., the root is level 0, its immediate children are level 1, and so on). In this setting, we assume that at each level, the leaves are partitioned into a fixed number of clades (see Fig.~\ref{fig:intro-framework}, left). 
So each leaf is assigned to exactly one high-level clade, but the tree structure among clades and \emph{within} each clade remains unknown. 
Consequently, we can infer the structure of some of the quartets based on the partition prior. As demonstrated in Fig.~\ref{fig:quartets_parition}, any quartet with a pair of leaves drawn from the same clade has a unique unrooted topology among the three possible configurations.

\begin{table}[ht]
\centering
\footnotesize
\caption{\textbf{Weakly supervised results on \emph{C.~elegans Small} under different partition levels.} 
\(\mathrm{K\text{-}QD}\) and \(\mathrm{U\text{-}QD}\) 
are quartet distances on the \emph{known} and \emph{unknown} quartets, respectively.
The reported values are means across five runs, with
standard deviations in parentheses.}
\label{tbl:high_level_partition_celegans_small}
\begin{tabular}{@{}lrrrrrr@{}}
\toprule
\textbf{Method}             & \multicolumn{1}{c}{\textbf{RF\(\downarrow\)}} & \multicolumn{1}{c}{\textbf{$\Delta$\%RF\(\uparrow\)}} & \multicolumn{1}{c}{\textbf{QD\(\downarrow\)}} & \multicolumn{1}{c}{\textbf{$\Delta$\%QD\(\uparrow\)}} & \multicolumn{1}{c}{\textbf{$\Delta$\%K-QD\(\uparrow\)}} & \multicolumn{1}{c}{\textbf{$\Delta$\%U-QD\(\uparrow\)}} \\ \midrule
\textbf{Partition Level: 3} & \multicolumn{1}{l}{}   & \multicolumn{1}{c}{}                                  & \multicolumn{1}{l}{}   & \multicolumn{1}{c}{}                                  & \multicolumn{1}{c}{}                                    & \multicolumn{1}{c}{}                                    \\
CellTreeQM                  & \textbf{0.538(0.01)}   & \textbf{0.403(0.01)}                                  & \textbf{0.156(0.00)}   & \textbf{0.715(0.01)}                                  & \textbf{0.999(0.00)}                                    & \textbf{0.247(0.02)}                                    \\
Triplet                     & 0.702(0.01)            & 0.232(0.01)                                           & 0.264(0.01)            & 0.518(0.01)                                           & 0.722(0.02)                                             & 0.183(0.01)                                    \\
Quadruplet                  & 0.854(0.01)            & 0.081(0.02)                                           & 0.241(0.01)            & 0.560(0.02)                                           & 0.971(0.01)                                             & -0.117(0.04)                                 \\ \midrule
\textbf{Partition Level: 2} & \textbf{}              & \textbf{}                                             & \textbf{}              & \textbf{}                                             & \textbf{}                                               & \textbf{}                                               \\
CellTreeQM                  & \textbf{0.657(0.01)}   & \textbf{0.286(0.02)}                                  & \textbf{0.129(0.01)}   & \textbf{0.764(0.03)}                                  & \textbf{0.998(0.00)}                                    & \textbf{0.511(0.06)}                                    \\
Triplet                     & 0.773(0.02)            & 0.168(0.02)                                           & 0.326(0.01)            & 0.405(0.01)                                           & 0.726(0.01)                                             & 0.058(0.02)                                             \\
Quadruplet                  & 0.871(0.01)            & 0.052(0.02)                                           & 0.286(0.00)            & 0.477(0.00)                                           & 0.995(0.00)                                             & -0.084(0.01)                                            \\ \midrule
\textbf{Partition Level: 1} & \multicolumn{1}{l}{}   & \multicolumn{1}{c}{}                                  & \multicolumn{1}{l}{}   & \multicolumn{1}{c}{}                                  & \multicolumn{1}{c}{}                                    & \multicolumn{1}{c}{}                                    \\
CellTreeQM                  & \textbf{0.773(0.02)}   & \textbf{0.164(0.01)}                                  & \textbf{0.209(0.03)}   & \textbf{0.619(0.05)}                                  & \textbf{0.998(0.00)}                                    & \textbf{0.467(0.07)}                                    \\
Triplet                     & 0.813(0.02)            & 0.120(0.02)                                           & 0.352(0.01)            & 0.357(0.02)                                           & 0.831(0.01)                                             & 0.168(0.02)                                             \\
Quadruplet                  & 0.859(0.01)            & 0.061(0.00)                                           & 0.423(0.01)            & 0.228(0.01)                                           & 0.672(0.02)                                             & 0.051(0.01)                                    \\ \bottomrule
\end{tabular}
\end{table}
We compute the additivity loss \emph{only} on the known quartets. We then evaluate on the full set of leaves, measuring the overall RF and QD against the ground-truth tree. Additionally, we separate the results of quartets into known versus unknown and measure QD for each group (K-QD for known quartets, U-QD for unknown). The performance on the U-QD shows the ability of learning on known quartets to generalize to the unknown quartets.

Table~\ref{tbl:high_level_partition_celegans_small} summarizes performance for varying levels on the \emph{C.~elegans Small} dataset. In all cases, \emph{CellTreeQM} outperforms contrastive baselines. Notably, $\Delta\%\mathrm{K\text{-}QD}$ for \textit{CellTreeQM} is near 1.0 at all levels, indicating near-perfect recovery of the \emph{supervised} quartets. More importantly, $\Delta\%\mathrm{U\text{-}QD}$ remains substantial (up to 0.51), demonstrating that the learned embedding preserves tree structure for \emph{unsupervised} quartets as well. The contrastive baselines generally struggle to maintain strong global alignment in these partially supervised settings, especially \textit{Quadruplet}, whose $\Delta\%\mathrm{U\text{-}QD}$ can even dip below zero (indicating worse performance than raw data).

Interestingly, Partition Level~2 yields the best $\mathrm{QD} (0.129)$ and highest $\Delta\%\mathrm{QD} (0.764)$, meaning it most accurately resolves local quartet structure. In line with the theoretical result in \S\ref{app:high_level_partition}, having \emph{four} broad clades maximizes the fraction of quartets whose topology can be directly inferred, thereby aiding the learning process. Partition Level~3, in contrast, achieves the lowest $\mathrm{RF}$ (best global structure) but slightly higher $\mathrm{QD}$. This indicates that going one level deeper captures better “big-picture” topology (lower $\mathrm{RF}$) at a small cost to local quartet fidelity ($\mathrm{QD}$).

\begin{table}[ht]
\centering
\scriptsize
\caption{\textbf{Weakly supervised results on \emph{C.~elegans Small} across different known fractions.} 
\(\mathrm{K\text{-}QD}\), \(\mathrm{P\text{-}QD}\) and \(\mathrm{U\text{-}QD}\) 
are quartet distances on the \emph{known}, \emph{partial} and \emph{unknown} quartets, respectively.
The reported values are means across ten runs, with standard deviations in parentheses.}
\label{tbl:partial_label_results_celegans_small}
\scalebox{0.9}{
\begin{tabular}{@{}lcccccccc@{}}
\toprule
\textbf{Method} & \textbf{Train RF\(\downarrow\)} & \textbf{RF\(\downarrow\)}                   & \textbf{$\Delta$\%RF\(\uparrow\)} & \textbf{QD\(\downarrow\)}                   & \textbf{$\Delta$\%QD\(\uparrow\)} & \textbf{$\Delta$\%K-QD\(\uparrow\)} & \textbf{$\Delta$\%P-QD\(\uparrow\)} & \textbf{$\Delta$\%U-QD\(\uparrow\)} \\ \midrule
\multicolumn{9}{l}{\textbf{Known Fraction: 0.8}}                                                                                                                                                                                                                                          \\
CellTreeQM      & \textbf{0.024(0.03)}            & \textbf{0.506(0.05)} & \textbf{0.448(0.06)}              & \textbf{0.086(0.03)} & \textbf{0.842(0.05)}              & \textbf{0.999(0.00)}                & \textbf{0.742(0.07)}                & \textbf{0.465(0.13)}                \\
Triplet         & 0.454(0.06)                     & 0.750(0.04)          & 0.175(0.05)                       & 0.149(0.02)          & 0.728(0.03)                       & 0.895(0.01)                         & 0.624(0.05)                         & 0.339(0.14)                         \\
Quadruplet      & 0.066(0.04)                     & 0.542(0.03)          & 0.403(0.03)                       & 0.112(0.03)          & 0.796(0.06)                       & 0.953(0.02)                         & 0.697(0.09)                         & 0.435(0.22)                         \\ \midrule
\multicolumn{9}{l}{\textbf{Known Fraction: 0.5}}                                                                                                                                                                                                                                          \\
CellTreeQM      & \textbf{0.012(0.01)}            & 0.827(0.04)          & 0.092(0.05)                       & \textbf{0.214(0.03)} & \textbf{0.609(0.05)}              & \textbf{0.999(0.00)}                & \textbf{0.598(0.06)}                & \textbf{0.398(0.05)}                \\
Triplet         & 0.303(0.09)                     & 0.869(0.03)          & 0.049(0.04)                       & 0.272(0.03)          & 0.505(0.05)                       & 0.879(0.02)                         & 0.493(0.05)                         & 0.304(0.06)                         \\
Quadruplet      & 0.023(0.02)                     & \textbf{0.807(0.04)}          & \textbf{0.115(0.04)}                       & 0.248(0.02)          & 0.549(0.04)                       & 0.934(0.03)                         & 0.537(0.04)                         & 0.340(0.06)                         \\ \midrule
\multicolumn{9}{l}{\textbf{Known Fraction: 0.3}}                                                                                                                                                                                                                                          \\
CellTreeQM      & \textbf{0.000(0.00)}            & 0.932(0.02)          & -0.023(0.02)                      & \textbf{0.347(0.03)} & \textbf{0.368(0.06)}              & \textbf{1.000(0.00)}                & \textbf{0.401(0.06)}                & 0.250(0.06)                         \\
Triplet         & 0.156(0.06)                     & 0.921(0.01)          & -0.001(0.02)                      & 0.353(0.02)          & 0.358(0.04)                       & 0.889(0.02)                         & 0.384(0.04)                         & \textbf{0.263(0.05)}                \\
Quadruplet      & 0.008(0.02)                     & 0.924(0.02)          & -0.020(0.03)                      & 0.352(0.02)          & 0.358(0.03)                       & 0.918(0.04)                         & 0.386(0.03)                         & 0.259(0.02)                         \\ \bottomrule
\end{tabular}
}
\end{table}

\paragraph{Partially Leaf-labeled Setting.} In this scenario, a fraction $\kappa$ of leaves (e.g., 30\%, 50\%, or 80\%) have known lineage information, while the rest remain unresolved (see Fig.~\ref{fig:intro-framework} right). Quartets are considered as \emph{Known} if all four leaves are labeled, \emph{Unknown} if no leaves are labeled, and \emph{Partial} otherwise. During training, we apply the additivity loss only to Known quartets. After training, we embed \emph{all} leaves—labeled and unlabeled—and measure the reconstructed tree quality (RF and QD), alongside three specialized QD metrics for Known (K-QD), Partial (P-QD), and Unknown (U-QD) quartets.

Table~\ref{tbl:partial_label_results_celegans_small} shows that when $\kappa = 80\%$ or even 50\% of leaves are labeled, \textit{CellTreeQM} provides strong performance, highlighting that a moderate amount of lineage information can guide an effective tree embedding. However, at only $\kappa = 30\%$ known labels, \textit{all} methods exhibit only modest gains in $\mathrm{QD}$ and even negative $\Delta\%\mathrm{RF}$. Indeed, $\kappa^4 = 0.3^4 = 0.0081$, meaning fewer than 1\% of quartets are fully labeled (see Fig.~\ref{fig:quartet_partial}). With such sparse supervision, the model has very little direct signal to guide the global tree structure.

Although \textit{Quadruplet} overfits (yielding negative $\Delta\%\mathrm{U\text{-}QD}$) in the high-level partitioning setting, it performs more competitively here—often surpassing \textit{Triplet} in $\Delta\%\mathrm{K\text{-}QD}$, $\Delta\%\mathrm{P\text{-}QD}$, and $\Delta\%\mathrm{U\text{-}QD}$. This suggests that introducing a second negative leaf (as in \textit{Quadruplet} Loss) can help capture some global structure and extrapolate to unlabeled quartets.

\subsection{Unsupervised Setting}\label{sec:unsupervised_setting}
Lastly, we present preliminary results under the unsupervised regime, where no quartet constraints or partial subtree information are given. Here, we compute empirical $S_1, S_2, S_3$ quantities in the latent space and choose the ordering that best fits the observed quartet quantities. Although the problem is more difficult—and in practice, purely unsupervised lineage inference from phenotypes alone can be error-prone—we find that \textit{CellTreeQM} can still learn a representation that partially respects tree-metric properties on simulated data. The performance on real empirical data was not significant, suggesting better strategies are needed for a completely unsupervised setting. See details in \S\ref{app:unsupervised}.

\begin{figure}[h]
    \centering
    \vspace{-10pt}
    \includegraphics[width=0.5\linewidth]{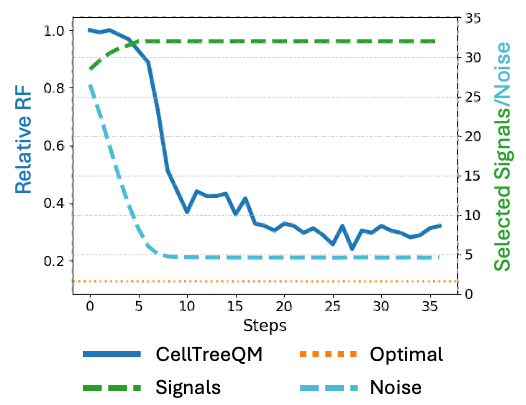}
    \caption{\textbf{Training dynamics of CellTreeQM in a purely unsupervised setting on a simulated dataset.} Optimal is the RF of reconstructed tree only based on signal features.
    }        
    \vspace{-10pt}
    \label{fig:unsupervsied_setting}
\end{figure}


\section{Discussion}
\paragraph{Summary.} We introduce \textit{CellTreeQM}, a deep learning framework for reconstructing cell lineage trees from phenotypic features via metric learning. The key idea is that the geometric properties of certain spaces are conducive to mapping to (inferring) tree-graph structures. Therefore, we formulate our problem as learning a latent space whose metric properties are optimized for quartets of tree-graphs. To our knowledge, this is the first method to cast cell lineage inference explicitly as a metric learning problem. Experimental results in supervised, weakly supervised, and unsupervised settings show that \textit{CellTreeQM} considerably improves lineage reconstruction accuracy over standard contrastive baselines (e.g., triplet and quadruplet losses), indicating that transcriptomic data contain cryptic lineage information, which can be uncovered with carefully designed metric learning models.

\paragraph{Limitations.} Despite its advantages, \textit{CellTreeQM} has several limitations. First, our loss function integrates tree-metric constraints, gating, and distortion regularization. Although the distortion term limits the deviation of the latent space geometry from the input feature space, it can impede learning when lineage signals in real data are significantly distorted. While gating helps suppress random noise in simulations, its impact for the real data is mild, suggesting that real “noise” features may be correlated with signal features. Second, we rely on NJ for final tree construction. Although NJ is widely adopted, it can fail under complex noise conditions. Future research could explore more robust tree-inference methods, such as Bayesian approaches or graph neural networks. Third, our benchmarks focus on lineage-resolved datasets with well-defined ground truth. Extending \textit{CellTreeQM} to more heterogeneous single-cell datasets (e.g., developmental atlases with partial lineage annotations) will be crucial for broader applicability. 


\paragraph{Guidance for Future Work.} The modest success of unsupervised \textit{CellTreeQM} on small, clean simulations opens the door to future improvements. Additional data-driven methods for determining optimal latent space geometry and more extensive hyperparameter tuning may further enhance performance. In real single-cell transcriptome datasets, the phenotype–lineage correlation can be weak, particularly at later stages of development; thus, stronger regularization or heuristic constraints might be needed. Nevertheless, our initial findings show that unsupervised approaches can capture coarse lineage structure without explicit topological supervision. We believe our strategy of learning metric properties of the latent space rather than directly trying to learn the tree-graph can be an effective approach.


\bibliography{references}
\bibliographystyle{plainnat}

\newpage
\appendix

\section{Preliminaries of Phylogeny}\label{app:prelimnaries_of_phylogeny}
A tree is a connected graph where every pair of nodes is connected by a unique path. Among other things, this restriction implies that in a tree there are no links ($v_i$, $v_{i'}$) that connect a node $v_i$ to itself. An additive tree is a connected undirected network where every pair of nodes is connected by a unique path \citep{sattath_additive_1977}. Since there exists only one path between any two nodes, the minimum path length distance between two nodes is equal to the length of the unique path that connects them. These distances are often referred to as path length distances or additive tree distances. 

In evolutionary biology, a phylogeny is a bifurcating tree that models the evolutionary relationships among a set of species or other biological entities. The leaves of this tree represent extant (or observed) entities, while the internal nodes represent their hypothetical common ancestors. Such trees help us understand how these entities have diverged over time.

With the advent of large-scale genomic data, the field of phylogenomics has emerged. Phylogenomics integrates phylogenetic analysis with genome-wide data, allowing for more accurate and comprehensive inferences about evolutionary history. Rather than focusing on a single gene (as in classical phylogenetics), phylogenomics considers data from multiple genes, entire genomes, or high-dimensional molecular measurements, providing a richer context for reconstructing the evolutionary relationships among species.

In this section, we introduce fundamental concepts for distance-based phylogeny reconstruction—an approach that finds an additive tree by mapping $n \times n$ distance matrix to $n \times n$ additive distance matrix \citep{kim_tutorial_1999,de_soete_least_1983}. These concepts form the theoretical basis of the learning objectives proposed in this work.

\subsection{Distance Matrices and Additivity}
Given a set of $n$ nodes, each represented by a vector $\bm{x}_v$, we can derive an $n \times n$ distance matrix $M$, where $M_{ij}$ denotes the distance between nodes $i$ and $j$. If $M$ is a valid metric, it must satisfy the following properties:
\begin{itemize}
    \item \textbf{Symmetric}: $M_{ij} = M_{ji}$ and $M_{ii} = 0$;
    \item \textbf{Triangle Inequality}: $M_{ij} + M_{jk} \ge M_{ik}$
\end{itemize}

Now, suppose ${\bm{x}_v}$ represents the leaves of a phylogeny and $w$ is the lowest common ancestor of two nodes $u$ and $v$, we expect $M_{uv} = M_{uw} + M_{wv}$. This property introduces the concept of an additive metric. A distance matrix $M$ is \textbf{additive} if there exists a phylogeny $T$ such that:

\begin{itemize}
\item Each edge $(u,v)$ in $T$ is associated with a positive edge weight $\delta_{uv}$.
\item For every pair of nodes $u, v$, the distance $M_{uv}$ equals the sum of the edge weights along the unique path from $u$ to $v$ in $T$.
\end{itemize}

For distance matrices with fewer than four leaves, $M$ must be additive. However, for matrices with four or more leaves, $M$ may not be additive. Buneman’s 4-point condition provides a criterion to determine additivity:
\begin{theorem}[4-Point Condition]\label{four_point_condition}
A distance matrix $M$ is additive if and only if the following holds: For any distinct leaves $i, j, k, l$, we can label them such that:
\[
M_{ik} + M_{jl} = M_{il} + M_{jk} \geq M_{ij} + M_{kl}.
\]
\end{theorem}
This condition ensures that, among any four leaves, the two largest sums of pairwise distances are equal. This property is fundamental for the additivity, guaranteeing that $M$ corresponds to an additive tree. When $M$ satisfies this condition, the additive tree can be uniquely reconstructed in $O(n^2)$ time.


\subsection{Non-Additivity and Approximate Solutions}\label{preliminaries}
In practical phylogenomics, observed distance matrices $M$ often deviate from perfect additivity. In such cases, the objective becomes finding an additive matrix $M_T$ that corresponds to a tree $T$ that minimizes the sum of squared errors (SSQ), defined as:
\[
\min_T \|M_{T} - M\|_2.
\]
It is known that finding the optimal $T$ is $NP$-hard for various norms ($p=1,2,\infty$) \citep{farach_robust_1995}.

However, given a fixed tree topology $T$, one can at least solve for the optimal edge lengths $E$ that best fit $M$. This subproblem can be formulated as a non-negative least squares (NNLS) problem:
\begin{equation*}
    \min_{E \ge 0} \|P_T E - M_{\text{vec}}\|_2, 
\end{equation*}
where :
\begin{itemize}
    \item $E$ is an $m$-dimensional vector of edge lengths ($m$ is the number of edges in $T$).
    \item $M_{\text{vec}}$ the vectorized form of $M$, containing $\binom{n}{2}$ pairwise distances.
    \item $P_T$ is a ${n \choose 2} \times m$ path matrix encoding the tree’s topology. Each row corresponds to a pair $(i,j)$ of leaves, and each column corresponds to an edge $e$ in $T$. The entry $[P_T]_{(i,j);e}$ is 1 if edge $e$ lies on the path between $i$ and $j$, and 0 otherwise.    
\end{itemize}
This NNLS problem is convex and can be solved efficiently. The difficulty lies in choosing the optimal topology $T$. Because the space of possible tree topologies grows super-exponentially. Since enumerating and evaluating all possible trees is not practical for larger $n$, heuristic methods are employed to approximate the solution. One of the widely used heuristics is the Neighbor-Joining (NJ) method \citep{saitou_neighbor-joining_1987}. The NJ algorithm is efficient and makes no assumptions about the edge lengths. As shown in \cite{gascuel_concerning_1997}, NJ reconstructs the unique tree when given an additive distance matrix. Moreover, \cite{atteson_performance_1999} proved that if a distance matrix $M$ is nearly additive, there exists an additive distance matrix $D_T$ such that:
\[
|M - D_T|_\infty < \mu(T) / 2
\]
where $\mu(T)$ is the minimum edge length in $T$. All distance matrices $M$ that satisfy this condition share the same tree topology, meaning $T$ is the unique tree corresponding to these distances. The NJ algorithm has an optimal reconstruction radius in the sense that: (a) given a nearly additive distance function it reconstructs the unique tree $T$ and (b) there can be more than one tree for which $|M - D_T| < \delta$ holds if $\delta \ge \mu(T)/2$.

\subsection{From Genomic Data to Latent Representations}
In classical phylogenomics, evolutionary distances are estimated from genomic sequences using probabilistic models. However, in the context of single-cell data, the observed distance matrix often deviates substantially from additivity due to factors such as measurement noise, high dimensionality, and features unrelated to lineage. In this work, we propose learning a nonlinear mapping that projects high-dimensional observations into a latent space, and then computing distances in that space. The goal is for these learned distances to reflect the underlying lineage structure more accurately than distances directly computed from the original data. In essence, we aim to transform empirical dissimilarities into additive phylogenetic distances, thereby bridging the gap between observed data and their development histories.
\subsection{Constructing the Tree with NJ}
After learning the embeddings $\{\bm{z}_i\}$ in which pairwise distances approximate the true tree distances, the next step is to build the lineage tree $T$. A standard choice is the Neighbor-Joining (NJ) algorithm \citep{saitou_neighbor-joining_1987}, a greedy approach that iteratively merges pairs of nodes or clusters based on a pairwise distance matrix. Neighbor-Joining does not assume equal branch lengths and can yield accurate topologies even when distances are only approximately additive. 

The algorithm is guaranteed to recover the correct tree topology when the distances perfectly adhere to an additive tree metric, and it often performs well even when this assumption is not strictly met. In general, the method has been shown to work well with finite datasets, and it is one of the most widely used distance methods for tree graph inference.

 While NJ is a polynomial-time algorithm, its time complexity is $O(n^3)$ for $n$ taxa. This is typically manageable for moderate dataset sizes, but it may become computationally expensive for very large trees \citep{atteson_performance_1999,mihaescu_why_2009}.

\section{Markov Process on a Tree with Brownian Motion}\label{appendix:lemma-proof}
The motivation for using a distance-based formulation arises from Markov tree processes. Specifically, consider a continuous-state Markov process on a tree $T$ in which each node’s state evolves by Brownian motion along the edges. Such processes naturally induce an additive (tree) metric on the leaves, as described next.

We model cellular phenotypes as a vector-valued stochastic process evolving along a tree $T=(V, E)$ using a Brownian motion framework. Each vertex $v \in V$ has a continuous random state $x_v \in \mathbb{R}^d$, representing the phenotype of the corresponding cell or ancestral state.

\paragraph{Assumptions:}
\begin{itemize}
    \item Root Initialization: The trait value at the root $r$ of the tree is  $\mathbf{x}_r = \mathbf{x}_0$.
    \item Tree Markov Property: For each edge $e = (u \to v)$, the phenotype at $v$ depends only on its parent $u$.
    \item Independent Increments: Each trait change $\Delta_e = \mathbf{x}_v - \mathbf{x}_u$ is an independent Gaussian increment with zero mean.
    \item Brownian Covariance: For an edge $e$ of length $\ell_e, \Delta_e \sim \mathcal{N}(\mathbf{0}, \sigma^2 \ell_e \mathbf{I})$.
\end{itemize}
Under these assumptions, the expected squared Euclidean distance between any two leaves is proportional to the sum of the edge lengths on the unique path connecting them. Formally:

\begin{lemma}[Additivity of Expected Distances]
\label{lem:additivity-gaussian}
Suppose $\{X(v)\}_{v \in V}$ is a Gaussian Markov process on the tree T satisfying the four assumptions above. For any two leaves $i$ and $j$, define
\[
D(i,j) \;=\; \mathbb{E}\bigl[\|X(i) - X(j)\|^2\bigr].
\]

Then $D(i,j)$ is an additive (tree) metric. In particular, letting
\[
\ell_e := \operatorname{trace}(\Sigma_e) \quad \text{for each edge } e,
\]
we have
\[
D(i,j)
\;=\; \sum_{e \in \mathrm{path}(i,j)} \ell_e,
\]
where $\mathrm{path}(i,j)$ is the unique sequence of edges from $i$ to $j$.
\end{lemma}

\begin{proof}
\textbf{Step 1: Expressing Leaf States via Edge Increments.}
Pick a root $r$ in $T$.  For any vertex $v$, we have
\[
  X(v) 
  \;=\; X(r) \;+\; \sum_{e \in \mathrm{path}(r,v)} \Delta_e.
\]
All increments $\Delta_e$ are independent.

\textbf{Step 2: Difference of Leaf States.}
For leaves $i$ and $j$:
\[
  X(i) - X(j)
  \;=\; \bigl[X(r) + \textstyle\sum_{e \in \mathrm{path}(r,i)} \Delta_e\bigr]
         \;-\; \bigl[X(r) + \textstyle\sum_{e \in \mathrm{path}(r,j)} \Delta_e\bigr].
\]
Canceling $X(r)$, we get
\[
  X(i) - X(j)
  \;=\; \sum_{e \in \mathrm{path}(r,i)} \Delta_e 
         \;-\; \sum_{e \in \mathrm{path}(r,j)} \Delta_e.
\]
Since $\mathrm{path}(i,j)$ is the intersection of $\mathrm{path}(r,i)$ and $\mathrm{path}(r,j)$ (with some edges traversed in opposite directions), we can rewrite
\[
  X(i) - X(j) 
  \;=\; \sum_{e \in \mathrm{path}(i,j)} \!\!\pm\,\Delta_e.
\]

\textbf{Step 3: Expected squared distance.}
By the linearity of expectation and the mutual independence of the increments $\{\Delta_e\}$ on different edges:
\[
  D(i, j) = \mathbb{E}\!\bigl[\|X(i) - X(j)\|^2\bigr]
  \;=\; \mathbb{E}\!\Bigl[\Bigl\|\sum_{e \in \mathrm{path}(i,j)} \pm\, \Delta_e\Bigr\|^2\Bigr]
  \;=\; \sum_{e \in \mathrm{path}(i,j)} \mathbb{E}\!\bigl[\|\Delta_e\|^2\bigr],
\]
because cross-terms $\mathbb{E}[\Delta_e \cdot \Delta_{e'}]$ vanish for $e \neq e'$.

\textbf{Step 4: Relate to Trace of Covariance Matrices.}
Since $\Delta_e$ is a zero-mean Gaussian with covariance matrix $\Sigma_e$,
\[
  \mathbb{E}\!\bigl[\|\Delta_e\|^2\bigr]
  \;=\; \operatorname{trace}(\Sigma_e).
\]
Define
\[
   \ell_e 
   \;:=\; \operatorname{trace}(\Sigma_e)
   \;\ge\; 0.
\]
Thus,
\[
  D(i,j)
  \;=\; \sum_{e \in \mathrm{path}(i,j)} \ell_e,
\]
which is by definition a \emph{tree-additive} metric on the leaves. 
\end{proof}

\section{Motivation and Completeness of the Quartet-Based Approach}
\label{appendix:quartet_motivation}

In this section, we provide additional support for our use of local, quartet-based constraints in lineage reconstruction. As discussed in \S\ref{sec:problem_formulation}, quartet-based learning circumvents the need for a fully specified tree metric. Here, we discuss more deeply why quartets can be both theoretically sound and practically robust. This modular approach is well-suited to the cases where we have varying degree of supervised information for tree topologies.

\subsection{Completeness}
Quartets provide a complete characterization of an additive (tree) metric through the four-point condition (Them.~\ref{four_point_condition}). \cite{buneman_recovery_1971} shows that a distance matrix is tree-additive if and only if each quartet satisfies the four-point condition. Hence, enforcing consistency on all quartets is theoretically equivalent to enforcing a global tree metric. 
\subsection{Motivations}\label{appendix:local-vs-global}
\paragraph{Partial Constraints Fit Better with Quartets.} In many biological contexts, we do not know the full structure of the lineage tree but do have partial information about leaf labels. For example:
\begin{enumerate}
    \item \textbf{High-level Partitioning Setting.} Biologists may know that certain leaves belong to one of a few major clades but lack details on how those clades split internally.
    \item \textbf{Partially Leaf-labeled Setting.} A subset of leaves (cells) might have known lineage relationships from direct tracing experiments, while the others remain unlabeled.
\end{enumerate}

A distance-based approach that minimizes $\|D(f(x)) - D_{T_0}\|$ for a pruned tree $T_0$ can handle the Partial Leaf Labels scenario.  However, such a scheme cannot handle cases where we have qualitative partial knowledge. A quartet-based approach can integrate varying degrees of partial knowledge more naturally:
\begin{itemize}
\item In the High-level Partitioning Setting, a quartet is considered ``unknown'' if three or four of its leaves belong to the same clade, and ``known'' otherwise.
\item In the Partially Leaf-labeled Setting, a quartet is ``known'' if all four leaves have known relationships.
\end{itemize}

\paragraph{Error Toleration.} Classical phylogenetic methods often rely on a \emph{global} objective such as minimizing \(\|\,D - D_T\|\). However, in the presence of noise, outliers in \(D\) can have an outsized impact on the entire tree topology. In contrast, a \emph{quartet-based} method imposes constraints on only a subset of quartets at a time (as in Eq.~\ref{eq:quartet_loss}). This local enforcement often limits the impact of noisy distances to a small number of quartets, rather than globally distorting the inferred tree.

Moreover, even a fraction of erroneous quartets can still yield a correct (or near-correct) topology, provided that the majority are consistent with the true tree . This means local mistakes do not necessarily compromise the entire reconstruction, as long as there is sufficient consensus among the ``clean'' quartets. \citep{strimmer_quartet_1996,berry_faster_1999,richards_bayesian-weighted_2021,han_quartets_2023}

\subsection{Empirical Evidence from Quartet Methods}
\label{appendix:empirical_quartet}

Quartet-based phylogenetic methods have long been explored in the computational biology literature. For example:

\begin{itemize}[leftmargin=*]
    \item \textbf{Quartet Puzzling} \citep{strimmer_quartet_1996} uses repeated sampling of quartets from the full data, builds local trees, and then combines them to form a consensus. Empirical studies show this method performs competitively under moderate noise.
    \item \textbf{qNet} \citep{grunewald_qnet_2007} adapt quartet or triplet methods to construct more complex network or supertree models. They demonstrate robustness in settings where data are partially missing or have conflicting signals.
\end{itemize}

These studies highlight how localized, sample-based constraints often handle measurement noise more effectively than single-objective global optimization. Uncertain or contradictory quartets can be down-weighted or filtered out, preventing dubious distances from distorting the entire topology.
\section{Dataset Details}\label{app:dataset}

\subsection{Lineage-Resolved \emph{C.~elegans} Dataset}
Among model organisms, the nematode \emph{C.~elegans} is uniquely suited for benchmarking lineage-reconstruction methods because its embryonic development follows an \emph{invariant} pattern. Using the transcriptomic atlas by \citet{packer_lineage-resolved_2019} and \citep{large_lineage-resolved_2024}, we define three datasets with supervised tree-graph ground truth of varying sizes:
\begin{itemize}
    \item \textbf{C. elegans Large (295 leaves).} This dataset, drawn directly from \citep{packer_lineage-resolved_2019}, provides a broad coverage of terminal lineages, offering a moderately sized yet comprehensive benchmark.
    \item \textbf{C. elegans Small (102 leaves).} To create a simpler, more tractable dataset, we prune the original \emph{C.~elegans} lineage to include only nodes that have a clearly defined mapping to annotated cell types. This smaller tree is useful for rapid prototyping and evaluating basic performance. Details of curation can be seen in later of this subsection.
    \item \textbf{C. elegans Mid (183 leaves).} Building on the lineage-resolved atlas for both \emph{C.~elegans} and \emph{C.~briggsae} by \citet{large_lineage-resolved_2024}, we curate a set of 183 leaves that can be consistently mapped to both species. We use just the \emph{C.~elegans} portion here, providing an intermediate-sized dataset that balances coverage and complexity.
\end{itemize}

\paragraph{Preprocessing.} Following standard filtering criteria, we retain about 13,000 genes per dataset by removing those with minimal counts (fewer than 10 UMIs) or zero variance across cells.

\paragraph{Summary Statistics.} 
Table~\ref{tbl:c_elegans_stats} summarizes key statistics for each \emph{C.~elegans} dataset, including the number of cells per leaf, total leaves, Colless index \citep{lieberman_phylogenetics_2011}, tree diameter, depth, Faith’s Phylogenetic Diversity (PD)\citep{faith_conservation_1992}, and mean pairwise distance. 

The Colless index is a measure of tree imbalance, where a smaller value indicates a more balanced (symmetric) tree. Indeed, the relatively low Colless indices suggest that the \emph{C.~elegans} lineage is quite symmetric, and our pruned subsets remain well-distributed across the leaves. Faith’s PD is a measure of biodiversity, defined as the sum of branch lengths in a phylogenetic tree.

We provide two sets of tree metrics. In the \textit{Based on Reference Tree} section, branch lengths are set to be $1$s in the complete reference tree (containing 669 leaves), and these lengths are retrained and accumulated while pruning so that the distances among leaves in our curated datasets reflect the distances in the complete reference tree. In the \textit{Based on Topology} section, we recompute statistics under the assumption that all branch lengths are equal to 1 in the curated dataset, allowing us to evaluate the inherent structure of each pruned tree independently of the original branch lengths.

\begin{table}[ht]
\centering
\scriptsize
\caption{\textbf{Statistics of Datasets in the Cell Lineage Tree Reconstruction Benchmark.} 
Values in parentheses indicate standard deviations where applicable.}
\label{tbl:c_elegans_stats}
\scalebox{0.75}{
\begin{tabular}{@{}lccl|cccc|cccc@{}}
\toprule
\multirow{2}{*}{\textbf{Dataset}} & \multirow{2}{*}{\textbf{N Cells/Leaf}} & \multirow{2}{*}{\textbf{Leaves}} & \multirow{2}{*}{\textbf{Colless}} & \multicolumn{4}{c|}{\textbf{Based on Reference Tree}}                       & \multicolumn{4}{c}{\textbf{Based on Topology}}                              \\ \cmidrule(l){5-12} 
                                  &                                        &                                  &                                         & \textbf{Diameter} & \textbf{Depth} & \textbf{Faith's PD} & \textbf{Mean PD} & \textbf{Diameter} & \textbf{Depth} & \textbf{Faith's PD} & \textbf{Mean PD} \\ \midrule
C. elegans Small                  & 107                                    & 102                              & 0.047                                   & 21                & 7.92 (2.07)    & 293                 & 15.52            & 19                & 6.66 (2.01)    & 202                 & 12.27            \\
C. elegans Mid                    & 101                                    & 183                              & 0.018                                   & 19                & 8.21 (1.81)    & 424                 & 14.23            & 18                & 8.01 (1.66)    & 364                 & 13.53            \\
C. elegans Large                  & 117                                    & 295                              & 0.010                                   & 20                & 7.91 (1.69)    & 735                 & 15.84            & 20                & 7.62 (1.64)    & 588                 & 14.77            \\ \bottomrule
\end{tabular}
}
\end{table}

\paragraph{The Curation of C. elegans Small.}\label{app:curation_celegans_small}
In \citet{packer_lineage-resolved_2019}, three cell ontologies are available: (1) cell identities by barcode, (2) cell type annotations (from Packer \emph{et al.}), and (3) Lineage Node Names on the full embryonic tree (from Sulston \emph{et al.}). As illustrated in Figure~\ref{fig:entity_relation_celegans}, the relationship between Cell Identities and Cell Type Annotations is many-to-one, determined by manual annotation. Meanwhile, the relationship between Cell Type Annotations and Lineage Node Names is many-to-many.

On one hand, a single Cell Type Annotation may correspond to multiple lineage nodes due to symmetry (e.g., \emph{Cx} could be either \emph{Ca} or \emph{Cp}). On the other hand, one lineage node can be associated with multiple Cell Type Annotations because each annotation is defined as a distribution of cell states, which may overlap.

To create a relatively small, clean dataset while preserving the overall lineage structure of \emph{C.~elegans}, we selected only those cell type annotations that map to a single lineage node. We then constructed a lineage tree using these “clean” lineage nodes. This effectively serves as pruning the annotation tree. After pruning, we randomly drop a few clades to make sure the final lineage tree is in binary structure. As shown in Table ~\ref{tbl:c_elegans_stats}, the curated C.elegans Small dataset spans most major cell lines of \textit{C. elegans} (with large diameter and depth) while being relatively easier to reconstruct than the full lineage tree. 

\begin{figure}[ht]
    \centering
    \includegraphics[width=0.3\linewidth]{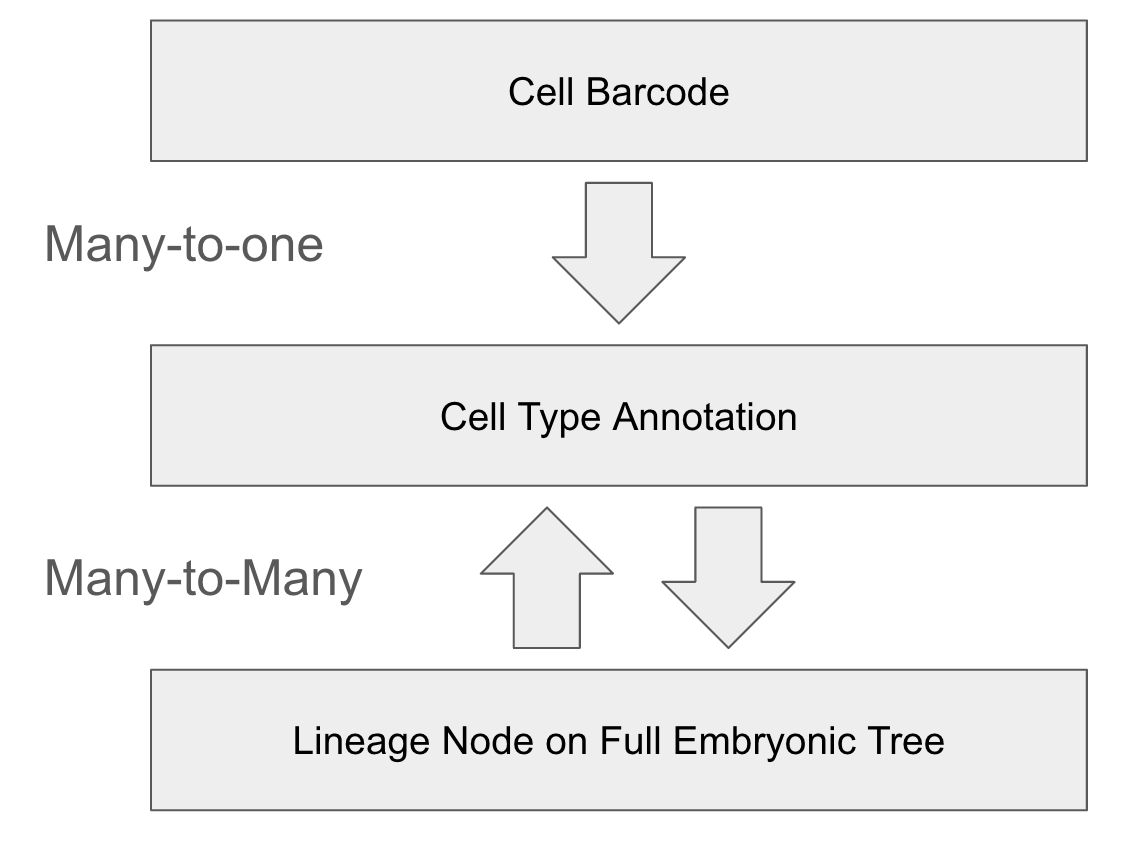}
    \caption{Entity relationships among the annotations in \citet{packer_lineage-resolved_2019}.}
    \label{fig:entity_relation_celegans}
\end{figure}

\subsection{Simulation: Brownian Motion on Lineage Tree}\label{app:simulation}
We developed a synthetic dataset that encompasses both heritable and non-heritable features to model cell lineage relationships in a structured, probabilistic manner. Each leaf $u$ in this dataset is assigned three types of features:
\[
\underbrace{(S_{u,1}, \dots, S_{u,n_{\text{signal}}})}_{\text{signal variables}}, \quad
\underbrace{(N_{u,1}, \dots, N_{u,n_{\text{noise}}})}_{\text{Gaussian noise}}, \quad
\underbrace{(A_{u,1}, \dots, A_{u,n_\text{AltSig}})}_{\text{alternative-tree noise}},
\]
providing a comprehensive representation that includes both true lineage signals and potential confounding factors.

\begin{figure}[ht]
\centering
\includegraphics[width=0.9\linewidth]{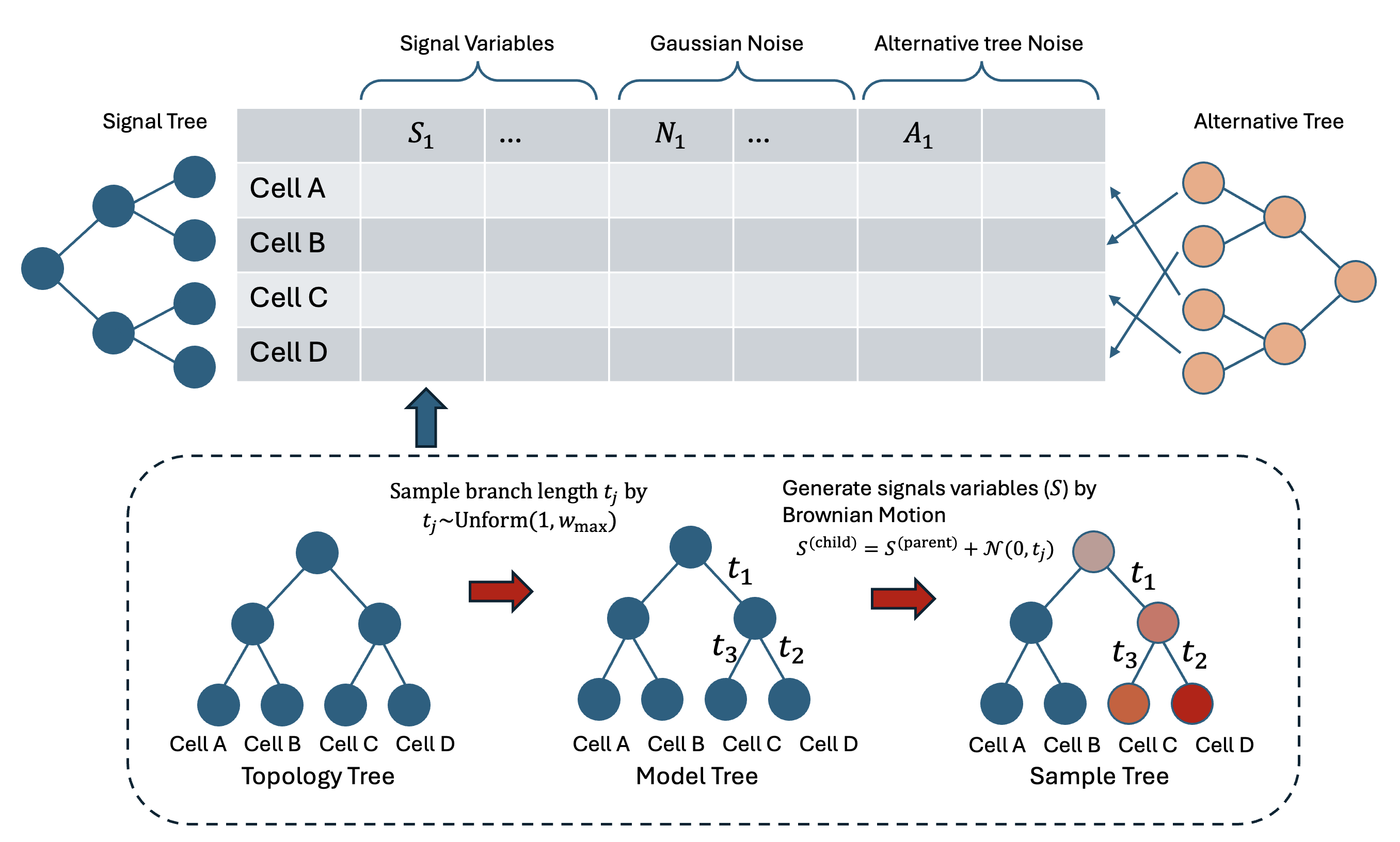}
\caption{Overall schematic of the lineage tree simulation. The simulation starts with a \emph{Topology Tree} defining cell relationships. Each edge is assigned a random length, producing a \emph{Model Tree}. Signal variables ($S_i$) are then generated along the branches via Brownian motion process, resulting in the \emph{Sample Tree}. Finally, Gaussian noise ($N_i$) and alternative tree noise ($A_i$) are added independently, yielding the complete feature matrix for each leaf. Alternative trees are constructed separately to simulate confounding topologies.}
\label{fig:simulation_schema}
\end{figure}

\subsubsection{Cell Linage Signals}
We represent the cell lineage as a full binary tree with $n_{\text{leaves}}$ terminal nodes. Each edge in the tree is assigned a length $t_j \sim \text{Uniform}(1, w_{\text{max}})$, which can be viewed as a developmental or temporal distance between parent and child nodes. To simulate heritable changes, we apply Brownian motion, a standard approach for continuous-character evolution \citep{omeara_testing_2006,pan_tedsim_2022}.

Starting from a zero vector at the root, each child node’s signal vector is obtained by adding a Gaussian increment $\mathcal{N}(0, t_j)$ to the parent’s signal vector along the branch $j$. Let $i$ denote the index of one out of $n_\text{signal}$ Brownian motion realizations, thereby defining a signal feature $S_i$ that encodes lineage information. Formally, Brownian motion at edge $j$ is
\[
S_i^{(\text{child})} = S_i^{(\text{parent})} + \mathcal{N}(0, t_j) \quad \text{for } i = 1, \dots, n_\text{signal}.
\]
This part of the simulation is governed by the hyperparameters ($n_\text{signal}, n_\text{leaves}, w_\text{walk}$). After generating all signal features, we compute the pairwise distance matrix and use Neighbor-Joining (NJ) to reconstruct the lineage tree. The reconstructed trees are compared with the true tree using the Robinson-Foulds (RF) distance. A grid-sweep of the hyperparameters was conducted, and the results are presented in Fig. \ref{fig:signal_dataset}.

\begin{figure}
    \centering
    \includegraphics[width=\linewidth]{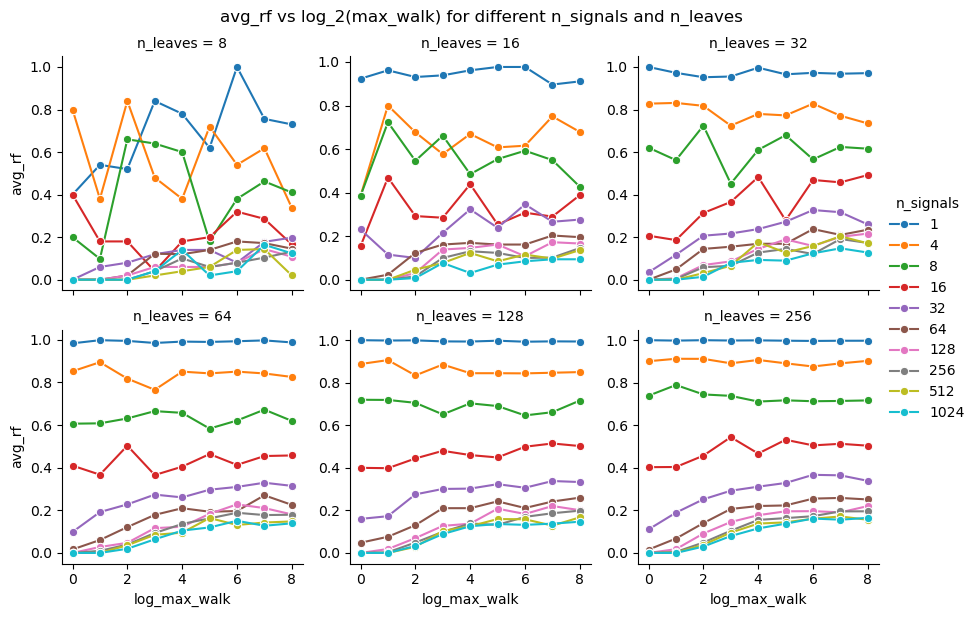}
    \caption{Mean Robinson-Foulds (RF) distance between reconstructed and true lineage trees as a function of $\log w_\text{max}$ for varying numbers of signals $n_\text{signal}$ and leaves $n_\text{leaves}$. Each panel represents a different value of $n_\text{leaves}$, while distinct lines within each panel correspond to different numbers of signals. Here, $w_\text{max}$ is the maximum possible edge length in the tree. Each data point shows the average RF distance across 25 trails (5 distinct model tree $\times$ 5 replicate sample tree per model tree.}
\label{fig:signal_dataset}
\end{figure}

Figure~\ref{fig:signal_dataset} illustrates several key factors affecting lineage reconstruction accuracy. Increasing the number of signals $n_\text{signals}$ generally reduces the average RF distance, reflecting better reconstructions. Conversely, as the number of leaves $n_\text{leaves}$ grows, the RF distance tends to rise, indicating the challenges of reconstructing larger and more complex trees. The parameter $\log w_\text{max}$ can help distinguish lineages by amplifying their differences—particularly when $n_\text{signals}$ or $n_\text{leaves}$ is small—but its benefits come with a random variation. Consequently, beyond a certain threshold, increasing $\log w_\text{max}$ does not further reduce the RF distance; instead, an inherent ``floor'' of imperfect reconstruction persists. Finally, there is a saturation effect: beyond some point, neither adding more leaves nor more signals consistently pushes the RF distance toward zero, highlighting the limit imposed by Brownian motion’s stochasticity and the maximum possible branch length $w_\text{max}$.

\subsubsection{Independent Gaussian Noises}
We introduce non-heritable noise features that are independent of the lineage. Each noise variable follows a Gaussian distribution with mean 0 and standard deviation $\sigma_\text{gau} = \beta\sigma_\text{signal}$. Symbolically,
\[
N_i \sim \mathcal{N}(0, \sigma_\text{gau}) \quad \text{for } i = 1, \dots, n_\text{gaussian}.
\]

\subsubsection{Alternative Tree Noise}

In addition to independent Gaussian noise, we introduce a second type of noise structured by an \emph{alternative tree}. This \emph{alternative tree noise}, or ``anti‐signal,'' represents confounding factors that give rise to correlated features unrelated to the true lineage. Examples include environmental or microenvironmental influences (e.g., regional nutrient availability), cell‐cycle synchronization, or technical batch effects. These factors can cluster cells according to a topology distinct from the actual lineage, thus posing a challenge for tree reconstruction.

\paragraph{Key parameters.}
To generate alternative tree noise, we first partition the set of leaves into one or more subsets. In our running example in Fig.~\ref{fig:alternative_tree_noise}, we split the leaves into two partitions of equal size, i.e., [0.5, 0.5]. For each partition, we construct an alternative tree of a specified size (equal to the number of leaves in the largest partition, set to 4 in the example) and create a new set of ``alternative'' feature columns, \text{AltSigs}. Each leaf in the original (signal) tree thus has a corresponding ``alternative'' leaf in the alternative tree. We can optionally have multiple alternative trees per partition, but in the example, we set this number to 1.

\paragraph{Generation procedure.}
Similar to the main (signal) tree, each alternative tree has its own branch lengths drawn from some distribution (e.g., $\text{Uniform}(1, w_{\max}^{(\mathrm{alt})})$). For every leaf $u$ in the main (signal) tree, there is a corresponding leaf $u^{\prime}$ in the alternative tree. We then generate ``alternative'' features for each such leaf by simulating a Brownian-motion–like process (or another model as desired), along the branches of the alternative tree. These ``anti-signals'' are thus correlated according to the \emph{alternative} topology rather than the true lineage topology. In the final dataset, each leaf $u$ accumulates an additional vector of alternative noise features,
\[A_{u,1}, A_{u,2}, \dots, A_{u,n_\text{AltSig}}.\]
Because these features reflect a different branching structure than the true cell lineage, they act as confounders in any tree reconstruction method that does not isolate genuine signals from extraneous structure. By controlling \textit{the number of partitions}, \textit{leaves per alternative tree}, \textit{number of alternative signals}, and \textit{total number of alternative trees}, one can tune the difficulty of distinguishing true lineage signal from these carefully structured ``anti-signal'' features. (Fig.~\ref{fig:alternative_tree_noise})

\begin{figure}[ht]
\centering
\includegraphics[width=0.9\linewidth]{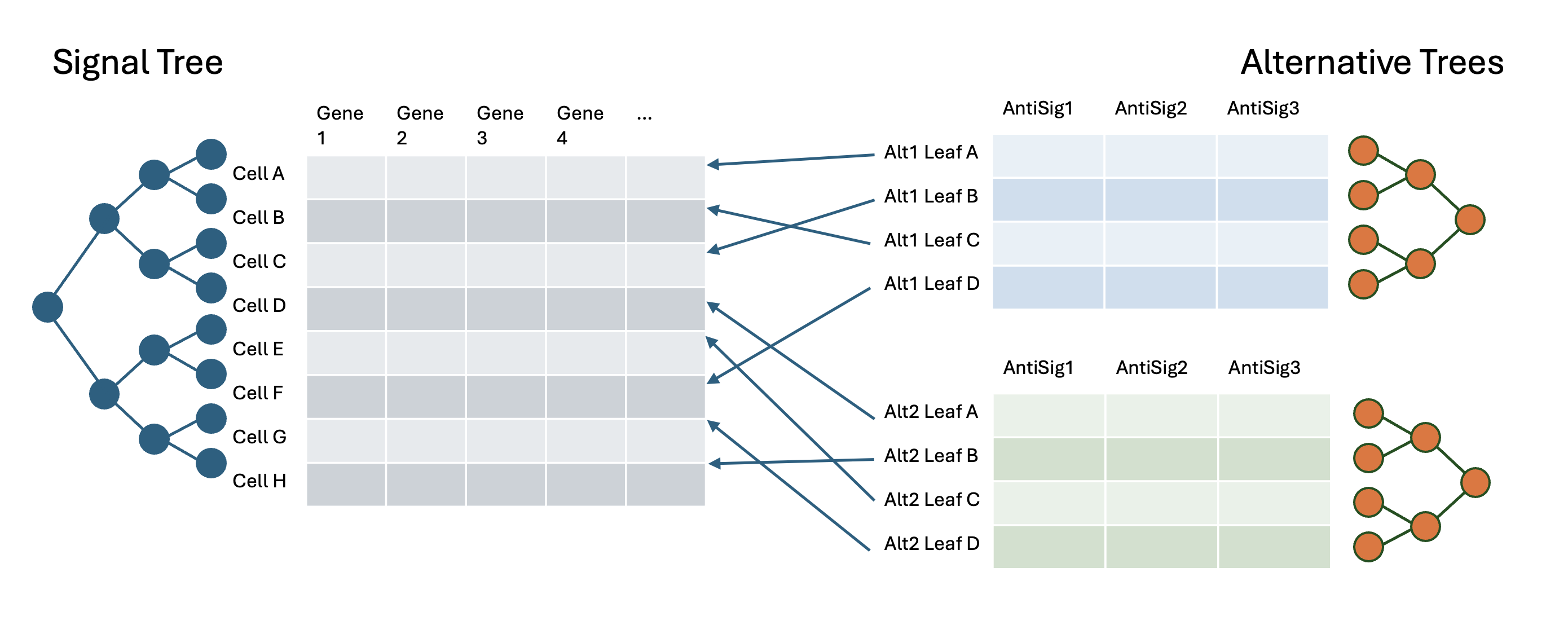}
\caption{Schematic of Alternative Tree Noise Generation. The left panel depicts the true signal tree, representing the actual lineage relationships between cells. The middle panel shows the feature matrices for signal genes (left) and alternative tree noise or "anti-signals" (right). Arrows indicate the mapping of leaves from the signal tree to the corresponding leaves in the alternative trees. The right panel illustrates the alternative trees, which are constructed independently for each partition of leaves. These alternative trees have different topologies and branch lengths than the true signal tree, resulting in correlated noise features that can confound lineage reconstruction.}
\label{fig:alternative_tree_noise}
\end{figure}

\subsubsection{Supervised Setting}
To evaluate lineage reconstruction methods in a supervised setting, we generate two sets of signal variables: one for training and one for testing (Fig.~\ref{fig:simulation_supervised_setting}). First, we generate signal variables using the Brownian motion process described earlier. Then, for each signal variable, we add a small amount of Gaussian noise, $\mathcal{N}(0, 0.1 \times \bar{\sigma}_\text{signal})$, where $\bar{\sigma}_\text{signal}$ represents the mean standard deviation of the generated signal features across all leaves and signal dimensions. This creates a replicate of the signal variables with slight perturbations, suitable for evaluating a method's ability to generalize from training data to unseen test data. The training and testing datasets, therefore, have the same underlying lineage structure but differ in the specific signal values due to the added noise. Gaussian and alternative tree noises are added independently to both the training and testing datasets, as described in the previous subsections.

\begin{figure}[ht]
\centering
\includegraphics[width=0.9\linewidth]{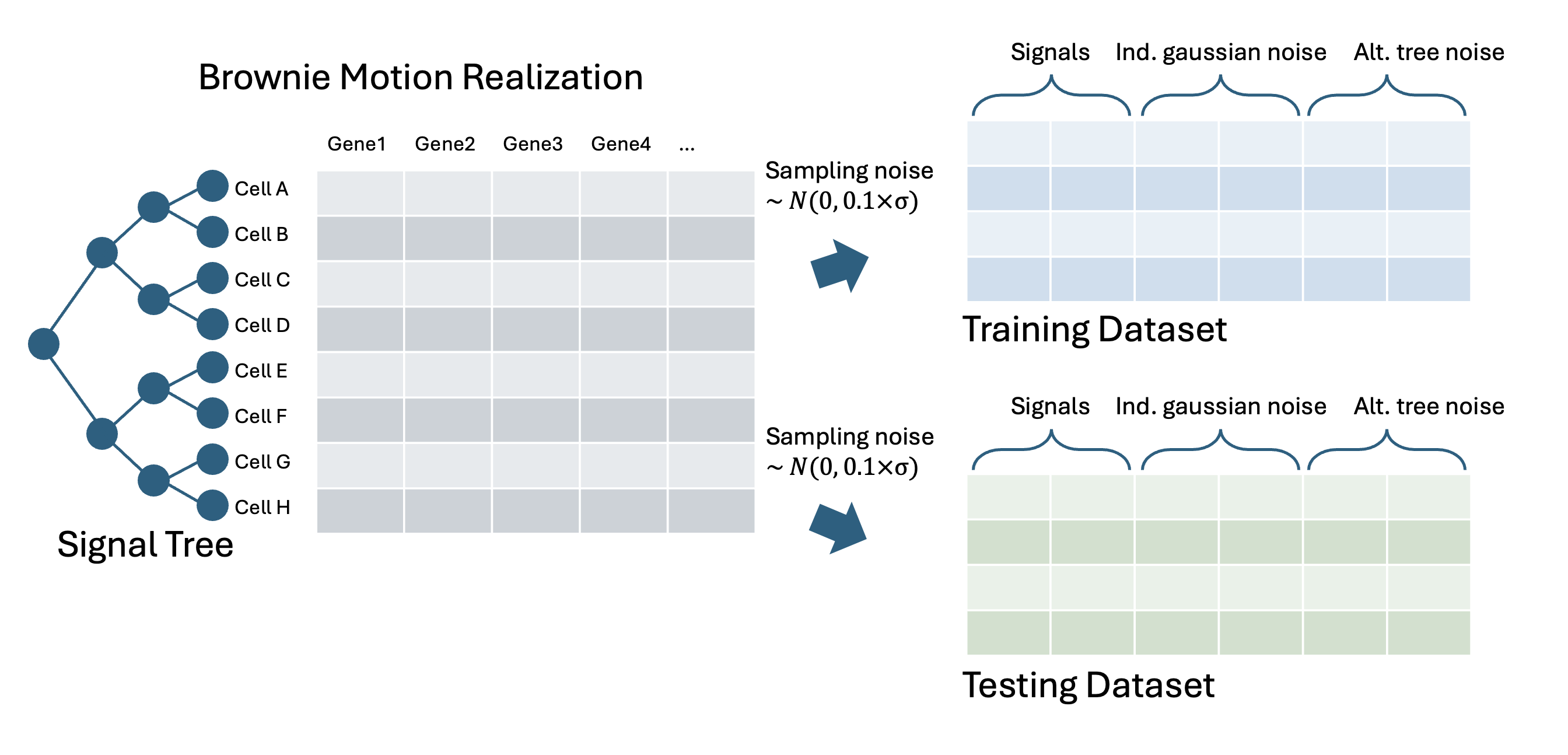}
\caption{Illustration of the supervised setting for simulation. After generating signal variables through Brownian motion on the signal tree, a small amount of Gaussian noise is added to create a replicate set of signal features. One set is used for training, and the other for testing. Both the training and testing datasets then receive independent additions of Gaussian noise and alternative tree noise.}
\label{fig:simulation_supervised_setting}
\end{figure}

\section{Problem Formulation Details}
We consider an underlying lineage tree ${T}$ that describes the data-generating process, where a single progenitor cell initiates binary replication while also modifying molecular states such that a diverse set of daughter cells is generated. In this bifurcating tree, each internal node corresponds to ancestral mother cells and edges represent mother-daughter relationships. The leaves of the tree correspond to sampled cells, typically terminally differentiated cells (\textit{i.e.}, cells that no longer replicate). Our goal is to recover this tree structure ${T}$ using high-dimensional phenotypic data, such as transcriptomic profiles, measured on the leaves. 

\subsection{Prior Information}\label{app:prior_info}
The problem of lineage reconstruction from single-cell phenotypic data poses challenges reminiscent of classical systematics in biology, a field historically concerned with organizing biological diversity based on observable traits. \citep{zeng_what_2022,domcke_reference_2023,quake_cellular_2024} The low correlation between phenotype and cell lineage complicates the task. Moreover, selecting signal features is an NP-hard problem, and reconstructing the phylogeny is also a well-known NP-hard problem. Despite these challenges, heuristic solutions under reasonable constraints remain viable. After all, the essence of systematics is to address these complexities. \citep{zaharias_recent_2022}

In a fully supervised setting, we assume complete knowledge of tree topology and pairwise distances between leaves. While idealized, this scenario sets an upper bound on performance. However, such an assumption is rarely achievable in practice. We therefore formulate the following two weakly supervised scenarios.

\subsubsection{High-level Partitioning Setting}\label{app:high_level_partition}
In this setting, we assume prior knowledge that partitions the leaves into $k$ broad clades or categories based on partial qualitative information. Systematics commonly employs simple yet decisive features that elucidate population structure. For example, the presence or absence of a vertebral column distinguishes vertebrates from invertebrates. \citep{cavalli-sforza_phylogenetic_1967} Similarly, it is possible to identify some molecular phenotypes—such as surface marker genes—that reveal functional insights into cell identity and lineage. \citep{crowley_atlas_2024} With such ``clade'' assignments in place, certain quartet topologies can be inferred directly from membership information.

\paragraph{Theoretical Proportions}
Suppose there are $k$ clades, each containing many leaves. When four leaves (a ``quartet'') are chosen uniformly and independently from all leaves, we categorize them by how many leaves come from each clade (Fig.~\ref{fig:quartets_parition}) :

\begin{enumerate}
    \item $(4, 0, 0, 0)$ – All four leaves in the same clade.
    \item $(3, 1, 0, 0)$ – Three leaves in one clade, the fourth in another.
    \item $(2, 2, 0, 0)$ – Two leaves in one clade, two leaves in another.
    \item $(2, 1, 1, 0)$ – Two leaves in one clade, one leaf in each of two other clades.
    \item $(1, 1, 1, 1)$ – All four leaves in different clades.
\end{enumerate}
We refer to $(4,0,0,0)$ and $(3,1,0,0)$ as \textbf{in‐clade quartets}, because most or all of the leaves lie in a single clade. We call $(2,2,0,0)$ and $(2,1,1,0)$ \textbf{resolvable quartets}, since their topology can be inferred from clade membership. Finally, $(1,1,1,1)$ are cross‐clade quartets (one leaf per clade). By this scheme, the known quartets are the resolvable ones, while in‐clade and cross‐clade quartets remain unknown, as clade membership alone does not determine their topology.

\begin{figure}[ht]
\centering
\includegraphics[width=0.9\linewidth]{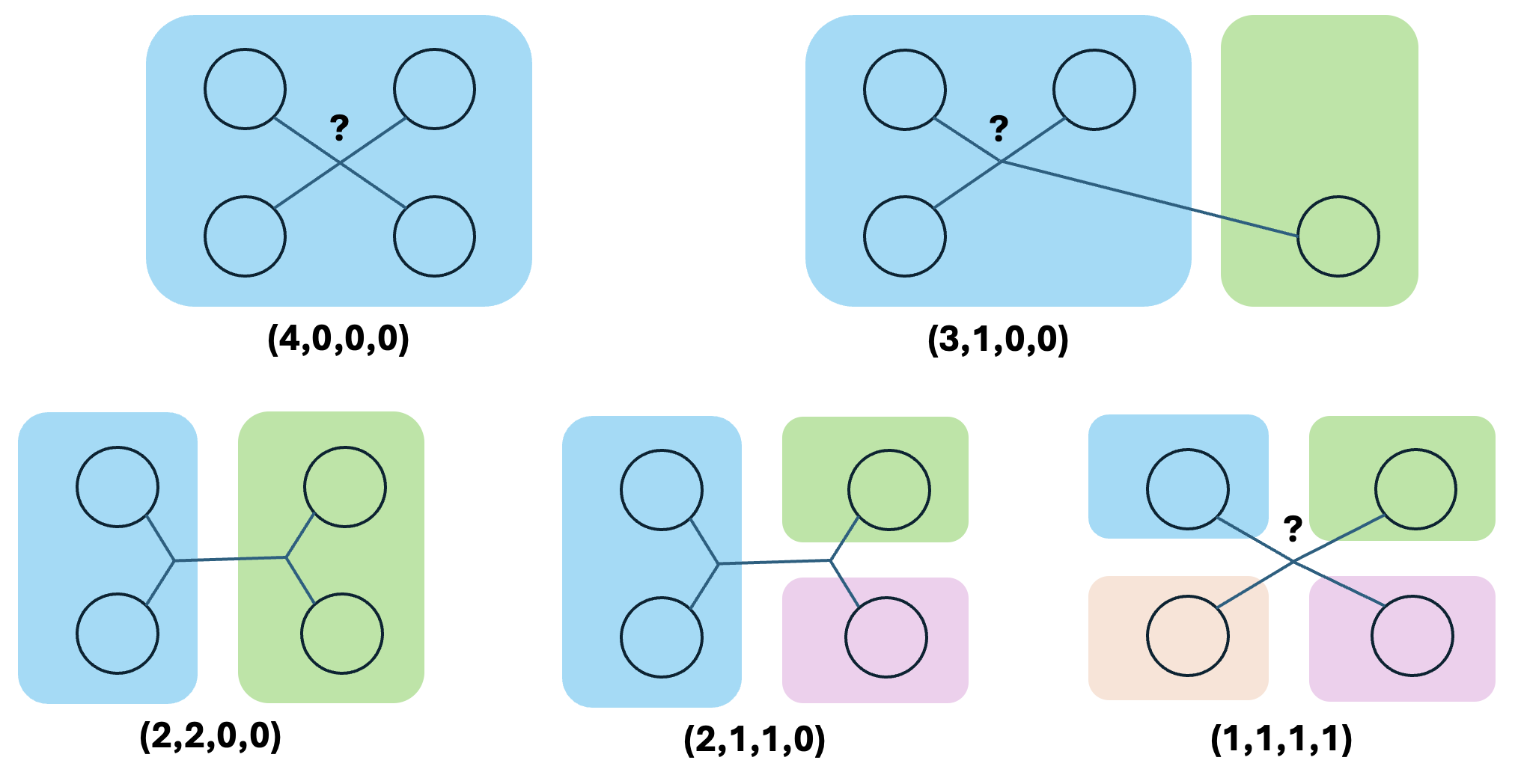}
\caption{\textbf{Illustration of the five possible quartets.}
Each panel shows how four leaves (circles) can be distributed across up to four clades (indicated by color‐shaded regions).
$(4,0,0,0)$ and $(3,1,0,0)$ are ``in‐clade'' (unknown),
$(2,2,0,0)$ and $(2,1,1,0) $are ``resolvable'' (known), and
$(1,1,1,1)$ is ``cross‐clade'' (unknown).}
\label{fig:quartets_parition}
\end{figure}

If each leaf is equally likely to come from any of the $k$ clades, the probability of each partition can be derived as follows. These probabilities sum to 1 and define the expected fraction of each quartet type under uniform sampling. Figure \ref{fig:quartet_prop} plots these theoretical proportions as $k$ increases from 2 to 16. Notably, in this idealized scenario, $k=4$ maximizes the overall fraction of resolvable ($(2,2,0,0)$ and $(2,1,1,0)$) quartets.

\begin{align*}
P(4,0,0,0) \;=&\; \frac{1}{k^3}.\\
P(3,1,0,0) \;=&\; \frac{4\,(k-1)}{k^3}.\\
P(2,2,0,0) \;=&\; \frac{3\,(k-1)}{k^3}.\\
P(2,1,1,0) \;=&\; \frac{6\,(k-1)\,(k-2)}{k^3}.\\
P(1,1,1,1) \;=&\; \frac{(k-1)\,(k-2)\,(k-3)}{k^3}.
\end{align*}

\begin{figure}[ht]
    \centering
    \includegraphics[width=0.75\linewidth]{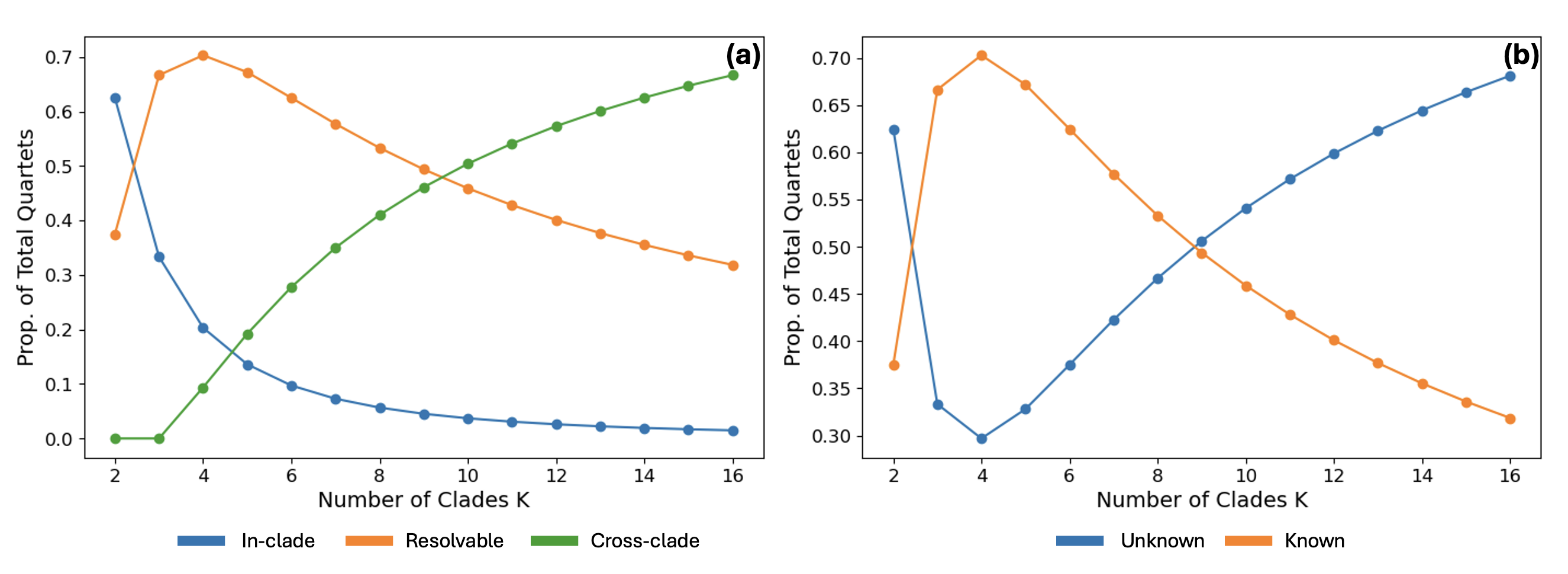}
    \caption{\textbf{Theoretical proportions} of the five quartet partitions under the uniform clade model, plotted as a function of the number of clades $k$. (a) ``in‐clade'' covers $(4,0,0,0)$ and $(3,1,0,0)$, ``resolvable'' covers $(2,2,0,0)$ and $(2,1,1,0)$, and ``cross‐clade'' is $(1,1,1,1)$. The fraction of resolvable quartets is highest at $k=4$. (b) In-clade and cross-clade quartets are unknown quartets. Resolvable quartets are known quartets. }
    \label{fig:quartet_prop}
\end{figure}

\paragraph{Empirical vs. Theoretical Results} Figure \ref{fig:quartet_prop_heuristic} compares the above theoretical proportions to empirical observations (“heuristics”) from various datasets, including simulated 64‐leaf trees (“Sim‐64”) and several C. elegans lineage data sets with different clade definitions. In most cases, the overall trend holds, with known quartets peaking near $k=4$. The main exception is the \textit{C. elegans Small} dataset, which has very unbalanced clade sizes at $k=4 (54, 10, 25, 13)$. Because the uniform sampling assumption is strongly violated there, the fraction of resolvable quartets deviates from the ideal theory.

\begin{figure}[ht]
    \centering
    \includegraphics[width=0.75\linewidth]{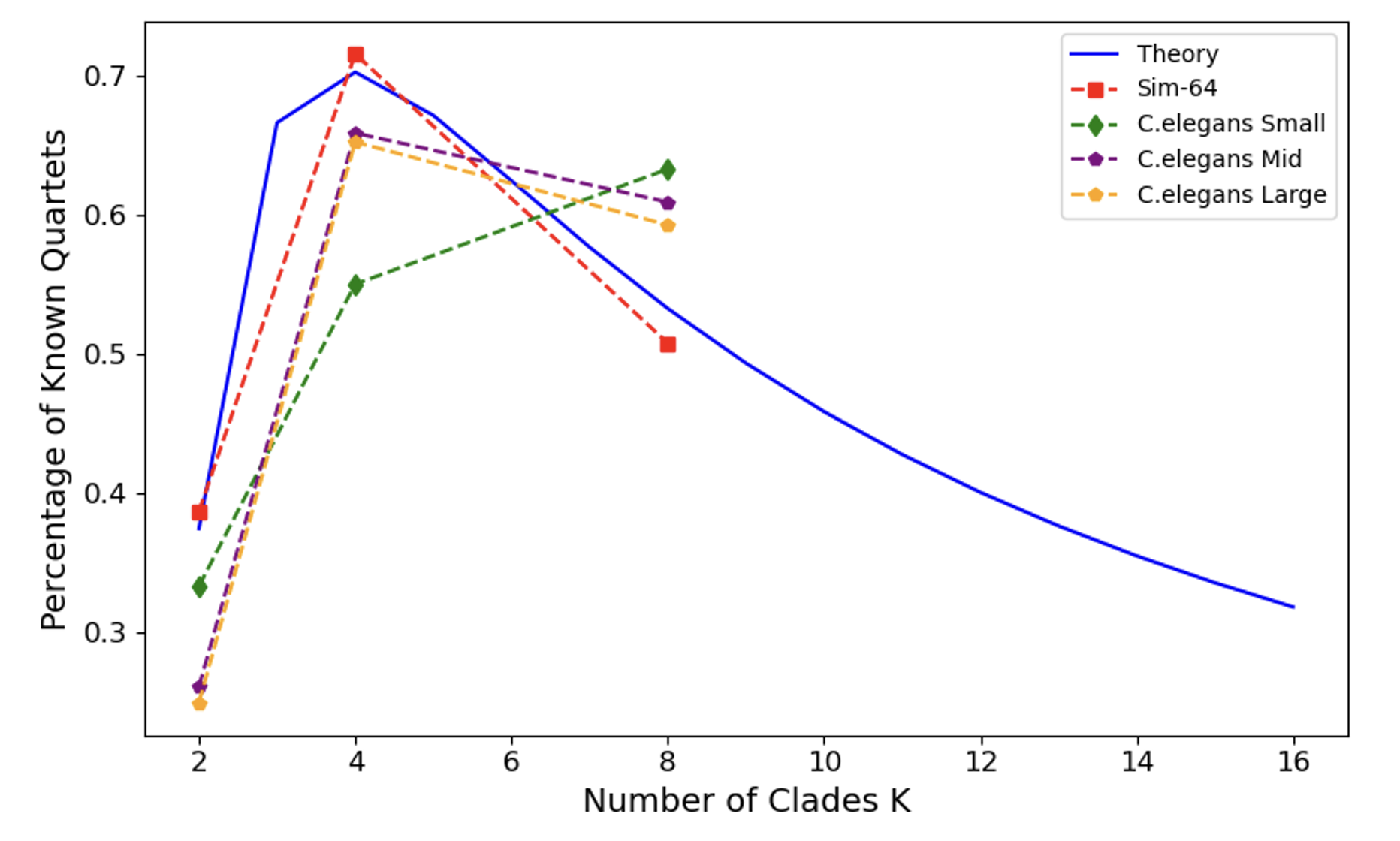}
    \caption{
    \textbf{Comparison of theoretical vs.\ empirical quartet proportions} for several datasets. While the general shape (peak near $k=4$) holds in most cases, strong deviations arise when the clade sizes are highly unbalanced, as in C. elegans Small.
    }
    \label{fig:quartet_prop_heuristic}
\end{figure}

\subsubsection{Partially Leaf-labeled Setting}
In this setting, ground-truth lineage relationships are known for only a subset of cells—often via lineage-tracing techniques such as CRISPR–Cas9-based barcoding—while the other cells remain unlabeled. Let $\kappa$ be the fraction of leaves with known lineage relationships. Then the fraction of fully labeled quartets is $\kappa^4$. For instance, when $\kappa=0.5$, approximately $6.25\%$ of quartets are fully labeled, rising to around $40.96\%$ for $\kappa=0.8$, and to roughly $85\%$ for $\kappa=0.9$.

Figure~\ref{fig:quartet_partial} shows how the proportions of known, unknown, and partially labeled quartets shift with $\kappa$. Even a moderate decrease in $\kappa$ leads to a substantial drop in known quartets, underscoring the challenges posed by partial labeling. 

\begin{figure}[ht]
    \centering
    \includegraphics[width=0.75\linewidth]{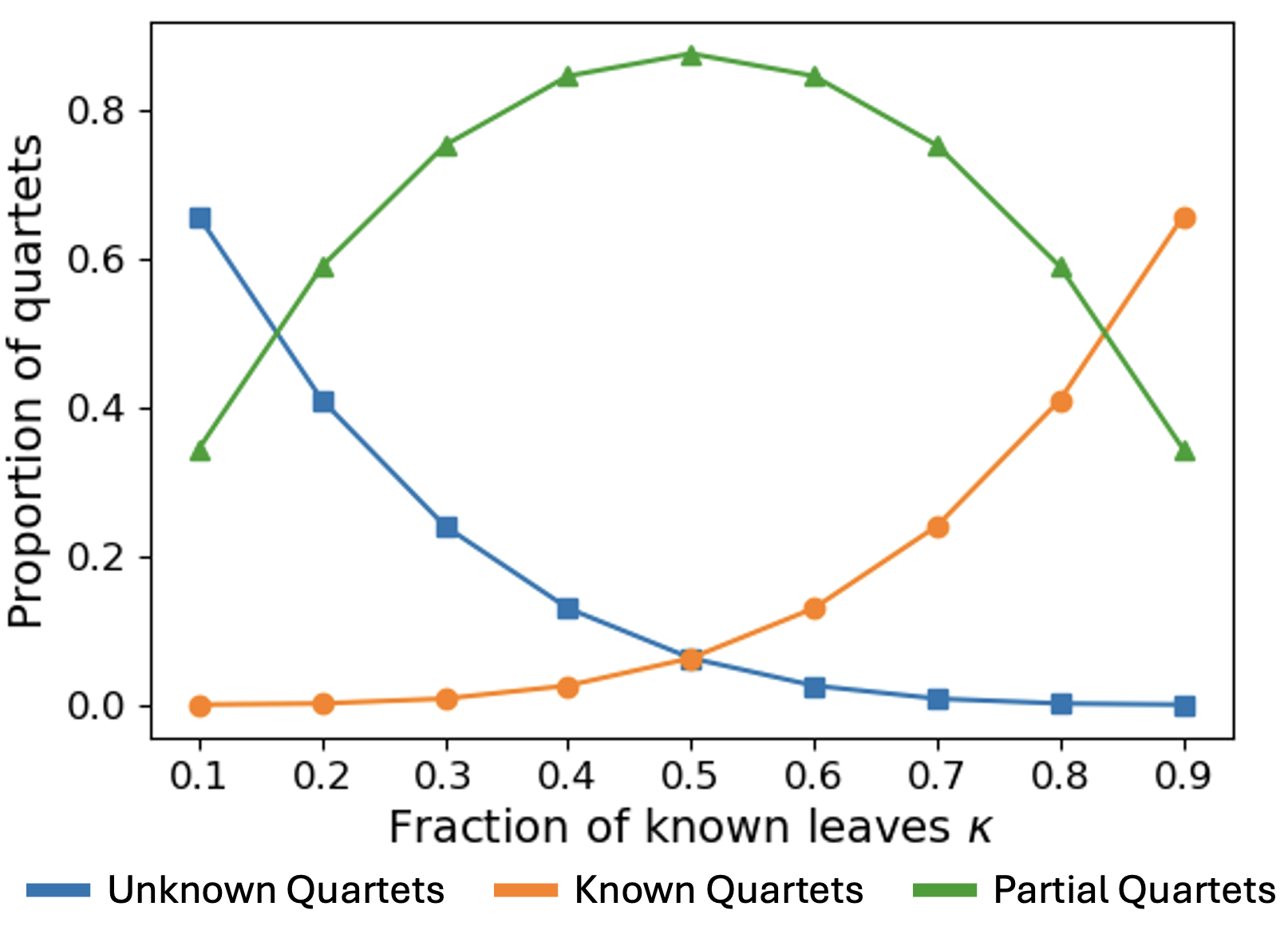}
    \caption{\textbf{Proportions of known, unknown, and partial quartets} as a function of the fraction $\kappa$ of known leaves.}
    \label{fig:quartet_partial}
\end{figure}

\section{Model}
\subsection{Model Architecture}
Our proposed framework, CellTreeQM, aims to learn embeddings from high-dimensional phenotypic data that facilitate phylogenetic reconstruction. To effectively learn relationships among cells, we use a sequence of Transformer encoder blocks as the backbone of the network, illustrated in Figure~\ref{fig:phylodist}. Unlike classical Transformer models, we do not include positional encodings, as our input cells are not inherently ordered. Without positional constraints, the self-attention mechanism can focus purely on learning meaningful relationships based on the feature similarities between cells. 

Moreover, we incorporate two types of dropout regularization: Data Dropout, which is applied in the attention encoder to prevent overfitting on input features; Metric Dropout, a dropout layer added after the network’s output layer\citep{qian_distance_2014}. Empirically, we found that Metric Dropout improves the model's performance. We call the dropout layer after the output layer the metric dropout. The dropout in the attention encoder is called data dropout.

For comparison, we also implemented a CellTreeQM-FC model with a straightforward feedforward architecture. It consists of a stack of eight fully connected layers, each with a hidden dimension of 1024. Nonlinear activation functions (ReLU) are applied between layers to introduce model capacity and expressiveness. 

In Table \ref{tab:transformer-vs-fc}, we compared the performance between the network with Transformer encoder as backbones and fully connected layers as backbone on a real dataset under a supervised setting. The network with the attention module performs better than FC. This indicates that while this FC is simple and relatively fast, it learns the pairwise distances within each quartet without modeling the global pairwise relationship. 

\begin{table}[h]
\centering
\scriptsize
\caption{RF distances for various datasets and embeddings. Lower values indicate trees that are more similar to the ground truth topology. Permutation experiments are conduced with CellTreeQM-transformer. Experiments on C.elegans Small are repeated 3 times. Values in parentheses are standard errors.}
\label{tab:transformer-vs-fc}
\begin{tabular}{@{}llllllll@{}}
\toprule
                 & N leaves & Dir. Recon. & \begin{tabular}[c]{@{}l@{}}CellTreeQM-\\ transformer\end{tabular} & \begin{tabular}[c]{@{}l@{}}CellTreeQM-\\ fc\end{tabular} & \begin{tabular}[c]{@{}l@{}}Leaves \\ Permutation\end{tabular} & \begin{tabular}[c]{@{}l@{}}Cell\\ Permutation\end{tabular} & \begin{tabular}[c]{@{}l@{}}Gene\\ Permutation\end{tabular} \\ \midrule
C.elegans Small    & 102      & 0.929  & 0.246 (0.012)                                                    & 0.400 (0.026)                                           & 0.510 (0.058)                                                 & 0.927 (0.014)                                              & 0.937 (0.012)                                              \\
C.elegans Large & 295      & 0.955  & 0.571                                                            & 0.647                                                  & 0.876                                                         & 0.969                                                      & 0.973                                                      \\
\bottomrule
\end{tabular}
\end{table}


\begin{figure}[ht]
    \centering    
    \includegraphics[width=0.75\linewidth]{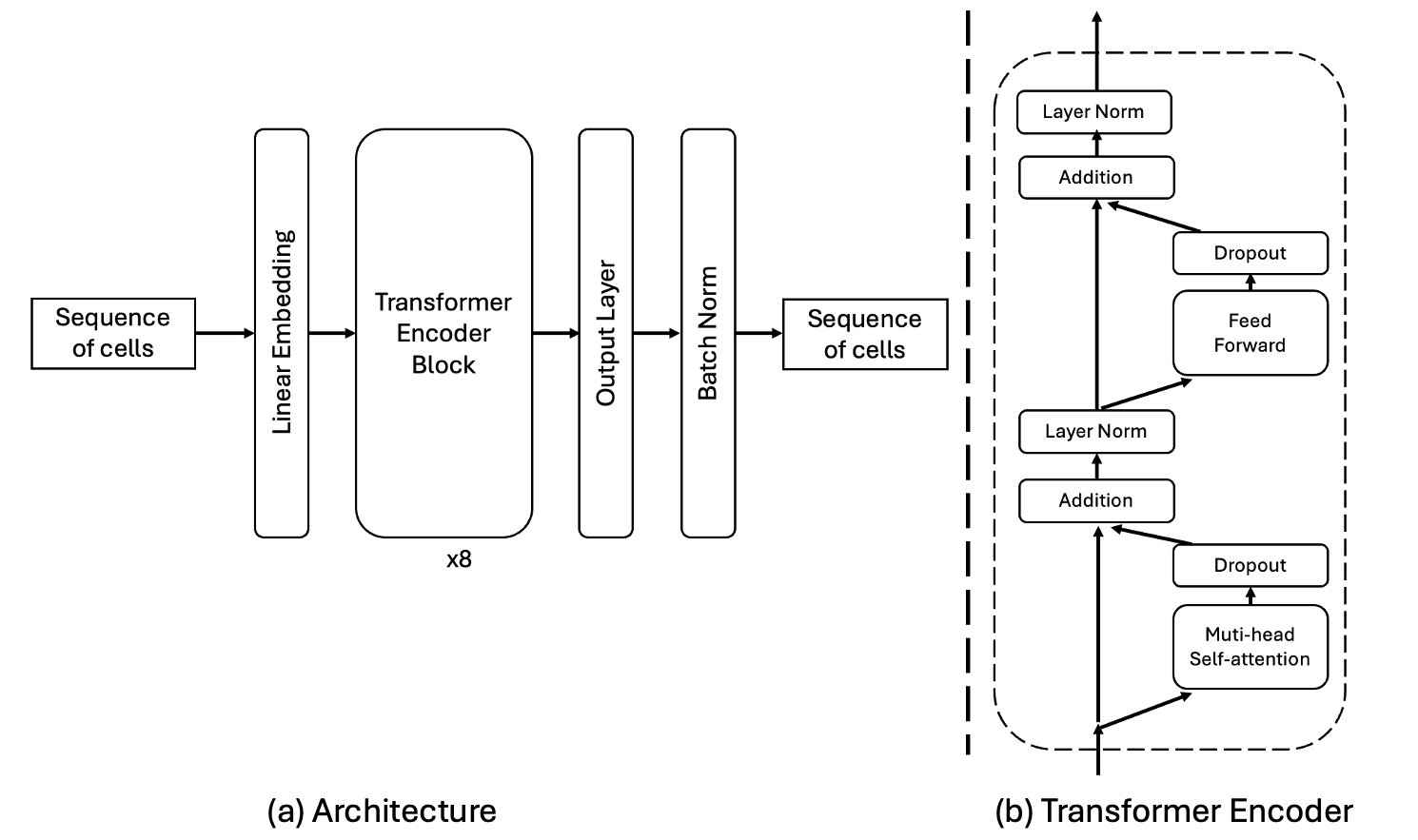}
    \caption{(a) The architecture of PhyloDist (b) Transformer Encoder Block}
    \label{fig:phylodist}
\end{figure}

\subsection{Loss Function Component Study}\label{app:loss_ablation}
To further dissect the contributions of each component in our framework, we conduct an ablation study using the C.elegans Small dataset under a supervised setting. Table~\ref{tbl:loss_ablation} reports the RF distances from trees reconstructed under various modified configurations of CellTreeQM, isolating the effects of the ``close'' and ``push'' terms of the quartet loss, as well as the regularization term $\Omega$.

The ablation results confirm that each component of our loss function plays a significant role in shaping the latent space. The combination of the ``close'' and ``push'' quartet constraints is essential to effectively approximate an additive metric, while the $\Omega$ regularization ensures that the learned representation remains grounded in the original phenotypic data. Together, these components enable CellTreeQM to achieve superior topology recovery compared to baseline or partially ablated models.


\begin{table}[ht]
\centering
\footnotesize
\caption{Loss component study on the C.elegans-Dev dataset. Lower RF distance values indicate better topological similarity to the ground truth. The reported values are means
across ten runs, with standard deviations in parentheses.}
\label{tbl:loss_ablation}
\begin{tabular}{@{}lllllllll@{}}
\toprule
         &    & Quartet Sampling & Deviation & Additivity & Close & Push & margin & Relative RF $\downarrow $  \\ \midrule
baseline &    & mismatched       & 0.01      & 2          & 1     & 10   & 0.5    & 0.274 (0.021) \\ \midrule
(a)      & a1 &                  & 0         &            &       &      &        & 0.839 (0.011) \\
         & a2 &                  &           & 0          &       &      &        & 0.846 (0.008) \\ \midrule
(b)      & b1 &                  &           & 0.005      &       &      &        & 0.838 (0.008) \\
         & b2 &                  &           & 0.01       &       &      &        & 0.836 (0.008) \\
         & b3 &                  &           & 0.1        &       &      &        & 0.707 (0.006) \\
         & b4 &                  &           & 0.5        &       &      &        & 0.277 (0.016) \\
         & b5 &                  &           & 1          &       &      &        & 0.253 (0.016) \\
         & b6 &                  &           & 4          &       &      &        & 0.525 (0.203) \\ \midrule
(d)      & d1 &                  &           &            & 1     & 0    &        & 0.514 (0.042) \\
         & d2 &                  &           &            & 0     & 10   &        & 0.547 (0.019) \\ \midrule
(e)      & e1 &                  &           &            &       & 1    &        & 0.433 (0.031) \\
         & e2 &                  &           &            &       & 2    &        & 0.377 (0.026) \\
         & e3 &                  &           &            &       & 5    &        & 0.321 (0.023) \\
         & e4 &                  &           &            &       & 20   &        & 0.249 (0.019) \\
         & e5 &                  &           &            &       & 30   &        & 0.257 (0.017) \\ \midrule
(f)      & f1 &                  &           &            &       &      & 0      & 0.267 (0.012) \\
         & f2 &                  &           &            &       &      & 0.1    & 0.263 (0.017) \\
         & f3 &                  &           &            &       &      & 0.3    & 0.275 (0.030) \\
         & f4 &                  &           &            &       &      & 0.7    & 0.287 (0.040) \\
         & f5 &                  &           &            &       &      & 1      & 0.314 (0.018) \\ \midrule
(g)      & g1 & matched          &           &            &       &      &        & 0.860 (0.018) \\
         & g2 & all              &           &            &       &      &        & 0.306 (0.022)\\\bottomrule
\end{tabular}
\end{table}

\section{Additional Results}\label{app:results}
\subsection{Baselines}\label{app:baselines}
\emph{CellTreeQM} and both baselines produce a learned embedding space from which we extract pairwise distances between leaves. 
\paragraph{Triplet Loss}  For each quartet of leaves, we first identify the two leaves with the smallest ground-truth distance (based on the known quartet ordering for the supervised setting and on observed quartet ordering for the unsupervised setting) and treat them as the \emph{anchor} (\(A\)) and \emph{positive} (\(P\)). We then select the leaf that is farthest from $A$ to serve as the \emph{negative} ($N$). Denoting the learned embeddings as \(f(\cdot)\) and using a margin \(m\), the triplet loss is defined as:
\begin{equation}
\label{eq:triplet_loss}
\scalebox{0.9}{$
\mathcal{L}_\text{tri} 
= \sum_{(A,P,N)} 
\Bigl[ 
\|f(A) - f(P)\|^2 
- \|f(A) - f(N)\|^2
+ m_0
\Bigr]_+,
$}
\end{equation}
where \([\cdot]_+\) denotes the hinge function \(\max(0, \cdot)\). Each quartet thus provides local distance orderings that the learned embedding space must respect.

\paragraph{Quadruplet Loss} We extend the triplet formulation by incorporating a second negative. From each quartet, after we identify the closest pair \((A, P)\), we designate \emph{two} negatives, \(N\) and \(N'\), chosen among the more-distant leaves. One common form of the quadruplet loss is:
\begin{equation}
\label{eq:quadruplet_loss}
\scalebox{0.9}{$
\begin{aligned}
\mathcal{L}_\text{quad}&=\sum_{(A,P,N,N')} \Bigl[\|f(A) - f(P)\|^2- \|f(A) - f(N)\|^2 + \alpha \Bigr]_+ \\
&+ \Bigl[\|f(A) - f(P)\|^2  - \|f(N') - f(N)\|^2 + \beta \Bigr]_+.
\end{aligned}
$}
\end{equation}
where \(\alpha\) and \(\beta\) are margins. The first bracket encourages the anchor--positive distance to be smaller than the anchor--negative distance, as in triplet loss, while the second bracket enforces additional separation between \((A, P)\) and the second negative pair \((N', N)\), thereby enhancing global distance structure.

\subsection{Evaluation}\label{app:evaluation}

We assess the reconstructed tree $\hat{\mathcal{T}}$ against a ground-truth tree $\mathcal{T}$ using metrics such as:
\paragraph{Robinson–Foulds (RF) Distance.} The RF distance quantifies the topological difference between two unrooted trees by comparing their sets of partitions, where a partition corresponds to a split in the tree that divides the taxa into two complementary subsets. The metric counts how many partitions differ between the inferred and the true trees. We will generally present normalized RF distance, where 0 implies identical tree topology and 1 implies that no partitions are shared between the two trees.
 
\paragraph{Quartet Distance (QD).} The RF distance is widely used due to its conceptual simplicity. However, it does have limitations. Because it treats all partitions equally, it does not distinguish between topological differences that might have different biological relevance. More importantly, the distance measure is sensitive to tree differences that arise from a cut-and-paste operation. For example, if one leaf vertex is moved from one side of the tree to another, RF might be 1, even if the remaining tree structure is identical. Thus, to consider more detailed subtree structure, we introduce quartet distance\citep{bryant_computing_2000}. Given any four leaf vertices (\ital{i.e.}, a quartet of leaves), there are three possible arrangements of their unrooted tree-graph relationships. Given two tree graphs, we consider all possible quartets of leaves and compute the percent of quartet tree graphs that are different as the quartet distance. For large trees, we may approximate the quartet distance by a sample of the leaf vertices.

For reference, we also compare to a ``direct reconstruction'' approach that uses raw Euclidean distances in gene-expression space, denoted as base RF and base QD. We use $\Delta\%$RF and $\Delta\%$QD to measure the relative improvement in RF or QD over direct raw-data reconstruction (higher is better).
\begin{equation}\label{delta_percentage}
\Delta\% \text{RF} = \frac{\text{RF}_{\text{base}} - \text{RF}_{\text{recon}}}{\text{RF}_{\text{base}}}; \quad\quad
\Delta\% \text{QD} = \frac{\text{QD}_{\text{base}} - \text{QD}_{\text{recon}}}{\text{QD}_{\text{base}}}
\end{equation}

\subsection{Supervised Setting}\label{app:supervised}

\paragraph{Simulation}
Table \ref{tbl:supervised_simulation} summarizes the performance of our method under four simulated scenarios (A–D), each parameterized by distinct signal, noise, and alternate tree noise settings. In each scenario, we vary the dimensionality of the signal features ($d_{\text{sig}}$), Gaussian noise ($d_{\text{noise}}$), and alternate tree noise ($d_{\text{AltSig}}$), as well as the scaling factors $\alpha$ (for noise standard deviation) and $\beta$ (for branch lengths in the alternate tree). We define the signal-to-noise ratio ($\mathrm{SNR}$) by the ratio of the total signal dimensionality to the sum of noise and alternate tree noise dimensions, $\mathrm{SNR} = d_{\text{sig}}/(\,d_{\text{noise}} + n_{T_{\text{Alt}}} \times d_{\text{AltSig}}\,)$. The total feature dimensionality $d$ thus reflects both signal and noise features combined.

Under ``CellTreeQM'' section, we compare two modes: $\mathcal{F}$ (using all features) and $\mathcal{G}$ (applying a Gumbel-based feature gate). The final three columns (“Recall,” “Precis.,” $\Delta$) further illustrate how well the gating module identifies true signal features. We observe that using all features without gating ($\mathcal{F}$) can be prone to noise contamination when $d_{\text{noise}}$ and $d_{\text{AltSig}}$ are large, while the Gumbel-based gate ($\mathcal{G}$) achieves higher Recall and Precision of signal features and remains relatively stable (as seen in $\Delta$). These trends hold across the different scenarios, with certain variations in performance as the SNR, scaling factors $\alpha$ and $\beta$, and total dimensionality d change. In additional scenarios with $n_{\text{leaves}}$ = 128 and 256, similar trends are observed. 

\renewcommand{\arraystretch}{1.2}
\begin{table}[ht]
\centering
\scriptsize
\caption{%
\textbf{Performance comparison on simulated data with $n_{\text{leaves}}=64$ (scenarios A--D).} 
Each row shows parameter choices for the signal ($w,\,d_{\text{sig}}$), Gaussian noise ($d_{\text{noise}},\,\alpha$), and alternate tree noise ($S,\,T,\,d_{\text{AltSig}},\,\beta$). $\alpha$ is the scaling factor of noise stander deviation such that $\sigma_{noise} = \alpha \bar \sigma_{sig}$. $\beta$ is the scaling factor of the branch length range of the alternative tree such that $w_\text{alt}=\beta w$.
We define $\textrm{SNR}$ as
$
\textrm{SNR}=d_{\text{sig}}/(d_{\text{noise}}+T \times d_{\text{AltSig}})
$
Here, $d$ denotes the total feature dimensionality. 
``Direct Recon.'' measures reconstruction when using only the signal features (Signal) vs.\ when using both signal and noise (Noisy). 
``CellTreeQM'' ($\mathcal{F},\mathcal{G}$) indicates whether all features are used ($\mathcal{F}$) or a Gumbel-based feature gate is applied ($\mathcal{G}$). 
Each value is averaged over five runs. 
``G-Gate Performance'' columns show: 
(\emph{Recall}) the percentage of signal features selected (out of all signal features) at the best-performing training step;
(\emph{Precis.}) the percentage of selected features that are signal features at that same step;
($\Delta$ g) the standard deviation of the number of gate changes over the final one-third of training steps.
Each value is averaged over five runs. 
Additional results (including standard deviations) for $n_{\text{leaves}}=128$ and $256$ are provided in the appendix.
}\label{tbl:supervised_simulation}
\scalebox{0.9}{
\begin{tabular}{l|ll|ll|llll|ll|ll|cc|ccc}
\hline
\multirow{2}{*}{}  & \multicolumn{2}{c|}{Signal}                                   & \multicolumn{2}{c|}{Gauss. Noise}                                    & \multicolumn{4}{c|}{Alt. Tree Noise}                                                                                     & \multicolumn{2}{c|}{Summary}                       & \multicolumn{2}{c|}{Direct Recon.}                      & \multicolumn{2}{c|}{CellTreeQM} & \multicolumn{3}{c}{$\mathcal{G}$-Gate Performance} \\
                   & \multicolumn{1}{c}{$w$} & \multicolumn{1}{c|}{$d_\text{sig}$} & \multicolumn{1}{c}{$d_\text{noise}$} & \multicolumn{1}{c|}{$\alpha$} & \multicolumn{1}{c}{$S$} & \multicolumn{1}{c}{$T$} & \multicolumn{1}{c}{$d_\text{AltSig}$} & \multicolumn{1}{c|}{$\beta$} & \multicolumn{1}{c}{SNR} & \multicolumn{1}{c|}{$d$} & \multicolumn{1}{c}{Signal} & \multicolumn{1}{c|}{Noisy} & $\mathcal{F}$    & $\mathcal{G}$   & Recall      & Precis.      & $\Delta g$        \\ \hline
\multirow{3}{*}{A} & 2                       & 20                                  & 20                                   & 0.5                           & 20                      & 1                       & 20                                    & 0.5                          & 0.5                     & 60                       & 0.291                      & 0.615                      & 0.318            & 0.126           & 0.85        & 0.946        & 1.276        \\
                   & 2                       & 50                                  & 50                                   & 0.5                           & 50                      & 1                       & 50                                    & 0.5                          & 0.5                     & 150                      & 0.095                      & 0.413                      & 0.103            & 0.041           & 0.804       & 0.951        & 2.158        \\
                   & 2                       & 100                                 & 100                                  & 0.5                           & 100                     & 1                       & 100                                   & 0.5                          & 0.5                     & 300                      & 0.028                      & 0.275                      & 0.037            & 0.014           & 0.824       & 0.884        & 3.624        \\ \hline
\multirow{3}{*}{B} & 2                       & 20                                  & 100                                  & 0.5                           &                         &                         &                                       &                              & 0.2                     & 120                      & 0.292                      & 0.534                      & 0.512            & 0.182           & 0.94        & 0.775        & 3.064        \\
                   & 5                       & 50                                  & 500                                  & 1                             &                         &                         &                                       &                              & 0.1                     & 550                      & 0.156                      & 0.865                      & 0.884            & 0.370           & 0.9         & 0.424        & 17.267       \\
                   & 10                      & 100                                 & 2000                                 & 2                             &                         &                         &                                       &                              & 0.05                    & 2100                     & 0.157                      & 0.980                      & 0.960            & 0.867           & 0.794       & 0.130        & 105.705      \\ \hline
\multirow{3}{*}{C} & 2                       & 20                                  &                                      &                               & 20                      & 1                       & 20                                    & 0.5                          & 1                       & 40                       & 0.315                      & 0.567                      & 0.233            & 0.108           & 0.89        & 0.981        & 1.217        \\
                   & 5                       & 50                                  &                                      &                               & 50                      & 2                       & 100                                   & 1                            & 0.5                     & 150                      & 0.177                      & 0.951                      & 0.603            & 0.056           & 0.848       & 0.954        & 3.605        \\
                   & 10                      & 100                                 &                                      &                               & 100                     & 2                       & 200                                   & 2                            & 0.5                     & 300                      & 0.144                      & 0.997                      & 0.856            & 0.263           & 0.844       & 0.738        & 12.961       \\ \hline
\multirow{3}{*}{D} & 2                       & 20                                  & 100                                  & 0.5                           & 20                      & 1                       & 20                                    & 0.5                          & 0.17                    & 140                      & 0.308                      & 0.663                      & 0.600            & 0.249           & 0.98        & 0.663        & 3.407        \\
                   & 5                       & 50                                  & 500                                  & 1                             & 50                      & 2                       & 100                                   & 1                            & 0.08                    & 650                      & 0.174                      & 0.948                      & 0.921            & 0.403           & 0.828       & 0.403        & 22.853       \\
                   & 10                      & 100                                 & 2000                                 & 2                             & 100                     & 2                       & 200                                   & 2                            & 0.05                    & 2300                     & 0.144                      & 0.984                      & 0.974            & 0.875           & 0.772       & 0.112        & 129.442      \\ \hline
\end{tabular}
}
\end{table}

\paragraph{Real Data}
The results for the supervised setting on all four real datasets are in Table~\ref{tbl:supervised_full}. Experiments for \textit{C. elegans Small} and \textit{Mid} are repeated 3 times, and the mean and standard deviation are reported. We report the results from only one run of \textit{C. elegans Large} because it runs slowly at the evaluation step to calculate the RF distance of the tree with 295 leaves. We are running more repetitions and will update the appendix as soon as more experiments are finished.
\begin{table}[ht]
\centering
\caption{\textbf{Supervised results on Cell Lineage Benchmark.} 
Direct reconstruction on raw data yields Base RF and Base QDist. Suffix ``-G'' denotes feature gating, and ``-p'' indicates label permutation. Standard deviations are in parentheses. Experiments are repeat 3 times expect C. elegans Large.}
\label{tbl:supervised_full}
\begin{tabular}{lccccc}
\hline
\textbf{Method} & \textbf{Train RF\(\downarrow\)} & \textbf{Test RF\(\downarrow\)} & \textbf{Test QD\(\downarrow\)} & \multicolumn{1}{c}{\textbf{$\Delta$\%RF\(\uparrow\)}} & \multicolumn{1}{c}{\textbf{$\Delta$\%QD\(\uparrow\)}} \\ \hline
\multicolumn{6}{l}{\textbf{C. elegans Small}, Base RF=0.923; Base QD=0.554}                                                                                                                                            \\
CellTreeQM      & \textbf{0.000} (0.00)                    & \textbf{0.286} (0.05)                   & \textbf{0.074} (0.01)                   & \textbf{0.690} (0.05)                                          & \textbf{0.867} (0.02)                                          \\
CellTreeQM-G    & \textbf{0.000} (0.00)                    & \textbf{0.226} (0.02)                   & \textbf{0.084} (0.01)                   & \textbf{0.757} (0.03)                                          & \textbf{0.848} (0.01)                                          \\
CellTreeQM-p    & 0.013 (0.00)                    & 0.566 (0.04)                   & 0.208 (0.03)                   & 0.434 (0.04)                                          & 0.691 (0.05)                                          \\
Triplet         & 0.519 (0.03)                    & 0.741 (0.01)                   & 0.201 (0.01)                   & 0.179 (0.01)                                          & 0.637 (0.02)                                          \\
Triplet-G       & 0.545 (0.03)                    & 0.724 (0.00)                   & 0.207 (0.01)                   & 0.203 (0.03)                                          & 0.631 (0.02)                                          \\
Triplet-p       & 0.677 (0.06)                    & 0.963 (0.01)                   & 0.413 (0.04)                   & 0.037 (0.01)                                          & 0.385 (0.06)                                          \\
Quadruplet      & 0.057 (0.00)                    & 0.492 (0.03)                   & 0.120 (0.00)                   & 0.454 (0.02)                                          & 0.784 (0.01)                                          \\
Quadruplet-G    & 0.061 (0.00)                    & 0.471 (0.04)                   & 0.116 (0.00)                   & 0.484 (0.04)                                          & 0.791 (0.01)                                          \\
Quadruplet-p    & 0.118 (0.00)                    & 0.848 (0.01)                   & 0.310 (0.01)                   & 0.149 (0.01)                                          & 0.538 (0.02)                                          \\ \hline
\multicolumn{6}{l}{\textbf{C. elegans Mid}, Base RF=0.967, Base QD=0.579}                                                                                                                                                \\
CellTreeQM      & \textbf{0.022} (0.00)                    & \textbf{0.513} (0.01)                   & \textbf{0.155} (0.00)                   & \textbf{0.457} (0.02)                                          & \textbf{0.730} (0.01)                                          \\
CellTreeQM-G    & \textbf{0.031} (0.01)                    & \textbf{0.472} (0.00)                   & \textbf{0.130} (0.01)                   & \textbf{0.490} (0.01)                                          & \textbf{0.773} (0.01)                                          \\
CellTreeQM-p    & 0.035 (0.01)                    & 0.831 (0.00)                   & 0.356 (0.02)                   & 0.165 (0.00)                                          & 0.466 (0.03)                                          \\
Triplet         & 0.650 (0.02)                    & 0.811 (0.00)                   & 0.189 (0.00)                   & 0.095 (0.01)                                          & 0.673 (0.01)                                          \\
Triplet-G       & 0.693 (0.02)                    & 0.809 (0.01)                   & 0.183 (0.00)                   & 0.112 (0.01)                                          & 0.682 (0.00)                                          \\
Triplet-p       & 0.878 (0.03)                    & 0.970 (0.01)                   & 0.438 (0.01)                   & 0.026 (0.01)                                          & 0.344 (0.01)                                          \\
Quadruplet      & 0.146 (0.01)                    & 0.557 (0.01)                   & 0.179 (0.01)                   & 0.396 (0.01)                                          & 0.690 (0.01)                                          \\
Quadruplet-G    & 0.174 (0.04)                    & 0.537 (0.01)                   & 0.166 (0.01)                   & 0.423 (0.02)                                          & 0.711 (0.02)                                          \\
Quadruplet-p    & 0.283 (0.03)                    & 0.915 (0.01)                   & 0.395 (0.00)                   & 0.082 (0.01)                                          & 0.409 (0.01)                                          \\ \hline
\multicolumn{6}{l}{\textbf{C. elegans Large}, Base RF=0.949; Base QD=0.586}                                                                                                                                             \\
CellTreeQM      & \textbf{0.113}                           & \textbf{0.568}                          & \textbf{0.111}                          & \textbf{0.401}                                                 & \textbf{0.811}                                                 \\
CellTreeQM-G    & \textbf{0.140}                           & \textbf{0.521}                          & \textbf{0.115}                          & \textbf{0.439}                                                 & \textbf{0.805}                                                 \\
CellTreeQM-p    & 0.195                           & 0.863                          & 0.262                          & 0.137                                                 & 0.606                                                 \\
Triplet         & 0.822                           & 0.863                          & 0.229                          & 0.077                                                 & 0.608                                                 \\
Triplet-G       & 0.808                           & 0.873                          & 0.196                          & 0.073                                                 & 0.664                                                 \\
Triplet-p       & 0.979                           & 0.993                          & 0.424                          & 0.007                                                 & 0.365                                                 \\
Quadruplet      & 0.432                           & 0.675                          & 0.155                          & 0.284                                                 & 0.736                                                 \\
Quadruplet-G    & 0.377                           & 0.712                          & 0.144                          & 0.249                                                 & 0.754                                                 \\
Quadruplet-p    & 0.620                           & 0.932                          & 0.317                          & 0.068                                                 & 0.523                                                 \\ \bottomrule
\end{tabular}
\end{table}

\paragraph{Visualization on C. elegans Small}
\textit{CellTreeQM} embeds the leaves into 128-dimensional space. To have some intuition about the latent space, we project the embeddings to 20 principal components (via PCA), followed by a 2D t-SNE projection. Figure~ \ref{fig:supervised_tsne_all} displays the results on \emph{C. elegans Small}, coloring cells by common ancestors at different hierarchical levels. While triplet and quadruplet losses capture some broad structure, \emph{CellTreeQM} more effectively organizes the leaves according to the underlying tree hierarchy.

In addition to these embeddings, Figure~\ref{fig:supervised_circle_tree} illustrates the final lineage trees reconstructed from the learned distances of \emph{CellTreeQM}, Triplet, and Quadruplet. Each leaf is colored according to its major lineage annotation in \emph{C. elegans}. Notably, \emph{CellTreeQM} yields a more faithful global topology, with tighter, lineage-consistent subtrees. By contrast, trees reconstructed under triplet or quadruplet losses exhibit more dispersed leaf placements, sometimes grouping distantly related lineages together. These visual differences align with our quantitative findings in Table~\ref{tbl:supervsied_celegans_small} and underscore \emph{CellTreeQM}’s advantage in accurately capturing the hierarchical relationships among cells.

\begin{figure}[ht]
    \centering
    \includegraphics[width=\linewidth]{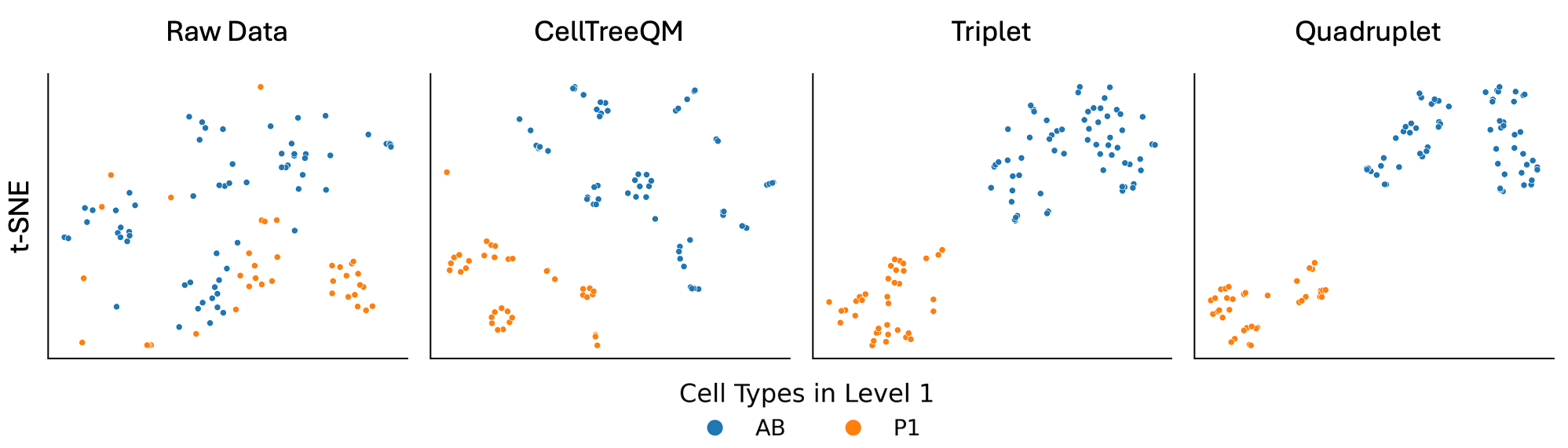}
    \includegraphics[width=\linewidth]{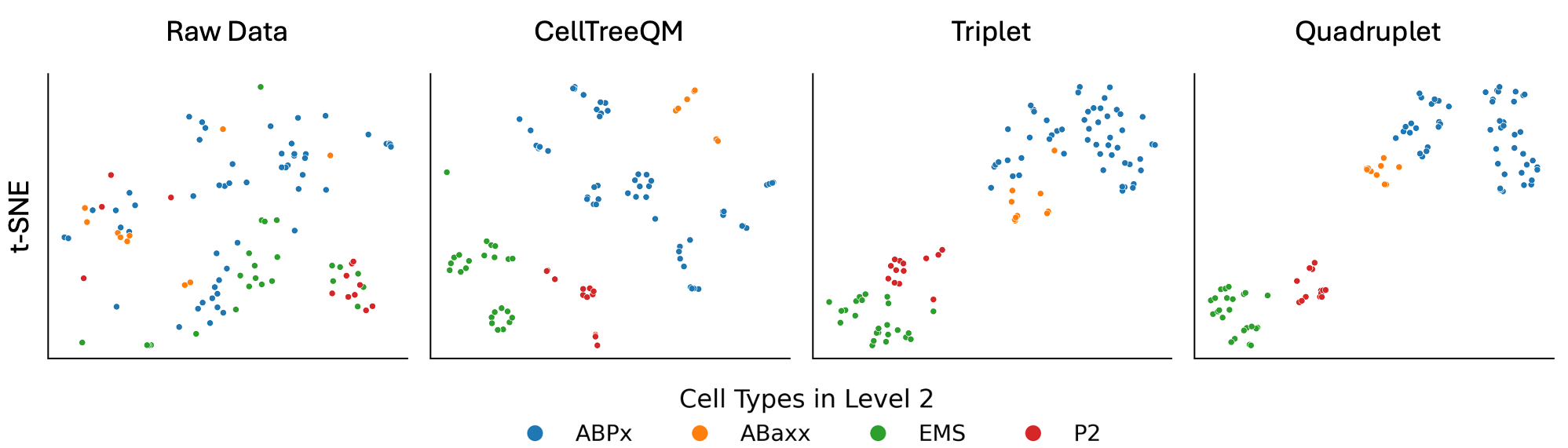}
    \includegraphics[width=\linewidth]{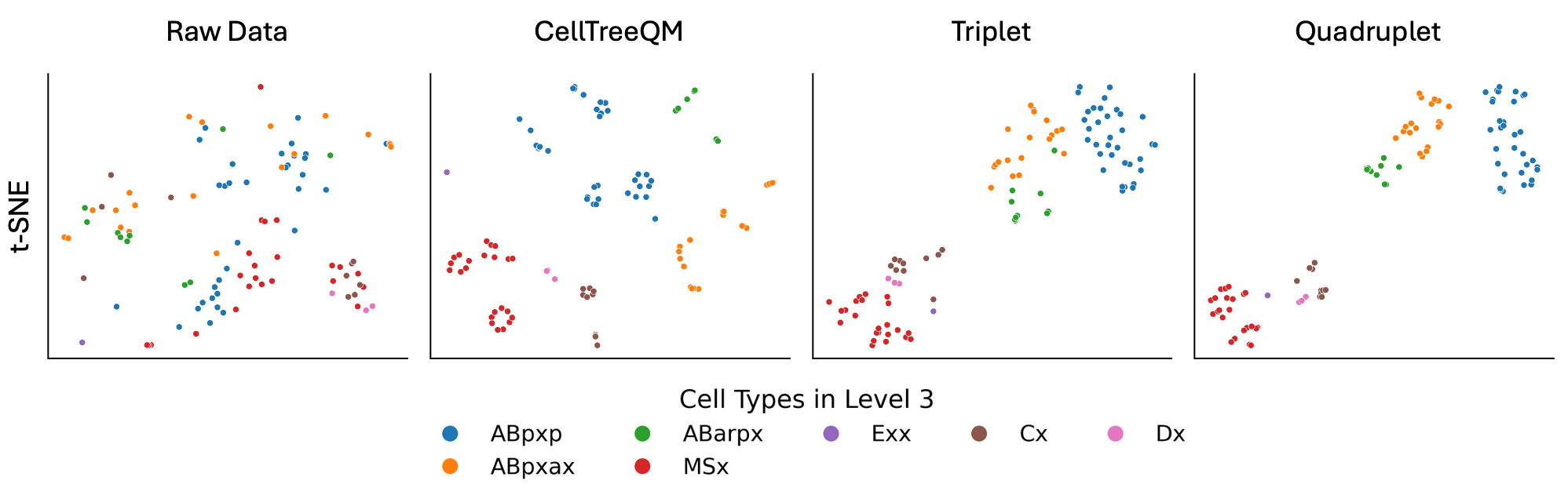}
    \includegraphics[width=\linewidth]{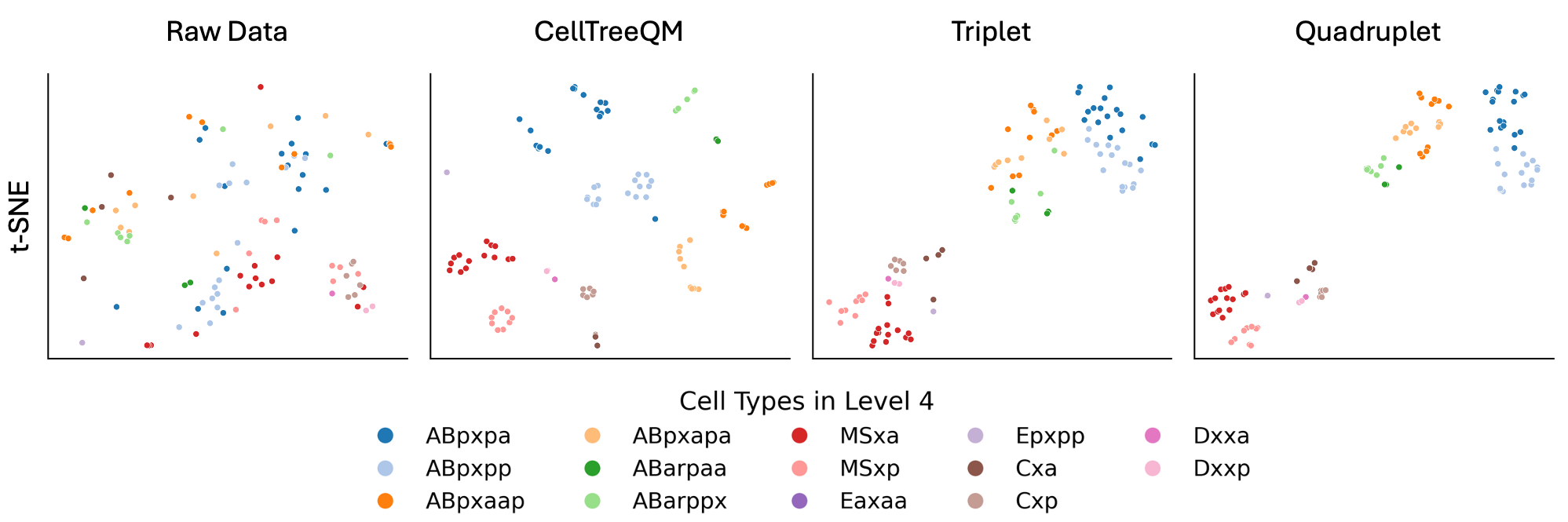}    
    \caption{t-SNE visualization of CellTreeQM embeddings for \emph{C. elegans Small}. The embeddings are first reduced to 20 principal components via PCA before applying t-SNE. Each panel corresponds to a different hierarchical level, with colors representing common ancestors at that level.}
    \label{fig:supervised_tsne_all}
\end{figure}

\begin{figure}[ht]
    \centering
    \includegraphics[width=\linewidth]{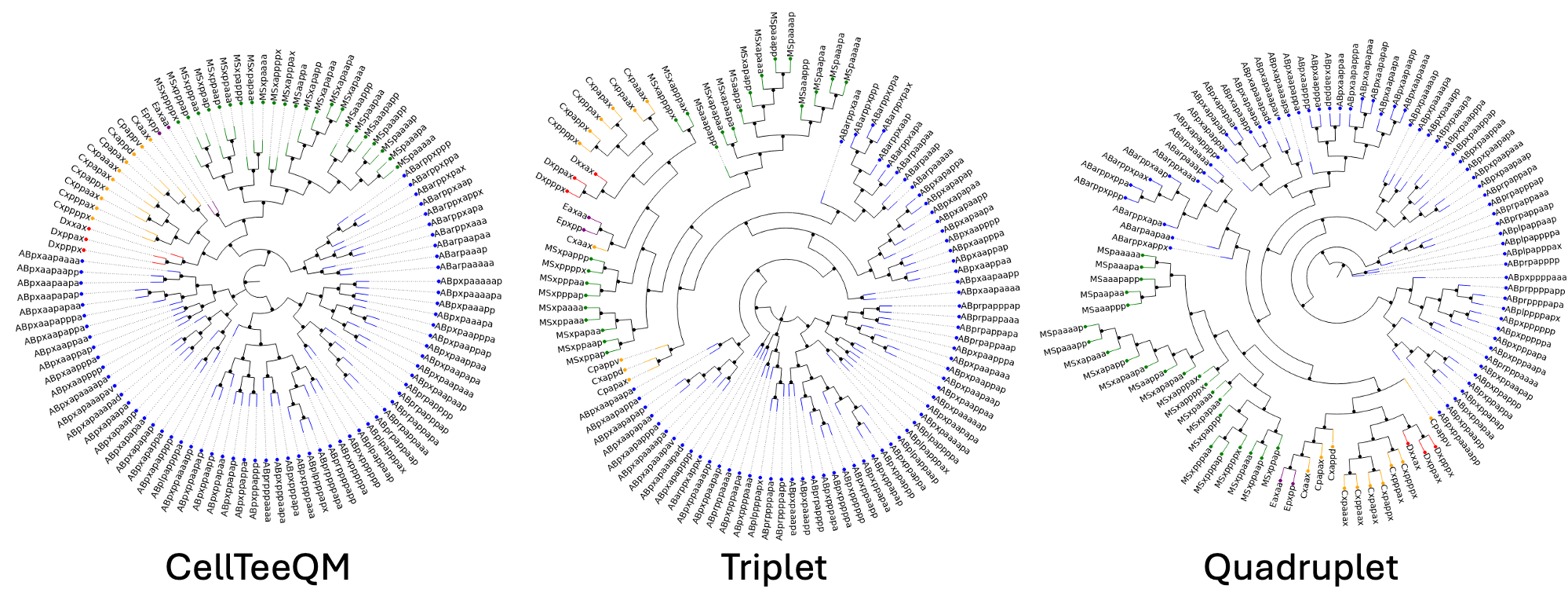}
    \caption{Reconstructed circular lineage tree for the \emph{C. elegans} dataset using the learned embeddings from \emph{CellTreeQM}. Each leaf is color‐coded according to its major lineage (e.g., AB, MS, E, C, D), with the circular layout highlighting the hierarchical structure. The faithful grouping of leaves into cohesive subtrees demonstrates the effectiveness of \emph{CellTreeQM} in capturing global genealogical relationships from the data.}\label{fig:supervised_circle_tree}
\end{figure}

\paragraph{Label Permutation as Null Experiments} To verify that our method is genuinely capturing lineage structure rather than high-dimensional artifacts, we conduct a series of null experiments by permuting labels or data entries in ways that disrupt the underlying biological patterns. Specifically, we apply three types of permutations:
\begin{itemize}[leftmargin=1.5em]
\item \textbf{Leaf Permutation:} Randomly shuffle the labels associated with each leaf.
\item \textbf{Cell Permutation:} Randomly shuffle the cell-wise entries (i.e., rows) within each feature.
\item \textbf{Gene Permutation:} Randomly shuffle the feature-wise entries (i.e., columns) within each cell.
\end{itemize}
Each permutation severs the original lineage relationships embedded in the data. By comparing the model’s performance on these null-permuted datasets to its performance on unpermuted data, we confirm that \emph{CellTreeQM}’s success relies on genuine lineage signals rather than incidental structural artifacts (see Table~\ref{tab:transformer-vs-fc}).

\subsection{Weakly Supervised Setting}

\subsubsection{High-level Partitioning Setting}
We report the \emph{known‐quartet distance} ($\text{K‐QDist}$) and the \emph{unknown‐quartet distance} ($\text{U‐QDist}$).  Unsurprisingly, $\text{K‐QDist}$ remains near zero for successful methods, since those quartets are directly supervised.  Meanwhile, $\text{U‐QDist}$ steadily decreases at deeper prior levels, indicating that CellTreeQM \emph{generalizes} better to unconstrained quartets as more high‐level clades become known.  We also measure $\Delta \mathrm{RF} = \mathrm{RF0} - \mathrm{RF}$, the improvement over a raw‐data baseline.  This gap grows with increased prior, demonstrating that additional top‐level constraints help the model recover more correct branching patterns.

To gauge how well the model recovers the unknown internal structure \emph{within} each clade, we also track local subtree errors at various “levels” of the reconstructed tree.  Even though these subclades were not all explicitly supervised, the table reveals that CellTreeQM obtains lower Robinson–Foulds distances on these deeper substructures—particularly in the presence of more known quartets.

We report the fraction of known quartets at levels of the full balanced dataset with 64 leaves and for the three \textit{C. elegans} datasets in Table \ref{tbl:high_level_partition_quartet_fraction}.


\begin{table}[ht]
    \centering
    \caption{Fraction of Known Quartets at Levels for High-level Partitioning Setting.}
    \label{tbl:high_level_partition_quartet_fraction}
    \begin{tabular}{@{}lcc|cc|cc|cc@{}}
        \toprule
        & \multicolumn{2}{c|}{Brownian64} 
        & \multicolumn{2}{c|}{C. elegans Small} 
        & \multicolumn{2}{c|}{C. elegans Mid} 
        & \multicolumn{2}{c}{C. elegans Large} \\ 
        \midrule
        \textbf{n leaves}   & \multicolumn{2}{c|}{64} & \multicolumn{2}{c|}{102} & \multicolumn{2}{c|}{182} & \multicolumn{2}{c}{295} \\
        \textbf{n quartets} & \multicolumn{2}{c|}{635,376} & \multicolumn{2}{c|}{4,249,575} & \multicolumn{2}{c|}{44,224,635} & \multicolumn{2}{c}{309,177,995} \\
        \midrule
        & \textbf{Counts} & \textbf{Prop.} 
        & \textbf{Counts} & \textbf{Prop.} 
        & \textbf{Counts} & \textbf{Prop.} 
        & \textbf{Counts} & \textbf{Prop.} \\ 
        \midrule
        Level 1 & 246,016  & 0.387   & 1,417,248   & 0.334  & 11,814,336   & 0.2613  & 77,232,330   & 0.250  \\
        Level 2 & 455,040  & 0.716   & 2,340,063   & 0.551   & 29,809,148   & 0.659  & 201,823,758  & 0.653  \\
        Level 3 & 323,008  & 0.508   & 2,633,325   & 0.620  & 27,562,455   & 0.610  & 183,448,315  & 0.593  \\
        Level 4 & 165,600  & 0.261   & 1,892,679   & 0.445  & 17,895,810   & 0.396  & 116,649,187  & 0.377  \\
        \bottomrule
    \end{tabular}
\end{table}

\paragraph{Performance on Simulated Data.}
The results for the high-level partition setting are shown in Figure~\ref{fig:high_level_partition_sim}. The dataset is simulated with $n_{\text{leaves}} = 64$ and a maximum walk of 2 per branch. Each leaf contains 50 signal features and 500 Gaussian noise features, with the same mean and standard deviation as the signal features. Additionally, each leaf has 50 spurious features derived from an alternative tree. There are two alternative trees, each with 32 leaves and the same maximum walk as the true tree. Each experiment is repeated 10 times, and the mean and standard deviation are shown as the central value and error bars in the figure.

\begin{figure}[ht]
    \centering
    \includegraphics[width=\linewidth]{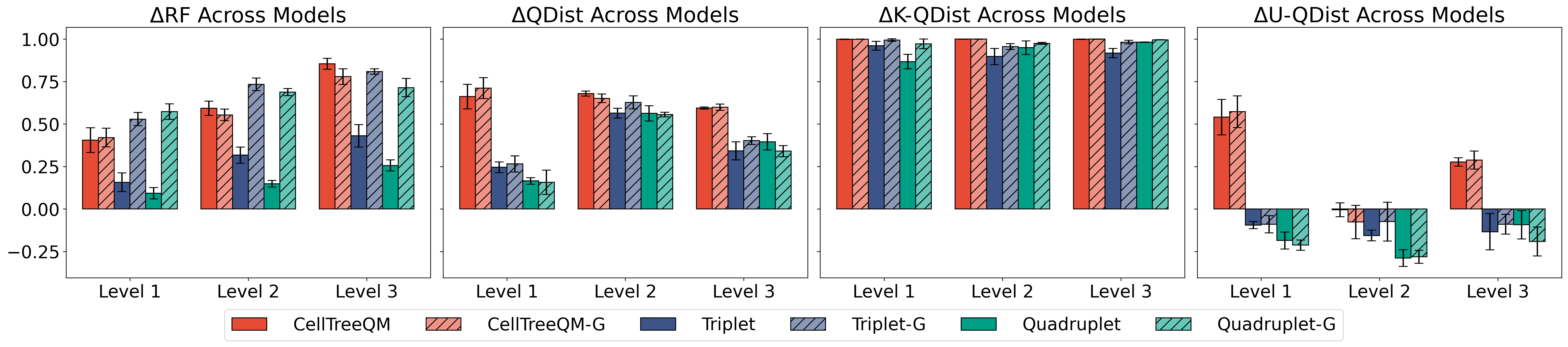}
    \caption{Results for high-level partitioning setting on simulation with 64 leaves.}
    \label{fig:high_level_partition_sim}
\end{figure}

\paragraph{Performance on \textit{C. elegans} Data.}

Under the high-level partition setting, the results for \textit{C. elegans Small} are shown in Figure~\ref{fig:high_level_partition_celegans_dev} and Table~\ref{tbl:high_level_partition_celegans_small_full}. The results for C. elegans Mid. are shown in Figure~\ref{fig:high_level_partition_celegans_large} and Table~\ref{tbl:high_level_partition_celegans_mid_full}.

\begin{table}[ht]
\centering
\footnotesize
\caption{\textbf{Weakly supervised high-level partitioning setting results on \emph{C.~elegans Small} under different partition levels.} 
\(\mathrm{K\text{-}QD}\) and \(\mathrm{U\text{-}QD}\) 
are quartet distances on the \emph{known} and \emph{unknown} quartets, respectively.
he reported values are means
across three runs, with standard deviations in parentheses.}
\label{tbl:high_level_partition_celegans_small_full}
\begin{tabular}{lcccccc}
\hline
\textbf{Method}             & \textbf{RF\(\downarrow\)}                              & \textbf{$\Delta$\%RF\(\uparrow\)} & \textbf{QD\(\downarrow\)}          & \textbf{$\Delta$\%QD\(\uparrow\)} & \textbf{$\Delta$\%K-QD\(\uparrow\)} & \textbf{$\Delta$\%U-QD\(\uparrow\)} \\ \hline
\textbf{Partition Level: 3} &                                          &                                   &                      &                                   &                                     &                                     \\
CellTreeQM                  & \textbf{0.538(0.01)}                     & \textbf{0.403(0.01)}                                  & \textbf{0.156(0.00)} & \textbf{0.715(0.01)}                                  & \textbf{0.999(0.00)}                                    & \textbf{0.247(0.02)}                                    \\
CellTreeQM-G                & \textbf{0.558(0.02)}                     & \textbf{0.379(0.01)}                                  & \textbf{0.171(0.01)} & \textbf{0.690(0.01)}                                  & \textbf{0.999(0.00)}                                    & 0.182(0.02)                                             \\
Triplet                     & 0.702(0.01)                              & 0.232(0.01)                                           & 0.264(0.01)          & 0.518(0.01)                                           & 0.722(0.02)                                             & \textbf{0.183(0.01)}                                    \\
Triplet-G                   & 0.694(0.02)                              & 0.236(0.02)                                           & 0.273(0.00)          & 0.503(0.01)                                           & \textbf{0.710(0.01)}                                    & 0.159(0.01)                                             \\
Quadruplet                  & 0.854(0.01)                              & 0.081(0.02)                                           & 0.241(0.01)          & 0.560(0.02)                                           & 0.971(0.01)                                             & \textbf{-0.117(0.04)}                                   \\
Quadruplet-G                & 0.818(0.02)                              & 0.097(0.03)                                           & 0.241(0.00)          & \textbf{0.560(0.01)}                                  & 0.979(0.00)                                             & -0.130(0.03)                                            \\ \hline
\textbf{Partition Level: 2} &                                          & \multicolumn{1}{c}{}                                  &                      & \multicolumn{1}{c}{}                                  & \multicolumn{1}{c}{}                                    & \multicolumn{1}{c}{}                                    \\
CellTreeQM                  & \multicolumn{1}{r}{\textbf{0.657(0.01)}} & \textbf{0.286(0.02)}                                  & \textbf{0.129(0.01)} & \textbf{0.764(0.03)}                                  & \textbf{0.998(0.00)}                                    & \textbf{0.511(0.06)}                                    \\
CellTreeQM-G                & \textbf{0.626(0.03)}                     & \textbf{0.311(0.04)}                                  & \textbf{0.117(0.01)} & \textbf{0.787(0.01)}                                  & \textbf{0.997(0.00)}                                    & \textbf{0.558(0.03)}                                    \\
Triplet                     & 0.773(0.02)                              & 0.168(0.02)                                           & 0.326(0.01)          & 0.405(0.01)                                           & \textbf{0.726(0.01)}                                    & 0.058(0.02)                                             \\
Triplet-G                   & 0.763(0.02)                              & 0.156(0.02)                                           & 0.322(0.00)          & 0.414(0.01)                                  & 0.726(0.01)                                             & 0.074(0.01)                                             \\
Quadruplet                  & 0.871(0.01)                              & 0.052(0.02)                                           & 0.286(0.00)          & 0.477(0.00)                                           & 0.995(0.00)                                    & -0.084(0.01)                                            \\
Quadruplet-G                & 0.861(0.01)                              & \textbf{0.063(0.02)}                                  & 0.282(0.00)          & 0.487(0.01)                                           & 0.991(0.00)                                             & -0.060(0.02)                                            \\ \hline
\textbf{Partition Level: 1} &                                          & \multicolumn{1}{c}{}                                  &                      & \multicolumn{1}{c}{}                                  & \multicolumn{1}{c}{}                                    & \multicolumn{1}{c}{}                                    \\
CellTreeQM                  & \textbf{0.773(0.02)}                     & \textbf{0.164(0.01)}                                  & \textbf{0.209(0.03)} & \textbf{0.619(0.05)}                                  & \textbf{0.998(0.00)}                                    & \textbf{0.467(0.07)}                                    \\
CellTreeQM-G                & \textbf{0.793(0.02)}                     & \textbf{0.128(0.02)}                                  & \textbf{0.191(0.00)} & \textbf{0.652(0.01)}                                  & \textbf{0.997(0.00)}                                    & \textbf{0.515(0.01)}                                    \\
Triplet                     & 0.813(0.02)                              & 0.120(0.02)                                           & 0.352(0.01)          & 0.357(0.02)                                           & 0.831(0.01)                                             & 0.168(0.02)                                             \\
Triplet-G                   & 0.801(0.01)                              & 0.112(0.01)                                           & 0.349(0.00)          & 0.363(0.01)                                           & 0.833(0.02)                                             & 0.175(0.00)                                             \\
Quadruplet                  & 0.859(0.01)                              & 0.061(0.00)                                           & 0.423(0.01)          & 0.228(0.01)                                           & 0.672(0.02)                                             & 0.051(0.01)                                    \\
Quadruplet-G                & 0.861(0.01)                              & 0.058(0.02)                                  & 0.426(0.01)          & 0.223(0.01)                                           & 0.662(0.04)                                             & 0.049(0.01)                                             \\ \hline
\end{tabular}
\end{table}
\begin{figure}[ht]
    \centering
    \includegraphics[width=\linewidth]{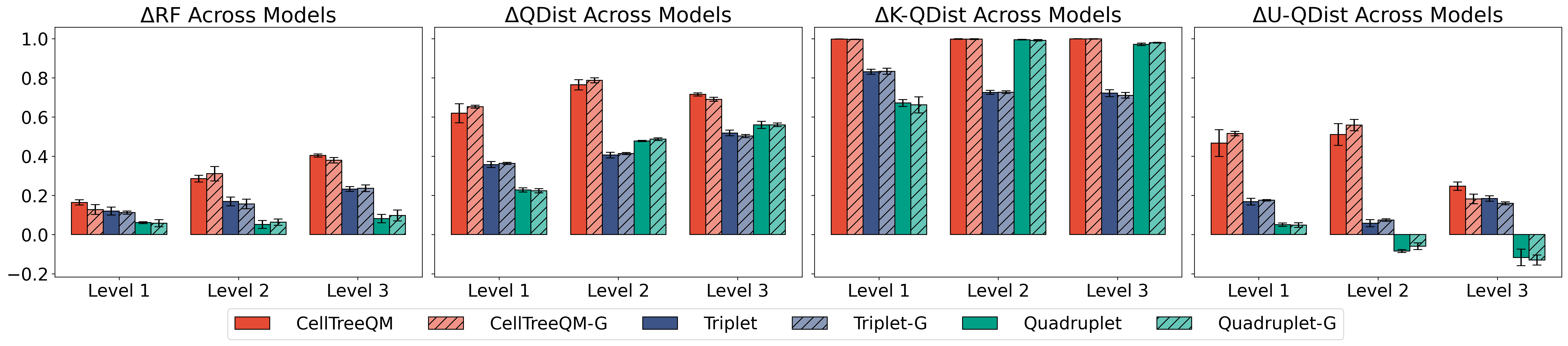}
    \caption{Results for high-level partitioning setting on \textit{C.elegans Small} dataset. The x-axis is the level of partition.}
    \label{fig:high_level_partition_celegans_dev}
\end{figure}

\begin{table}[ht]
\centering
\footnotesize
\caption{\textbf{Weakly supervised high-level partitioning setting results on \emph{C.~elegans Mid} under different partition levels.} 
\(\mathrm{K\text{-}QD}\) and \(\mathrm{U\text{-}QD}\) 
are quartet distances on the \emph{known} and \emph{unknown} quartets, respectively.}
\label{tbl:high_level_partition_celegans_mid_full}
\begin{tabular}{lrrrrrr}
\hline
\textbf{Method}             & \textbf{RF\(\downarrow\)}                              & \textbf{$\Delta$\%RF\(\uparrow\)} & \textbf{QD\(\downarrow\)}          & \textbf{$\Delta$\%QD\(\uparrow\)} & \textbf{$\Delta$\%K-QD\(\uparrow\)} & \textbf{$\Delta$\%U-QD\(\uparrow\)} \\ \hline
\textbf{Partition Level: 3} & \multicolumn{1}{l}{}            & \multicolumn{1}{c}{}                                  & \multicolumn{1}{l}{}            & \multicolumn{1}{c}{}                                  & \multicolumn{1}{c}{}                                    & \multicolumn{1}{c}{}                                    \\
CellTreeQM                  & \textbf{0.724(0.01)}            & \textbf{0.207(0.01)}                                  & \textbf{0.203(0.00)}            & \textbf{0.643(0.01)}                                  & \textbf{1.000(0.00)}                                    & \textbf{0.147(0.02)}                                    \\
CellTreeQM-G                & \textbf{0.728(0.01)}            & \textbf{0.188(0.01)}                                  & \textbf{0.203(0.00)}            & \textbf{0.641(0.00)}                                  & \textbf{1.000(0.00)}                                    & \textbf{0.142(0.01)}                                    \\
Triplet                     & 0.811(0.01)                     & 0.112(0.00)                                           & 0.287(0.01)                     & 0.494(0.01)                                           & 0.831(0.01)                                             & 0.026(0.02)                                             \\
Triplet-G                   & 0.785(0.01)                     & 0.136(0.01)                                           & 0.289(0.01)                     & 0.490(0.01)                                           & 0.822(0.01)                                             & 0.029(0.01)                                             \\
Quadruplet                  & 0.894(0.01)                     & 0.008(0.02)                                           & 0.254(0.01)                     & 0.551(0.02)                                           & 0.993(0.00)                                             & -0.062(0.03)                                            \\
Quadruplet-G                & 0.900(0.00)                     & -0.000(0.01)                                          & 0.254(0.00)                     & 0.553(0.01)                                           & 0.993(0.00)                                             & -0.062(0.02)                                            \\ \hline
\textbf{Partition Level: 2} & \multicolumn{1}{l}{}            & \multicolumn{1}{c}{}                                  & \multicolumn{1}{l}{}            & \multicolumn{1}{c}{}                                  & \multicolumn{1}{c}{}                                    & \multicolumn{1}{c}{}                                    \\
CellTreeQM                  & \textbf{0.822(0.01)}            & \textbf{0.110(0.01)}                                  & \textbf{0.144(0.00)}            & \textbf{0.746(0.01)}                                  & \textbf{0.999(0.00)}                                    & \textbf{0.235(0.02)}                                    \\
CellTreeQM-G                & 0.835(0.01)                     & 0.079(0.01)                                           & \textbf{0.120(0.02)}            & \textbf{0.788(0.04)}                                  & \textbf{0.999(0.00)}                                    & \textbf{0.360(0.11)}                                    \\
Triplet                     & \textbf{0.826(0.01)}            & \textbf{0.091(0.02)}                                  & 0.199(0.00)                     & 0.648(0.01)                                           & 0.910(0.01)                                             & 0.114(0.01)                                             \\
Triplet-G                   & 0.835(0.01)                     & 0.089(0.01)                                           & 0.204(0.00)                     & 0.642(0.01)                                           & 0.908(0.01)                                             & 0.101(0.03)                                             \\
Quadruplet                  & 0.928(0.01)                     & -0.018(0.00)                                          & 0.221(0.00)                     & 0.609(0.00)                                           & \textbf{0.999(0.00)}                                    & -0.188(0.00)                                            \\
Quadruplet-G                & 0.922(0.01)                     & -0.021(0.01)                                          & 0.223(0.00)                     & 0.606(0.00)                                           & 0.998(0.00)                                             & -0.190(0.01)                                            \\ \hline
\textbf{Partition Level: 1} & \multicolumn{1}{l}{}            & \multicolumn{1}{c}{}                                  & \multicolumn{1}{l}{}            & \multicolumn{1}{c}{}                                  & \multicolumn{1}{c}{}                                    & \multicolumn{1}{c}{}                                    \\
CellTreeQM                  & \textbf{0.902(0.00)}            & \textbf{0.006(0.01)}                                  & \textbf{0.261(0.01)}            & \textbf{0.539(0.02)}                                  & \textbf{0.999(0.00)}                                    & \textbf{0.414(0.03)}                                    \\
CellTreeQM-G                & \textbf{0.904(0.00)}            & \textbf{0.014(0.01)}                                  & \textbf{0.298(0.00)}            & \textbf{0.474(0.01)}                                  & \textbf{0.999(0.00)}                                    & \textbf{0.331(0.01)}                                    \\
Triplet                     & 0.850(0.01)                     & 0.063(0.01)                                           & 0.364(0.00)                     & 0.357(0.00)                                           & 0.831(0.01)                                             & 0.228(0.00)                                             \\
Triplet-G                   & 0.869(0.00)                     & 0.035(0.01)                                           & 0.369(0.00)                     & 0.350(0.01)                                           & 0.821(0.02)                                             & 0.222(0.01)                                             \\
Quadruplet                  & 0.896(0.01)                     & 0.015(0.03)                                           & 0.437(0.00)                     & 0.229(0.00)                                           & 0.618(0.03)                                             & 0.123(0.01)                                             \\
Quadruplet-G                & 0.885(0.01)                     & 0.022(0.02)                                           & 0.433(0.01)                     & 0.235(0.01)                                           & 0.637(0.02)                                             & 0.127(0.01)                                             \\ \hline
\end{tabular}
\end{table}
\begin{figure}[ht]
    \centering
    \includegraphics[width=\linewidth]{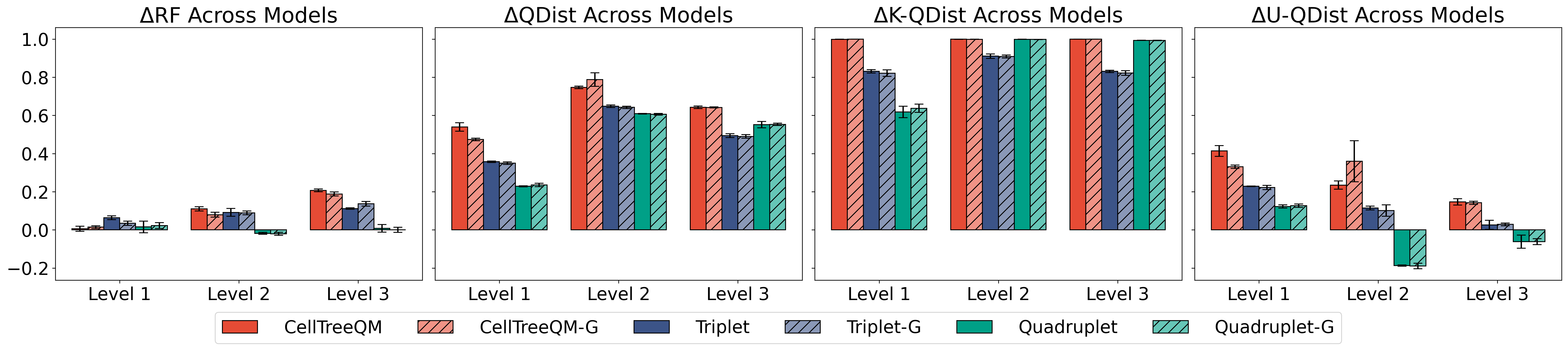}
    \caption{Results for high-level partitioning setting on \textit{C.elegans Mid.} dataset. The x-axis is the level of partition.}
    \label{fig:high_level_partition_celegans_large}
\end{figure}

\subsubsection{Partially Leaf-labeled Setting}
In many biological systems, researchers may know a coarse-grained lineage structure up to a certain “level” of the tree but not the fine-grained branching events beneath it. Concretely, let us define a level as the number of branching steps from the root. We assume we know how leaves are divided into clades at level, as well as the relationships among these clades (see Fig.~\ref{fig:intro-framework}, left). Hence, each leaf is assigned to exactly one of these high-level clades, but the tree structure \emph{within} each clade remains unknown.

When forming quartets under this partial knowledge, some quartets still have ambiguous structure (``unknown quartets''). Specifically, a quartet is unknown if all four leaves lie in the same clade or if three of the leaves come from the same clade, because the high-level partition does not determine the precise branching among these leaves. Quartets with leaves spanning different clades (e.g., two leaves from one clade and two from another) become ``known quartets,'' whose topological order \emph{can} be inferred from the clade-level prior. In our weakly supervised training, we compute the additivity loss \emph{only} on these known quartets, since their correct topology is implied by the high-level partitions.

Table~\ref{tbl:partial_labeled_known_fraction_stats} reports the Known, Partial, and Unknown counts and proportions for different known fractions in \textit{C. elegans Small} and \textit{Mid.} datasets. The experiments for \textit{C. elegans Small} are repeated 10 times. The experiments for \textit{C. elegans Mid.} are only run once. Since we try to iterate all the quartets during the evaluation step to have a reliable count about the small fraction of QD for unknown quartets, this is computationally intensive. We are running more repeats and will update the appendix as soon as more repetitions are done.

\begin{table}[ht]
    \centering
    \caption{Known, Partial, and Unknown counts and proportions for different known fractions in \textit{C. elegans} datasets for partially leaf-labeled setting.}
    \label{tbl:partial_labeled_known_fraction_stats}
    \begin{tabular}{@{}lcc|cc@{}}
        \toprule
        & \multicolumn{2}{c|}{\textbf{C. elegans Small}} 
        & \multicolumn{2}{c}{\textbf{C. elegans Mid}} \\ 
        \midrule
        \textbf{Known Fraction} & \textbf{Count} & \textbf{Prop.} 
        & \textbf{Count} & \textbf{Prop.} \\ 
        \midrule
        \multicolumn{5}{l}{\textbf{Known Fraction: 0.8}} \\ 
        Known     & 1,663,740  & 0.392  & 18,163,860  & 0.402  \\
        Partial   & 2,579,850  & 0.607  & 26,982,990  & 0.597  \\
        Unknown   & 5,985      & 0.001  & 66,045      & 0.001  \\ 
        \midrule
        \multicolumn{5}{l}{\textbf{Known Fraction: 0.5}} \\ 
        Known     & 249,900    & 0.059  & 2,672,670   & 0.059  \\
        Partial   & 3,749,775  & 0.882  & 39,746,070  & 0.879  \\
        Unknown   & 249,900    & 0.059  & 2,794,155   & 0.062  \\ 
        \midrule
        \multicolumn{5}{l}{\textbf{Known Fraction: 0.3}} \\ 
        Known     & 27,405     & 0.006  & 316,251     & 0.007  \\
        Partial   & 3,193,380  & 0.752  & 33,887,268  & 0.750  \\
        Unknown   & 1,028,790  & 0.242  & 11,009,376  & 0.244  \\ 
        \bottomrule
    \end{tabular}
\end{table}

\paragraph{Performance on Simulated Data.}
The results for the partial-labeled setting are shown in Figure~\ref{fig:partial_label_sim} with known fractions as $35\%, 55\%, 85\%$. The dataset is simulated with $n_{\text{leaves}} = 128$ and a maximum walk of 2 per branch. Each leaf contains 50 signal features and 500 Gaussian noise features, with the same mean and standard deviation as the signal features. Additionally, each leaf has 50 spurious features derived from an alternative tree. There are two alternative trees, each with 64 leaves and the same maximum walk as the true tree. Each experiment is repeated 10 times, and the mean and standard deviation are shown as the central value and error bars in the figure.

\begin{figure}[ht]
    \centering
    \includegraphics[width=\linewidth]{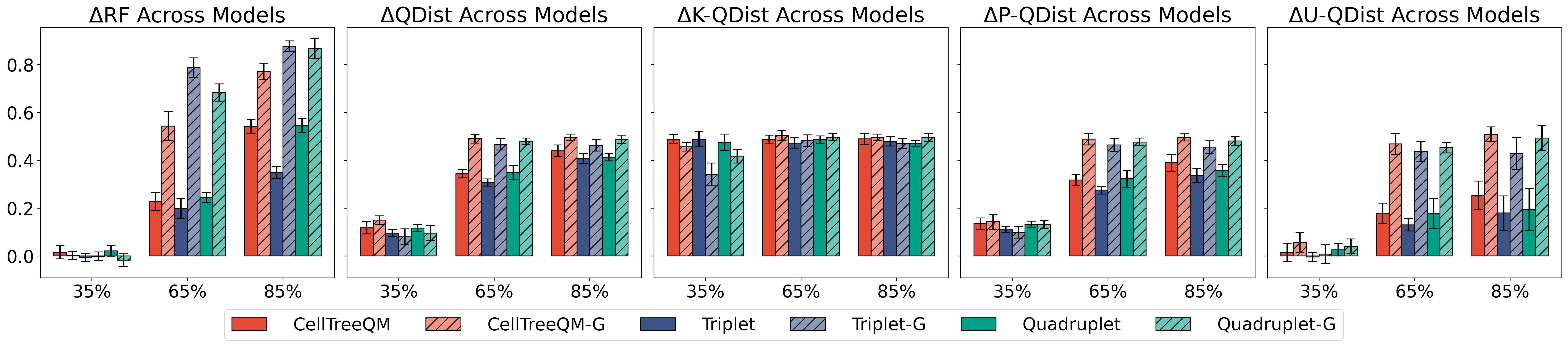}
    \caption{Results for partially leaf-labeled setting on simulation with 128 leaves dataset. The x-axis is the percentage of known leaves.}.
    \label{fig:partial_label_sim}
\end{figure}

\paragraph{Performance on Real Data}

The results for C.elegans Small and Mid are summarized in Table~\ref{tbl:partial_label_elegans_small_full} and \ref{tbl:partial_label_elegans_mid_full} for fraction $30\%, 50\%, 80\%$. The results are also visualized by bar plots in Figure.~ \ref{fig:partial_label_celegans_dev} and \ref{fig:partial_label_celegans_large}. The experiments for \textit{C. elegans Small} are repeated 10 times. The experiments for C. elegans Mid are repeated 5 times.

\begin{table}[ht]
\centering
\caption{Weakly supervised Partial-labeled setting results across known fractions for C. elegans Small. The reported values are means across ten runs, with standard deviations in parentheses.}
\label{tbl:partial_label_elegans_small_full}
\scalebox{0.9}{
\begin{tabular}{@{}lcccccc@{}}
\toprule
\textbf{Method}              & \textbf{Train RF\(\downarrow\)} & \textbf{$\Delta$\%RF\(\uparrow\)} & \textbf{$\Delta$\%QD\(\uparrow\)} & \textbf{$\Delta$\%K-QD\(\uparrow\)} & \textbf{$\Delta$\%P-QD\(\uparrow\)} & \textbf{$\Delta$\%U-QD\(\uparrow\)} \\ \midrule
\textbf{Known Fraction: 0.8}                                                                                                                                                                                                        \\
CellTreeQM                   & \textbf{0.024(0.03)}            & \textbf{0.448(0.06)}              & \textbf{0.842(0.05)}              & \textbf{0.999(0.00)}                & \textbf{0.742(0.07)}                & \textbf{0.465(0.13)}                \\
CellTreeQM-G                 & \textbf{0.031(0.04)}            & \textbf{0.411(0.07)}                       & \textbf{0.837(0.03)}              & \textbf{0.999(0.00)}                & \textbf{0.700(0.06)}                & 0.371(0.13)                \\
Triplet                      & 0.454(0.06)                     & 0.175(0.05)                       & 0.728(0.03)                       & 0.895(0.01)                         & 0.624(0.05)                & 0.339(0.14)                \\
Triplet-G                    & 0.419(0.07)                     & 0.201(0.05)                       & 0.739(0.03)                       & 0.895(0.02)                         & 0.618(0.05)                         & 0.394(0.09)                         \\
Quadruplet                   & 0.066(0.04)                     & 0.403(0.03)              & 0.796(0.06)                       & 0.953(0.02)                         & 0.697(0.09)                         & \textbf{0.435(0.22)}                \\
Quadruplet-G                 & 0.060(0.04)                     & \textbf{0.411(0.03)}              & 0.784(0.04)                       & 0.936(0.01)                         & 0.662(0.05)                         & 0.320(0.09)                         \\ \midrule
\textbf{Known Fraction: 0.5} &                                 &                                   &                                   &                                     &                                     &                                     \\
CellTreeQM                   & \textbf{0.012(0.01)}            & 0.092(0.05)                       & \textbf{0.609(0.05)}              & \textbf{0.999(0.00)}                & \textbf{0.598(0.06)}                & \textbf{0.398(0.05)}                \\
CellTreeQM-G                 & \textbf{0.016(0.01)}            & 0.100(0.04)                       & \textbf{0.602(0.04)}              & \textbf{0.999(0.00)}                & \textbf{0.571(0.03)}                & \textbf{0.352(0.05)}                \\
Triplet                      & 0.303(0.09)                     & 0.049(0.04)                       & 0.505(0.05)                       & 0.879(0.02)                         & 0.493(0.05)                         & 0.304(0.06)                \\
Triplet-G                    & 0.319(0.09)                     & 0.067(0.03)                       & 0.520(0.05)                       & 0.883(0.02)                         & 0.499(0.05)                         & 0.308(0.05)                         \\
Quadruplet                   & 0.023(0.02)            & 0.115(0.04)                       & 0.549(0.04)                       & 0.934(0.03)                         & 0.537(0.04)                & 0.340(0.06)                \\
Quadruplet-G                 & 0.039(0.03)                     & 0.103(0.04)                       & 0.540(0.05)                       & 0.942(0.02)                         & 0.515(0.04)                         & 0.299(0.05)                         \\ \midrule
\textbf{Known Fraction: 0.3} &                                 &                                   &                                   &                                     &                                     &                                     \\
CellTreeQM                   & \textbf{0.000(0.00)}            & -0.023(0.02)                      & \textbf{0.368(0.06)}              & \textbf{1.000(0.00)}                & \textbf{0.401(0.06)}                & 0.250(0.06)                \\
CellTreeQM-G                 & \textbf{0.008(0.02)}            & -0.020(0.02)                      & 0.336(0.08)              & \textbf{0.999(0.00)}                & 0.340(0.08)                         & 0.180(0.07)                         \\
Triplet                      & 0.156(0.06)                     & -0.001(0.02)                      & 0.358(0.04)                       & 0.889(0.02)                         & 0.384(0.04)                         & \textbf{0.263(0.05)}                \\
Triplet-G                    & 0.169(0.09)                     & -0.017(0.03)                      & 0.352(0.03)                       & 0.890(0.03)                & 0.363(0.03)                         & 0.220(0.04)                         \\
Quadruplet                   & 0.008(0.02)                     & -0.020(0.03)                      & 0.358(0.03)                       & 0.918(0.04)                         & \textbf{0.386(0.03)}                & \textbf{0.259(0.02)}                \\
Quadruplet-G                 & 0.029(0.05)                     & -0.016(0.03)                      & \textbf{0.364(0.06)}              & 0.940(0.03)                         & 0.362(0.06)                & 0.208(0.07)                \\ \bottomrule
\end{tabular}
}
\end{table}

\begin{figure}[ht]
    \centering
    \includegraphics[width=\linewidth]{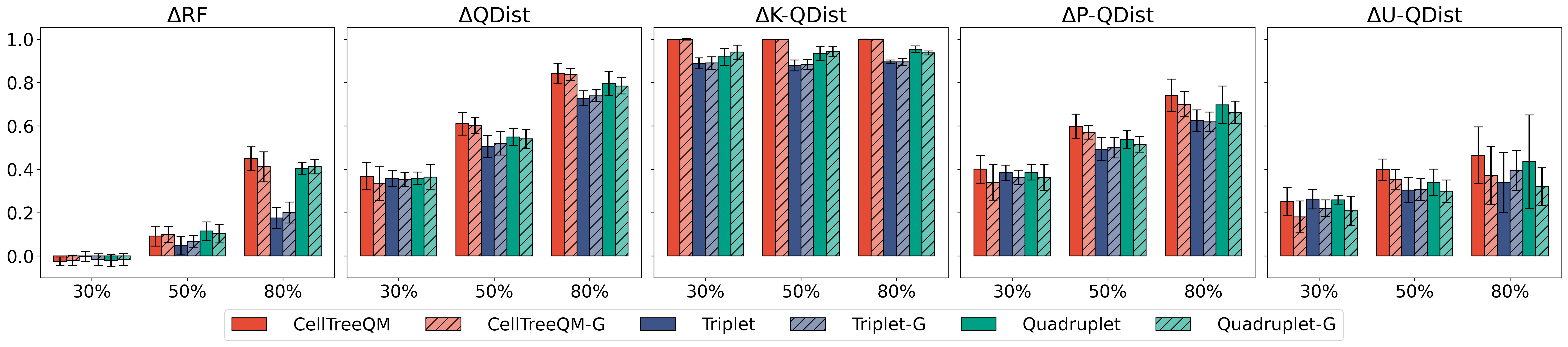}
    \caption{Results for partially leaf-labeled setting on \textit{C. elegans Small} dataset. The x-axis is the percentage of known leaves.}
    \label{fig:partial_label_celegans_dev}
\end{figure}

\begin{table}[ht]
\centering
\caption{Weakly supervised partially leaf-labeled setting results across known fractions for \textit{C. elegans Mid.} dataset. The reported values are means across five runs, with standard deviations in parentheses}
\label{tbl:partial_label_elegans_mid_full}
\scalebox{0.9}{
\begin{tabular}{@{}lcccccc@{}}
\toprule
\textbf{Method}              & \textbf{Train RF\(\downarrow\)} & \textbf{$\Delta$\%RF\(\uparrow\)} & \textbf{$\Delta$\%QD\(\uparrow\)} & \textbf{$\Delta$\%K-QD\(\uparrow\)} & \textbf{$\Delta$\%P-QD\(\uparrow\)} & \textbf{$\Delta$\%U-QD\(\uparrow\)} \\ \midrule
\textbf{Known Fraction: 0.8} &                                 &                                   &                                   &                                     &                                     &                                     \\
CellTreeQM                   & \textbf{0.108(0.02)}            & \textbf{0.339(0.03)}              & \textbf{0.840(0.02)}              & \textbf{1.000(0.00)}                & \textbf{0.733(0.04)}                & 0.407(0.07)                         \\
CellTreeQM-G                 & \textbf{0.150(0.05)}            & 0.309(0.04)                       & \textbf{0.847(0.06)}              & \textbf{1.000(0.00)}                & 0.718(0.10)                         & \textbf{0.441(0.13)}                \\
Triplet                      & 0.628(0.02)                     & 0.090(0.02)                       & 0.804(0.02)                       & 0.921(0.01)                         & \textbf{0.725(0.04)}                & \textbf{0.503(0.10)}                \\
Triplet-G                    & 0.656(0.02)                     & 0.082(0.01)                       & 0.801(0.03)                       & 0.918(0.01)                         & 0.681(0.07)                         & 0.431(0.11)                         \\
Quadruplet                   & 0.169(0.02)                     & \textbf{0.345(0.03)}              & 0.803(0.03)                       & 0.981(0.01)                         & 0.684(0.05)                         & 0.343(0.09)                         \\
Quadruplet-G                 & 0.250(0.07)                     & 0.271(0.07)                       & 0.817(0.02)                       & 0.980(0.00)                         & 0.691(0.03)                         & 0.351(0.07)                         \\ \midrule
\textbf{Known Fraction: 0.5} &                                 &                                   &                                   &                                     &                                     &                                     \\
CellTreeQM                   & \textbf{0.055(0.01)}            & 0.069(0.02)                       & \textbf{0.589(0.03)}              & \textbf{1.000(0.00)}                & \textbf{0.579(0.03)}                & 0.349(0.03)                         \\
CellTreeQM-G                 & 0.180(0.10)                     & 0.029(0.01)                       & \textbf{0.609(0.07)}              & \textbf{0.991(0.01)}                & 0.573(0.06)                         & 0.344(0.07)                         \\
Triplet                      & 0.491(0.03)                     & 0.025(0.03)                       & 0.569(0.02)                       & 0.923(0.01)                         & 0.560(0.02)                         & \textbf{0.361(0.04)}                \\
Triplet-G                    & 0.509(0.06)                     & -0.004(0.02)                      & 0.571(0.08)                       & 0.923(0.01)                         & 0.520(0.08)                         & 0.308(0.11)                         \\
Quadruplet                   & \textbf{0.093(0.02)}            & 0.058(0.03)                       & 0.587(0.03)                       & 0.988(0.00)                         & \textbf{0.576(0.03)}                & \textbf{0.365(0.05)}                \\
Quadruplet-G                 & 0.180(0.10)                     & 0.052(0.02)                       & 0.562(0.02)                       & 0.984(0.01)                         & 0.526(0.04)                         & 0.296(0.05)                         \\ \midrule
\textbf{Known Fraction: 0.3} &                                 &                                   &                                   &                                     &                                     &                                     \\
CellTreeQM                   & \textbf{0.000(0.00)}            & -0.026(0.02)                      & 0.385(0.05)                       & \textbf{1.000(0.00)}                & \textbf{0.420(0.05)}                & \textbf{0.259(0.06)}                \\
CellTreeQM-G                 & \textbf{0.027(0.03)}            & -0.036(0.02)                      & \textbf{0.400(0.04)}              & \textbf{0.999(0.00)}                & 0.393(0.04)                         & 0.227(0.05)                         \\
Triplet                      & 0.275(0.10)                     & -0.028(0.02)                      & 0.378(0.07)                       & 0.935(0.01)                         & 0.412(0.07)                         & 0.257(0.07)                         \\
Triplet-G                    & 0.306(0.13)                     & -0.029(0.01)                      & 0.311(0.06)                       & 0.932(0.01)                         & 0.332(0.05)                         & 0.164(0.06)                         \\
Quadruplet                   & 0.008(0.02)                     & -0.022(0.04)                      & 0.366(0.02)                       & 0.979(0.01)                         & 0.402(0.03)                         & 0.239(0.02)                         \\
Quadruplet-G                 & 0.067(0.06)                     & -0.045(0.02)                      & \textbf{0.403(0.05)}              & 0.987(0.01)                         & \textbf{0.424(0.05)}                & \textbf{0.265(0.05)}                \\ \bottomrule
\end{tabular}
}
\end{table}
\begin{figure}[ht]
    \centering
    \includegraphics[width=\linewidth]{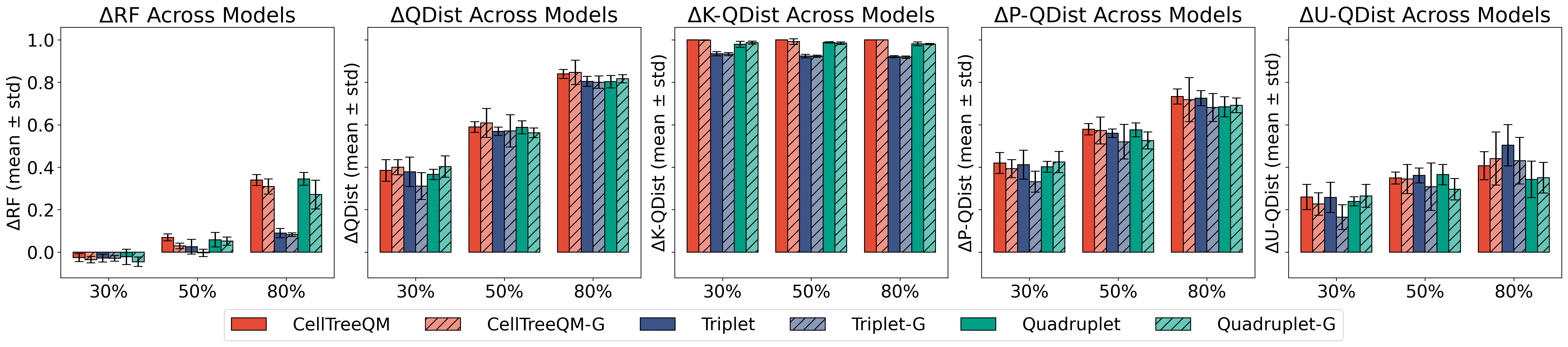}
    \caption{Results for partially leaf-labeled setting on \textit{C. elegans Mid} dataset. The x-axis is the percentage of known leaves.}
    \label{fig:partial_label_celegans_large}
\end{figure}

\subsection{Unsupervised Setting}\label{app:unsupervised}
Here, we provide a preliminary exploration to assess whether the \emph{CellTreeQM} framework can be extended to an unsupervised learning scenario. We use a Brownian motion simulation to generate a phylogenetic tree with 128 leaves with unit length. Each leaf is represented by 32 signal features and 32 Gaussian noise features, with the noise standard deviation set to 10. A direct reconstruction using the noisy data yields an RF value of 0.968. As baselines, we applied PCA with 5, 10, and 20 principal components, as well as Gaussian random projections with 2, 5, and 10 dimensions; all of these produced an RF value of 1. When we reconstruct the tree only based on the signal features, we got the optimal reconstruction accuracy RF = 0.128.

To adapt \emph{CellTreeQM} to the unsupervised setting, we make two main modifications to simplify the model compared to its supervised and weakly supervised versions:

\begin{enumerate}
\item We replace transformer encoder a fully connected backbone of 3 layers, each with 32 hidden units.
\item We replace the Gumbel-Softmax feature gating with a learnable linear projection matrix $G \in \mathbb{R}^{d \times d}$, where d is the number of input features.
\end{enumerate}

Because we do not have ground-truth distances among leaves in the unsupervised scenario, we approximate additivity by sorting the quartet distances $\{S_1, S_2, S_3\}$ (see \ref{dist_sums}) based on their values in the embedding space. Following the same idea as in DeSeto’s algorithm, we define the ``close’’ term as:
and take
\begin{equation*}
\mathcal{L}_{\text{quartet}} = \mathcal{L}_{\text{close}},
\quad
\mathcal{L}{_\text{additivity}} = \frac{1}{Q} \sum_{q=1}^{Q} \mathcal{L}_{\text{quartet}},
\end{equation*}
where Q is the total number of quartets.

We regularize the projection matrix G to encourage excluding noise features without excessively penalizing omissions. Specifically, we initialize G as the identity matrix and apply an $L_1$ penalty \(\lVert G - I \rVert_1\). Note that, unlike in the supervised and weakly supervised settings—where feature gating is applied directly to the input—here we apply feature gating only to the \emph{deviation} component. Concretely, we compare the learned embedding f(X) to its linear projection $GX$ via

\[
\Omega(f, X) = \frac{1}{N} \bigl\lVert \mathcal{D}\bigl(f(X)\bigr) - \mathcal{D}\bigl(GX\bigr) \bigr\rVert_{F}^{2},
\]

where $\mathcal{D}$ is an operator (e.g., pairwise distance) applied to the corresponding space. The overall loss function then combines the additivity loss, the embedding deviation term, and the regularization on G:

\[
\min_{f, G}
\Bigl[
\mathcal{L}_{\text{additivity}}(f, X) \;+\; \lambda \,\Omega(f, X) \;+\; \bigl\lVert G - I \bigr\rVert_1.
\Bigr]
\]

\begin{figure}[ht]
    \centering
    \includegraphics[width=0.5\linewidth]{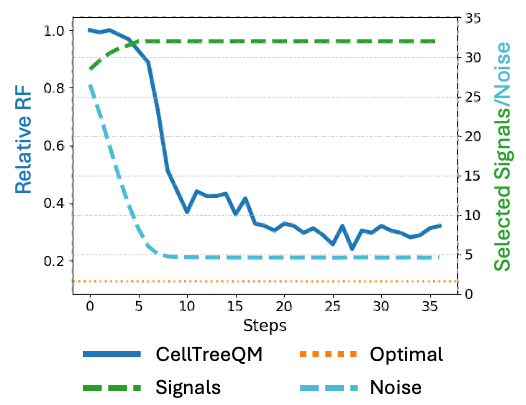}
    \caption{\textbf{Training dynamics of CellTreeQM in a purely unsupervised setting on a simulated dataset.} The left axis reports the relative RF of the reconstructed tree (blue) compared to the optimal RF (orange dotted line), while the right axis indicates the number of selected signal (green dashed) and noise (blue dashed) features over training steps. Early on, more noise features are included, resulting in poorer performance; as training proceeds, CellTreeQM discards the noise and approaches the optimal reconstruction.
    }\label{fig:unsupervsied_setting_app}
\end{figure}

Figure~\ref{fig:unsupervsied_setting_app} illustrates how \emph{CellTreeQM} learns in a purely unsupervised setting over the course of training on a simulated dataset. The figure shows that \emph{CellTreeQM}, even without any labels or supervision, progressively approaches the signal-only baseline, suggesting that the model’s gating mechanism and the unsupervised objective are successfully filtering out noise and capturing the underlying signal.
\section{Training Setups}\label{app:training_setups}
We evaluated our method in two settings: \emph{supervised} and \emph{weakly supervised}. For both settings, we tested on simulated and real datasets. Unless otherwise noted, we did not extensively tune the hyperparameters.

\paragraph{Supervised Setting.}
\begin{itemize}[leftmargin=*]
\item \textbf{Simulation:} We use a Transformer with 8 layers and 2 attention heads. The projection layer and hidden layer are both 256-dimensional, and the model outputs 128-dimensional embeddings. We apply a data dropout of 0.3 and a metric dropout of 0.2. When necessary, we set the gate regularization weight to 5.
\item \textbf{Real data:} The Transformer also has 8 layers and 2 attention heads, but with 1024-dimensional projection and hidden layers, producing 128-dimensional outputs. Here, we use a data dropout of 0.1 and a metric dropout of 0.1. When needed, the gate regularization weight is 0.01.
\end{itemize}

\paragraph{Weakly Supervised Setting.}
\begin{itemize}[leftmargin=*]
\item \textbf{Simulation:} We use a Transformer with 4 layers and 2 attention heads. The projection and hidden layers are both 256-dimensional, and the model outputs 128-dimensional embeddings. We apply a data dropout of 0.3 and a metric dropout of 0.2. When required, the gate regularization weight is 8.
\item \textbf{Real data:} The Transformer has 8 layers and 2 attention heads, with projection and hidden layers of 256 dimensions, producing 128-dimensional outputs. We use a data dropout of 0.3 and a metric dropout of 0.2, and set the gate regularization weight to 0.01 when needed.
\end{itemize}

\end{document}